\documentclass[11pt, a4paper, oneside, reqno]{amsart}

\usepackage[usenames, dvipsnames]{color}
\definecolor{darkblue}{rgb}{0.0, 0.0, 0.45}
\definecolor{lightblue}{RGB}{240,248,255}
\definecolor{lightblue2}{rgb}{0.68, 0.85, 0.9}
\definecolor{lightcyan}{rgb}{0.88, 1.0, 1.0}
\definecolor{palepink}{rgb}{0.98, 0.85, 0.87}

\usepackage[colorlinks	= true,
raiselinks	= true,
linkcolor	= darkblue, 
citecolor	= Mahogany,
urlcolor	= ForestGreen,
pdfauthor	= {Peyman Mohajerin Esfahani},
pdftitle	= {},
pdfkeywords	= {},
pdfsubject	= {},
plainpages	= false]{hyperref}

\usepackage{amsmath, amsthm, amssymb, amsfonts}
\usepackage{mathtools, mathrsfs}

\usepackage{enumerate, enumitem}
\usepackage{dsfont}

\usepackage[amssymb, thickqspace]{SIunits}
\usepackage{fancyhdr,mdframed,nicefrac}

\usepackage[norelsize, ruled, vlined]{algorithm2e}
\usepackage{algcompatible}

\usepackage{epsfig}
\usepackage{graphicx}
\usepackage{float}
\usepackage{caption}
\usepackage{subcaption}

\usepackage{courier}

\usepackage{multirow}
\usepackage{bigstrut}

\allowdisplaybreaks
\date{\today}
\makeatletter
\def\@settitle{\begin{center}%
		\baselineskip14\p@\relax
		\normalfont\LARGE\scshape\bfseries
		\@title
	\end{center}%
}

\def\@setauthors{%
  \begingroup
  \def\thanks{\protect\thanks@warning}%
  \trivlist
  \centering\footnotesize \@topsep30\p@\relax
  \advance\@topsep by -\baselineskip
  \item\relax
  \author@andify\authors
  \def\\{\protect\linebreak}%
  \authors%
  \ifx\@empty\contribs
  \else
    ,\penalty-3 \space \@setcontribs
    \@closetoccontribs
  \fi
  \endtrivlist
  \endgroup
}

\makeatother
\makeatletter

\def\subsection{\@startsection{subsection}{2}%
	\z@{.5\linespacing\@plus.7\linespacing}{.5\linespacing}%
	{\normalfont\large\bfseries}}

\def\subsubsection{\@startsection{subsubsection}{3}%
	\z@{.5\linespacing\@plus.7\linespacing}{.5\linespacing}%
	{\normalfont\itshape}}

\usepackage{multirow}
\usepackage{bigstrut}
\usepackage{wrapfig}
\usepackage{caption}

\newcommand{\mmode}[1]{\( #1 \)}

\newcommand{\setx}{\mathcal{X}}

\newcommand{\diameter}{d}

\newtheorem{theorem}{Theorem}[section]
\newtheorem{definition}[theorem]{Definition}
\newtheorem{lemma}[theorem]{Lemma}

\newtheorem{remark}[theorem]{Remark}
\newtheorem{corollary}[theorem]{Corollary}
\newtheorem{proposition}[theorem]{Proposition}

\newtheorem{assumption}[theorem]{Assumption}
\newtheorem{example}[theorem]{Example}

\newcommand{\R}[1]{\mathbb{R}^{#1}}

\renewcommand{\P}{\mathds{P}}
\newcommand{\Q}{\mathds{Q}}
\newcommand{\Pg}{\P_{\stepsize}}

\newcommand{\EE}[1]{\mathds{E}_{#1}}

\newcommand{\Pref}{\widehat{\P}}
\newcommand{\sigmaref}{\widehat{\Sigma}}
\newcommand{\muref}{\widehat{\mu}}
\newcommand{\vref}{V}

\newcommand{\simplex}[1]{\Delta^{#1}}
\newcommand{\amb}{\mathcal{P}}
\newcommand{\ambradius}{\rho}
\newcommand{\Wdist}[1]{\mathsf{WD}_{#1}}
\newcommand{\wamb}[1]{\mathcal{W}_{#1} (\Pref , \rho)}
\newcommand{\erwamb}[1]{\mathcal{W}_c ((\Pref_{#1} , \rho)) }

\newcommand{\alphamin}{\Bar{\alpha}}
\newcommand{\setprtfl}{\setx (\alphamin)}

\newcommand{\eps}{\varepsilon}

\newcommand{\fworacle}[2]{\mathcal{F} (#1, #2)}
\newcommand{\fwgap}{g}
\newcommand{\inner}[2]{\big \langle #1, #2 \big \rangle}
\newcommand{\ra}{\rightarrow}

\newcommand{\pnorm}[2]{\left\lVert#1\right\rVert_{#2}}
\newcommand{\norm}[1]{\pnorm{#1}{}}
\newcommand{\dualnorm}[1]{\pnorm{#1}{*}}

\newcommand{\abs}[1]{\left\lvert #1 \right\rvert}

\newcommand{\inprod}[2]{\left\langle#1 , \, #2\right\rangle}
\newcommand{\symmetric}[1]{\mathbb{S}^{#1 \times #1}}

\newcommand{\lagrangian}{L}
\newcommand{\identity}[1]{\mathbb{I}_{#1}}

\newcommand{\summ}[2]{\sum_{#1}^{#2}}

\newcommand{\define}{\coloneqq}

\newcommand{\opt}{^\ast}
\newcommand{\transp}{^\top}

\renewcommand{\geq}{\geqslant}
\renewcommand{\ge}{\geqslant}
\renewcommand{\le}{\leqslant}
\renewcommand{\leq}{\leqslant}

\newcommand{\symbolspace}[2]{#2 #1 #2}

\renewcommand{\mapsto}{\longmapsto}

\newcommand{\RP}[1]{R \big( #1 \big)}
\newcommand{\RPgrad}[2]{ {\rm d}\RP{#1;#2} }

\newcommand{\VP}[1]{V(#1)}
\newcommand{\Vgrad}[2]{{\rm d}\VP{#1;#2}}

\newcommand{\ERP}[1]{\mathcal{E}(#1)}
\newcommand{\ER}{\mathcal{E}}
\newcommand{\ERgrad}[2]{{\rm d}\ERP{#1;#2}}

\newcommand{\FS}[1]{\RP{#1}}
\newcommand{\FSgrad}[2]{{\rm d}\FS{#1;#2}}

\newcommand{\FP}[2]{F \big( #1, #2 \big)}
\newcommand{\fxp}[2]{F_{#1} (#2)}
\newcommand{\FPgrad}[3]{ {\rm d} F_{#1} (#2;#3)}

\newcommand{\vrisk}[2]{V(#1 , #2)}
\newcommand{\vxpgrad}[3]{{\rm d} V_{#1} (#2;#3)}
\newcommand{\vxpshrt}[2]{V_{#1} (#2)}

\newcommand{\sgmap}[1]{\Sigma_{#1}}
\newcommand{\mup}[1]{\mu_{#1}}

\newcommand{\errisk}[2]{\mathcal{E}(#1 , #2)}
\newcommand{\erxpgrad}[3]{{\rm d} \mathcal{E}_{#1} (#2;#3)}
\newcommand{\erxpshrt}[2]{ \mathcal{E}_{#1} (#2)}

\newcommand{\ximin}[1]{\underline{\xi}_{#1}}
\newcommand{\ximax}[1]{\overline{\xi}_{#1}}

\newcommand{\fxgrad}[3]{\left\langle \nabla_1 \FP{#1}{#2} , \ #3 - #1 \right\rangle}

\newcommand{\risk}{r}

\newcommand{\smooth}{C}
\newcommand{\stepsize}{\gamma}

\newcommand{\khat}{{\widehat{k}}}
\newcommand{\keps}{K(\eps)}

\newcommand{\fwval}{g}

\newcommand{\primaloptval}{F\opt}
\newcommand{\dualoptval}{F_*}

\newcommand{\xk}{x_k}
\newcommand{\pk}{\P_k}
\newcommand{\xeps}{x_{\eps}}
\newcommand{\peps}{\P_{\eps}}
\newcommand{\xopt}{x\opt}
\newcommand{\popt}{\P\opt}
\newcommand{\pg}{\P_{\stepsize}}

\newcommand{\xp}[1]{x(#1)}
\newcommand{\epssaddle}{\big( \xeps , \peps \big)}

\newcommand{\etax}[1]{\eta_{#1}}

\newcommand{\qeta}[1]{q_{#1}}

\newcommand{\boundx}{B_x}
\newcommand{\fapprox}[2]{F_{\eps}(#1, #2)}

\newcommand{\ellipsoid}{\mathcal{E}_{M}}

\DeclareSymbolFont{symbolsC}{U}{pxsyc}{m}{n}

\DeclareMathOperator{\trace}{tr}

\DeclareMathOperator{\sgn}{sgn}

\DeclareMathOperator*{\sbjto}{subject \; to}

\DeclareMathOperator*{\argmin}{argmin}
\DeclareMathOperator*{\argmax}{argmax}

\DeclarePairedDelimiter\ceil{\lceil}{\rceil}

\title[Nonlinear Distributionally Robust Optimization]{Nonlinear Distributionally Robust Optimization}

\author{Mohammed Rayyan Sheriff and Peyman Mohajerin Esfahani}
\thanks{The authors are with the Delft Center for Systems and Control, Delft University of Technology, Delft, The Netherlands. Emails: \href{mailto:mohammed.rayyan.sheriff@gmail.com}{\texttt{mohammed.rayyan.sheriff@gmail.com}},
\href{mailto:P.MohajerinEsfahani@tudelft.nl}{\texttt{P.MohajerinEsfahani@tudelft.nl}}.
The authors acknowledge the fruitful discussions with Armin Eftekhari. This research is supported by the European Research Council (ERC) under the grant TRUST-949796.}

\begin{document}
\maketitle

\begin{abstract}
This article focuses on a class of distributionally robust optimization (DRO) problems where, unlike the growing body of the literature, the objective function is potentially nonlinear in the distribution. Existing methods to optimize nonlinear functions in probability space use the Frechet derivatives, which present both theoretical and computational challenges. Motivated by this, we propose an alternative notion for the derivative and corresponding smoothness based on Gateaux (G)-derivative for generic risk measures. These concepts are explained via three running risk measure examples of variance, entropic risk, and risk on finite support sets. We then propose a G-derivative based Frank-Wolfe~(FW) algorithm for generic nonlinear optimization problems in probability spaces and establish its convergence under the proposed notion of smoothness in a completely norm-independent manner. We use the set-up of the FW algorithm to devise a methodology to compute a saddle point of the nonlinear DRO problem. Finally, we validate our theoretical results on two cases of the {\em entropic} and {\em variance} risk measures in the context of portfolio selection problems. In particular, we analyze their regularity conditions and ``sufficient statistic", compute the respective FW-oracle in various settings, and confirm the theoretical outcomes through numerical validation.

\end{abstract}
\textbf{Keywords.} Gateaux derivative, norm-free-smoothness, Frank-Wolfe algorithm, saddle point
\section{Introduction}
Modern-day decision problems involve uncertainty in the form of a random variable \mmode{\xi} whose behavior is modeled via a probability distribution~\mmode{\P_{o}}. A central object to formalize such decision-making problems under uncertainty is risk measures. The most popular risk measure is arguably the expected loss, yielding the classical decision-making problem of
\begin{equation}
\label{eq:stochastic-program}
\text{Stochastic Program:} \quad \quad
\min_{x \in \setx} \EE{\P_o} \big[ \ell(x , \xi) \big] ,
\end{equation}
where \mmode{\ell} is the loss function of interest, and \mmode{\setx} being the set of feasible decisions. The paradigm of stochastic programming (SP) relies on the assumption that the distribution \mmode{\P_o} is available (or at least up to its sufficient statistics), thereby the expectation can be computed for every decision \mmode{x \in \setx}. A common practical challenge is, however, that the complete information of \mmode{\P_o} may not be available. Moreover, it might also be the case that the distribution is varying over a period of time which could be difficult to characterize. These limitations call for a more conservative risk measure to ameliorate the decision performance in such situations. 

An alternative framework is Robust Optimization (RO) where the decision-maker has only access to the support of uncertainty and takes the most conservative approach:
\begin{equation}
\label{eq:RO-intro}
\text{Robust Optimization:} \quad \quad
    \min_{x \in \setx} \max_{\xi \in \Xi} \ell (x , \xi) .
\end{equation}
For many interesting examples, the RO min-max problem admits tractable reformulations that can be solved efficiently~\cite{bertsimas2011theory}. However, a generic RO problem is known to be computationally formidable as the inner maximization over \mmode{\xi} can be NP-hard. Moreover, if the support \mmode{\Xi} of the distribution is ``large'', the results of RO tend to be extremely conservative. 

\paragraph{{\bf Distributionally Robust Optimization (DRO)}}
The SP and RO decision models represent two extreme cases of having full or bare minimum distributional information, respectively. In practice, however, we often have more information about the ground truth distribution than just its support. A typical example is when we have sample realizations \mmode{\{ \widehat{\xi}_i: i = 1,2,\ldots,N \} }. Looking at such settings through the lens of SP, one may construct a \emph{nominal distribution}~\mmode{\Pref} and use it as a proxy for \mmode{\P_o} in the SP~\eqref{eq:stochastic-program}. A standard data-driven nominal distribution is the discrete distribution~\mmode{\Pref = \sum_{i = 1}^{N} \delta_{\widehat{\xi}_i}}. The SP decision when \mmode{\P_o = \Pref} in~\eqref{eq:stochastic-program} often admits a poor out-of-sample performance on a different dataset, a phenomenon that is also known as the ``optimizer’s curse" or ``overfitting"~\cite{smith2006optimizer}.  On the other hand, the RO viewpoint in~\eqref{eq:RO-intro} completely disregards the statistical information of \mmode{\P_o} available through the dataset \mmode{(\widehat{\xi}_i)_i}, or any other form of prior information.

An attempt to bridge the SP and RO modeling frameworks gives rise to the paradigm of \emph{Distributionally Robust Optimization}~(DRO), which dates back to the Scarf’s seminal work on the ambiguity-averse newsvendor problem in 1958~\cite{ref:scarf1958min}. The ``{ambiguity}'' set \mmode{\amb} is a family of distributions that are close in some sense to the nominal distribution~\mmode{\Pref}, potentially including the true unknown distribution~\mmode{\P_o}. With this in mind, the DRO problem is formulated as
\begin{equation}
\label{LDRO}
\min_{x \in \setx} \sup_{\P \in \amb} \EE{\P} [\ell (x , \xi)] .
\end{equation}
On the one hand, if the ambiguity set is \emph{very big}, possibly including all possible distributions, the DRO problem~\eqref{LDRO} reduces to the RO problem \eqref{eq:RO-intro}. On the other hand, if the ambiguity set~\mmode{\amb} is \emph{very small}, potentially a singleton containing the nominal distribution, then the DRO problem reduces to the SP problem~\eqref{eq:stochastic-program}. In this light, the DRO framework~\eqref{LDRO} provides flexibility for the decision-maker to fill the gap between SP~\eqref{eq:stochastic-program} and RO~\eqref{eq:RO-intro}. The ambiguity set~\mmode{\amb} in \eqref{LDRO} is typically constructed either based on the moments information~\cite{delage2010distributionally,ref:goh2010distributionally,ref:wiesemann2014distributionally}, or a neighborhood of \mmode{\Pref} with respect to a notion of distance over probability distributions, e.g., Prohorov~~\cite{ref:erdougan2006ambiguous}, Kullback-Leiber~\cite{ref:jiang2016data,ref:duchi2021learning}, Wasserstein \cite{ref:jiang2016data,mohajerin2018data,blanchet2022optimal, gao2022finite,gao2022distributionally}, Sinkhorn~\cite{wang2021sinkhorn}, to name but a few; see also the surveys~\cite{rahimian2019distributionally,kuhn2019wasserstein} and the references therein. 

\textit{Linear DRO problems:} An important feature of~\eqref{LDRO} is the linearity of the objective function in the distribution~\mmode{\P}, which is also shared among all the literature mentioned above. The simplicity of this linearity in the inner maximization of~\eqref{LDRO} is the underlying driving force to develop tractable convex reformulations and computational solutions for various combinations of ambiguity sets and cost functions~\cite{postek2016computationally, zymler2013worst, rujeerapaiboon2016robust, shafieezadeh2015distributionally,cai2020distributionally}. This line of research effectively translates the original infinite-dimensional DRO problem~\eqref{LDRO} to a tractable finite-dimensional one; see also the general optimal transport framework of~\cite{ref:zhen2021mathematical}, and the case of mean-covariance risk measure~\cite{ref:nguyen2021mean} in a financial context.

\textit{Nonlinear DRO (NDRO):} Our goal here is to generalize the linear setting~\eqref{LDRO} to 
\begin{equation}
\label{NDRO}
\min_{x \in \setx} \sup_{\P \in \amb} \FP{x}{\P} , 
\end{equation}
where \( F : \setx \times \amb \rightarrow \R{} \) is a generic, possibly nonlinear, function signifying a risk measure. We refer to the decision-making problem~\eqref{NDRO} as \emph{Nonlinear Distributionally Robust Optimization (NDRO)}. There are several interesting NDRO examples including the variance \cite{blanchet2022distributionally}, entropy \cite{song2019multivariate}, \cite[Example 1, p.\,14]{chen2020multivariate}. In this study, instead of focusing only on the outer decision~\mmode{x} in \eqref{NDRO}, we aim to compute an approximate saddle point between the decision-maker and the nature presented by the distribution~$\P$. A motivation supporting this effort is the fact that, unlike the worst-case distribution computed for a given decision, the saddle point distributions (also called Nash equilibria) naturally retain more realistic features~\cite{ref:shafieezadeh2023new}. From a computational perspective, the existing techniques deployed in linear DRO problems cannot be directly extended for the NDRO~\eqref{NDRO}. The objective of this work is to precisely tackle this challenge, where we seek to devise a methodology along with mild regularity conditions under which a saddle point solution exists and can be computed.

\paragraph{{\bf Frank-Wolfe (FW) algorithm}} It is a first-order method that only uses the information of the gradient to solve a constrained optimization problem \cite{jaggi2013revisiting, frank1956algorithm, clarkson2010coresets}. In a nutshell, each iteration of the FW algorithm optimizes the linear function given by the gradient of the objective function and then takes a step towards the optimizer of the linear problem. On the contrary, other first-order methods like projected gradient descent require a quadratic function to be minimized at each iteration. Since solving a linear problem at each iteration is easier than a quadratic one, the iteration complexity of the FW algorithm is much simpler than other first-order methods. This is crucial for optimization problems over probability distributions since the complexity of projecting onto the ambiguity set can be as challenging as solving the original NDRO problem. In comparison, optimizing a cost function that is linear in distributions admits strong duality and tractable finite-dimensional reformulations for many interesting examples as seen in linear DRO. Therefore, we seek to use the principles of the Frank-Wolfe algorithm in the context of NDRO problems and devise an iterative procedure to compute a saddle point of \eqref{NDRO}. 

The FW algorithm for optimization problems over probability distributions has already been introduced in \cite{kent2021frank}, which uses Frechet derivatives and the associated notion of smoothness to establish convergence. We would also like to highlight the work of \cite{liu2021infinite}, where the canonical gradient ascent-descent algorithm \cite{nemirovski2004prox, hamedani2018iteration} for min-max problems is extended to infinite-dimensional spaces involving probability distributions using the Frechet derivatives and its smoothness. In a similar spirit, the recent work of \cite{nguyen2022particle} proposes a mirror-descent algorithm \cite{nemirovskij1983problem} for constrained nonlinear optimization problems over probability distributions using the Frechet derivative. However, Frechet derivatives are difficult to deal with in practice due to several prominent challenges including their (i) existence, (ii) finite representability, and (iii) norm consistency; see Section~\ref{subsection:frechet-derivatives} for more details on this. Our focus in this study is to remedy this by proposing a FW algorithm based on an alternative G-derivative along with a completely norm-independent convergence analysis. 

\paragraph{{\bf Contributions}} The main contributions of this study are summarized as follows: 
\begin{enumerate}[label = \rm{(\roman*)}, itemsep = 0mm, topsep = 0mm, leftmargin = *]

\item {\bf Norm free smoothness in probability spaces.} We propose a novel G-derivative based notion of derivative for nonlinear risk measures~(Definition~\ref{def:directional-derivative}), and the associated notion of smoothness that is independent of the norm structure on the ambiguity set~(Definition~\ref{def:smoothness}). Moreover, we also derive conditions on the function \mmode{\FP{x}{\P}} such that the risk measure \mmode{\P \mapsto \min_{x \in \setx} \ \FP{x}{\P}} is smooth in the sense of the proposed notion~(Lemma~\ref{lemma:NDRO-smoothness}).

\item {\bf G-derivative based Frank-Wolfe algorithm.} We provide a Frank-Wolfe algorithm based on the proposed notion of derivative for optimizing nonlinear risk measures. Moreover, the classical proof techniques for the FW algorithm carry forward under the proposed notion of derivative and smoothness, resulting in apriori (Propositions~\ref{proposition:FW-convergence}) and aposteriori (Proposition~\ref{proposition:FW-gap}) convergence guarantees that exist for finite-dimensional problems.

\item {\bf Saddle point seeking algorithm for NDRO problems.} For the potentially infinite-dimensional min-max problem of NDRO, we propose a FW-based algorithm to compute a saddle point (Algorithm~\ref{algo:NDRO-min-max}), and also quantify its convergence properties~(Theorem \ref{theorem:NDRO-FW}). 

\item {\bf The entropic and variance risk measures.} 
We study the NDRO problem and our proposed algorithm for two cases of the entropic and variance risk measures in Sections~\ref{section:er-risk-portfolio-selection} and \ref{section:min-variance-portfolio-selection}, respectively. An interesting difference between these two examples is that in the context of portfolio selection, the variance preserves its ``sufficient statistic" throughout the FW-algorithm while the entropic risk requires the knowledge of the entire distribution. We establish the required convex regularity conditions (Lemmas~\ref{lemma:er-risk-smoothness-assumptions} and \ref{lemma:variance-smoothness-assumptions}, respectively), and then provide a complete description of the corresponding FW-oracle (Lemmas~\ref{lemma:er-risk-FW-oracle} and \ref{lemma:min-var-ldro-solution-for-m2}, \ref{lemma:FW-oracle-min-variance-ellipsoidal-support}, respectively). Furthermore, for the minimum variance portfolio selection problem, in the special case of \mmode{\Xi = \R{n}} and the type-$2$ Wasserstein ambiguity set, we slightly extend the results of \cite{blanchet2022distributionally} by explicitly describing the saddle point of the minimum variance problem when the feasible portfolio set~\mmode{\setx} is {\em any} arbitrary compact set (Proposition \ref{proposition:min-variance-blanchet}). To facilitate the reproducibility of the numerical results, we also provide the open source repository~\cite{DROpackage} including the respective {\fontfamily{pcr}\selectfont MATLAB} code.
\end{enumerate}

The rest of the paper is organized as follows: In Section \ref{section:worst-case-risk-problem}, we discuss a generic nonlinear optimization problem over distributions with relevant examples. In Section \ref{section:derivatives-and-smoothness}, we introduce the notion of a directional derivative and the associated notion of smoothness. In Section \ref{section:FW-algorithm}, the FW-algorithm and its convergence guarantees are discussed. In Section \ref{section:NDRO}, we introduce the NDRO problem, discuss a solution concept, and provide an FW-based algorithm to compute the solution. To illustrate the methodology (the required assumptions and algorithms), we discuss in detail the NDRO problems for two nonlinear risk measures of entropic and variance, respectively, in Sections \ref{section:er-risk-portfolio-selection} and \ref{section:min-variance-portfolio-selection}. The numerical results are provided for both risk measures in the context of optimal portfolio selection. All of the technical proofs have been relegated towards the end of the article in Section \ref{section:technical-proofs}.

\textbf{Notations.} The set of real-valued \mmode{n\times n} matrices is denoted by \mmode{\R{n \times n}}. The trace of a matrix \mmode{M} is denoted by \mmode{\trace (M)} and being positive semi-definite is denoted by \mmode{M \succeq 0}. For a function \mmode{\risk} defined over a finite-dimensional space, its gradient at \mmode{x} is denoted by \mmode{\nabla \risk (x)}. For a multivariate distribution~$\P$, the first and second moments are denoted by the shorthand notation~$\mu_{\P} := \EE{\P}[\xi]$ and $\Sigma_{\P} := \EE{\P}[\xi\xi{\transp}]$, respectively. Given two probability distributions~\mmode{\P, \Q}, we also use \mmode{\EE{\P - \Q}[\cdot] \define \EE{\P}[\cdot] - \EE{\Q}[\cdot]}. Other notations shall be defined whenever necessary.

\section{Worst-case Nonlinear Risk Measures}
\label{section:worst-case-risk-problem}

In this section, we focus our attention to the potentially infinite-dimensional optimization problem: 
\begin{align}
\label{eq:main-opt}
R\opt \define \sup_{\P \in \amb} \RP{\P} ,
\end{align}
where $R: \amb \ra \R{} $ is a desired, possibly {\em nonlinear}, concave risk measure, and $\amb$ is the ambiguity set containing a family of probability distributions over \mmode{\Xi \subset \R{n}}. Our main goal is to develop a framework with appropriate mathematical notions for a \emph{Franke-Wolfe} (FW) like algorithm to solve problem~\eqref{eq:main-opt} and investigate its convergence properties. Furthermore, since the FW algorithm for \eqref{eq:main-opt} operates in an infinite-dimensional setting, we also seek to derive its tractable finite-dimensional simplification to solve specific instances of \eqref{eq:main-opt}. 

\begin{definition}[Regular risk (RR) measures \& sufficient statistic]
\label{def:RR-measure}
A risk measure \mmode{R} in \eqref{eq:main-opt} is regular if it can be described as \( \RP{\P} = \risk \big( \EE{\P} [L(\xi)] \big) \), for some functions \mmode{L:\Xi \ra \R{m}} that is integrable for all \mmode{\P \in \amb} and \( \risk : \R{m} \ra \R{} \) that is concave and differentiable. We refer to $\EE{\P} [L(\xi)]$ the ``sufficient statistic" as the risk value depends on the distribution~$\P$ only through this vector.
\end{definition}

A particular feature of the regular risks in Definition~\ref{def:RR-measure} is that its value is determined based on a finite-dimensional vector~$\EE{\P} [L(\xi)]$. This is a concept close to ``sufficient statistic", which is particularly appealing when it comes to the computation of the worst-case risk in~\eqref{eq:main-opt}. We will return to this in Section~\ref{section:FW-algorithm} when discussing the FW algorithm in the space of probability distributions.

\begin{example}[RR-examples]
\label{example:RR-measures}\rm{Throughout this study, we discuss three particular examples of the regular risk measures to showcase the concepts and our theoretical statements:
\begin{enumerate}[label = {\rm (\alph*)}, itemsep = 0mm, topsep = 0mm, leftmargin = *]
\item \emph{Variance}: A popular risk measure is the variance associated with the distribution. Formally, considering \mmode{\xi} to be a \mmode{\P}-distributed random variable for \mmode{\P \in \amb}, the associated variance is 
\begin{subequations}
 \label{eq:variance-L-r}   
\begin{equation}
V(\P) \define \EE{\P}\big[\|\xi - \EE{\P}[\xi]\|_2^2\big] .
\end{equation}
Considering the functions \mmode{L : \Xi \rightarrow \R{n\times n} \times \R{n}} and \mmode{\risk : \R{n\times n} \times \R{n} \rightarrow \R{}}
\begin{equation}
\begin{cases}
\begin{aligned}
L (\xi) & = (\xi \xi\transp , \, \xi) \\
\risk (\Sigma, \mu) & = \trace (\Sigma) - \pnorm{\mu}{2}^2 ,
\end{aligned}
\end{cases}
\end{equation}
\end{subequations}
one can observe that the variance is indeed an RR measure, since
\[
V(\P) \ = \ \EE{\P}\big[\|\xi\|_2^2 \big] - \|\EE{\P}[\xi]\|_2^2 \ = \ \trace \left( \EE{\P}[\xi\xi\transp] \right) - \pnorm{\EE{\P} [\xi]}{2}^2 \ = \ \risk \Big( \EE{\P} [L(\xi)] \Big) .
\]

\item \emph{Entropic risk}: Another interesting example of a nonlinear RR measure is the entropic risk of a multivariate distribution \cite[Section 5]{song2019multivariate}. If \mmode{\P \in \amb} is a distribution with marginals \mmode{\P_j} for \mmode{j = 1,2,\ldots,n}, (i.e., the \mmode{j}-th component \mmode{\xi_j} is \mmode{\P_j} distributed). The entropic risk \mmode{\ERP{\P}} associated with the distribution \mmode{\P}  is defined as
\begin{subequations}
\label{eq:er-risk-L-r}
\begin{equation}
\ERP{\P} \define \summ{j = 1}{n}{\frac{1}{\theta_j} \log \left( \EE{\P_j} [e^{-\theta_j \xi_j}] \right) } ,
\end{equation}
where \mmode{(\theta_j)_{j = 1}^n} is a collection of positive real numbers referred to as the \emph{risk-aversion parameters}. This is indeed an RR measure, which can be seen by introducing the functions
\begin{equation}
\begin{cases}
\begin{aligned}
L(\xi) & = \big( e^{-\theta_1 \xi_1} , e^{-\theta_2 \xi_2} , \ldots , e^{-\theta_n \xi_n} \big) \\
\risk (z) & = \summ{j = 1}{n} \frac{1}{\theta_j} \log(z_j) .
\end{aligned}
\end{cases}
\end{equation}
\end{subequations}

\item \emph{Finite-support}: The final example is the case of \mmode{\Xi} being a finite set: \mmode{\Xi = \{\xi_i : i = 1,2,\ldots,N \}}. In this case, the simplex \mmode{\simplex{N}} is the set of all probability distributions on \mmode{\Xi}, and the ambiguity set of distributions is a subset, \mmode{\amb \subset \simplex{N} }. It turns out that \emph{any} arbitrary risk measure \mmode{R} can be characterized as an RR-measure in the sense of Definition \ref{def:RR-measure} by introducing appropriate functions \mmode{r : \R{N} \rightarrow \R{} }, and \mmode{L : \Xi \rightarrow \R{N} }. To see this, we first observe that \mmode{\Xi} being a finite set gives rise to the enumerating bijection \mmode{b : \Xi \rightarrow \{1,2,\ldots,N \} } defined by \mmode{ b(\xi) \define \{ i : \xi = \xi_i \} }; secondly, the matrix \mmode{M \in \R{N \times N}} given by \mmode{[M]_{ij} \define (\nicefrac{j}{N})^i } for \mmode{i,j = 1,2,\ldots,N}, is invertible. Then, considering the functions
\begin{equation}
\label{eq:finite-support-L-r}
\begin{cases}
\begin{aligned}
    L(\xi) & = \big( \left(\nicefrac{b(\xi)}{N} \right), \left(\nicefrac{b(\xi)}{N} \right)^2 , \ldots, \left(\nicefrac{b(\xi)}{N} \right)^N \big)\transp \\
    r(z) &= \RP{M^{-1} z} \ \text{ for } \ z \in M(\amb),
\end{aligned}
\end{cases}
\end{equation}
we have \mmode{\EE{\P}[L(\xi)] = M\cdot\P} (viewing \mmode{\P \in \simplex{N}} as an element of \mmode{\R{N}}), and as such, \mmode{\RP{\P} = r(\EE{\P}[L(\xi)])} for all \mmode{\P \in \amb}.
\end{enumerate}}
\end{example}

\section{Derivatives of Risk Measures}
\label{section:derivatives-and-smoothness}
\begin{wrapfigure}{r}{0.475\linewidth}
    \centering
    \includegraphics[width=0.4\textwidth]{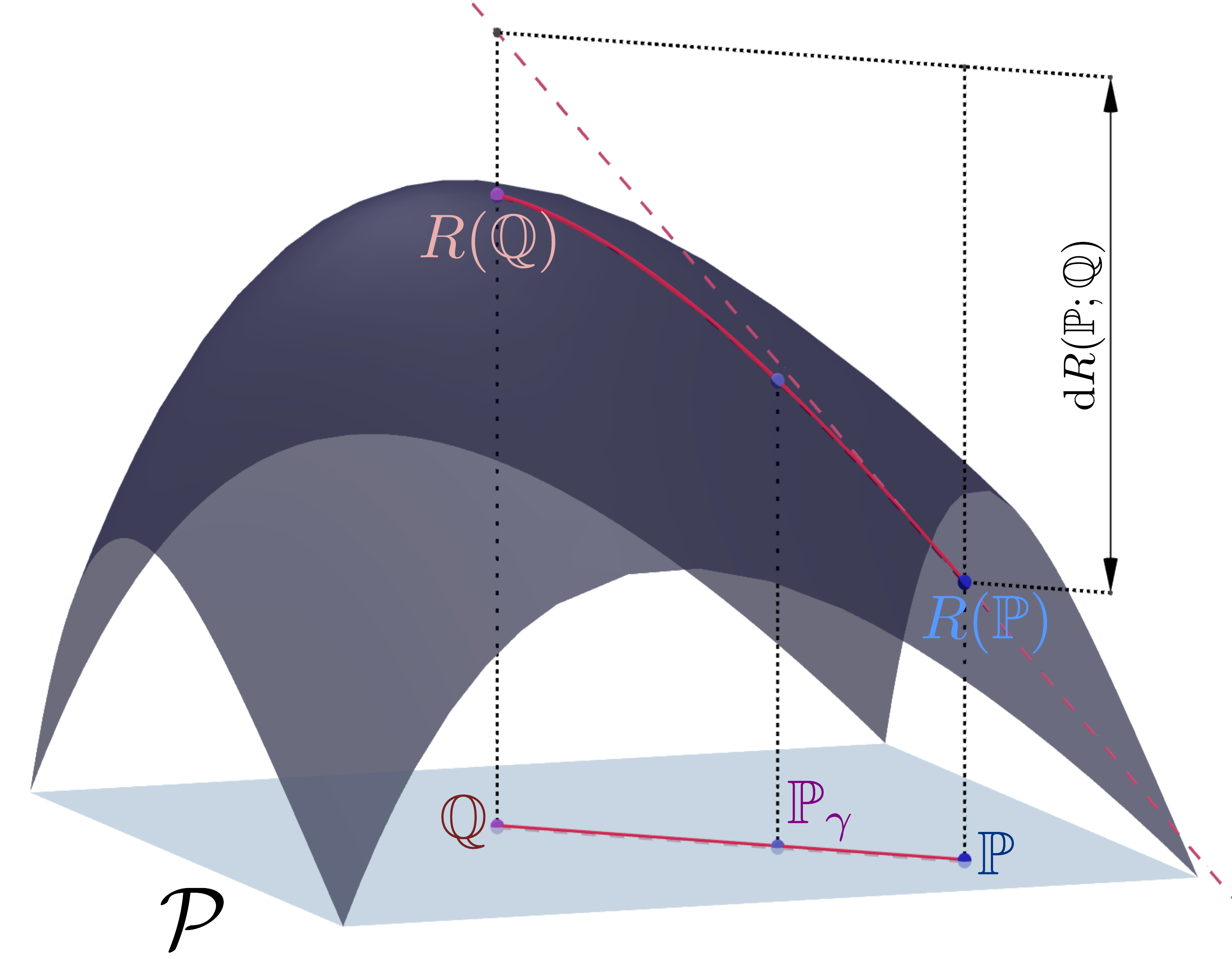}
    \caption{Pictorial representation of the risk surface and its directional derivative.} 
    \label{fig:myfig}
\end{wrapfigure}
A well-defined notion of the gradient is a fundamental quantity in developing iterative algorithms like that of FW to solve any optimization problem. For finite-dimensional convex problems, the FW algorithm optimizes the linear functional given by the gradient of the objective function over the feasible set at each iteration. Naturally, devising a similar algorithm for \eqref{eq:main-opt} requires at least a well-defined notion of \emph{directional derivative} or G-derivative \cite[(A.3), p. 152]{bertsekas1971control}.

\subsection{Gateaux directional derivatives}
\begin{definition}[G-derivative]
\label{def:directional-derivative}
Given \( \P , \Q \in \amb \), the Gateaux (G)-derivative \( \RPgrad{\P}{\Q} \) of the risk measure  \( R \) at \( \P \) in the direction~$\Q$ is defined as
\begin{equation}
\label{eq:derivative-definition}
\RPgrad{\P}{\Q} \define \lim_{\stepsize \downarrow 0} \frac{1}{\stepsize} \Big( \RP{\Pg} - \RP{ \P } \Big), \quad \text{where} \quad \Pg = \P + \stepsize (\Q - \P)\,.
\end{equation}
Whenever the above limit exists, we say that the function \( R \) is \( \Q \)-directionally differentiable at \( \P \). Moreover, we say that the risk measure \( R \) is directionally differentiable on \( \amb \) if it is \( \Q \)-directionally differentiable at every \( \P \in \amb \) and for all \( \Q \in \amb \).
\end{definition}

We note that the G-directional derivative in Definition~\ref{def:directional-derivative} does not rely on any metric underlying the space of probability distributions. This is in fact the main feature with respect to the alternative F-derivative~(Frechet derivative) that will be discussed in Section~\ref{subsection:frechet-derivatives}. Figure~\ref{fig:myfig} visualizes this directional derivative. Next lemma provides an explicit description of Definition~\ref{def:directional-derivative} for RR measures.

\begin{lemma}[Regular G-derivatives]
\label{lemma:RR-directional derivatives}
Suppose that the risk measure is regular, i.e., \mmode{\RP{\P} = \risk (\EE{\P}[L(\xi)] ) } with \mmode{\nabla \risk} denoting the gradient of function \mmode{\risk}. Then, for any \mmode{ \P, \Q \in \amb}, we have
\begin{equation}
\label{eq:RR-directional-derivative}
\RPgrad{\P}{\Q} \symbolspace{=}{\;} \EE{\Q-\P} \left[ \inprod{ \nabla \risk (\EE{\P} [L(\xi)]) }{  L(\xi)} \right] \symbolspace{=}{\;} \inprod{\nabla \risk (\EE{\P} [L(\xi)])}{ \EE{\Q - \P} [L(\xi)] } .
\end{equation}
\end{lemma}

Lemma~\ref{lemma:RR-directional derivatives} indicates that the G-derivative of an RR-measure is essentially characterized by the ``sufficient statistic" vectors \mmode{\EE{\P}[L(\xi)]} and \mmode{\EE{\Q}[L(\xi)]} (cf.\,\eqref{eq:RR-directional-derivative}). Thus, if an algorithm optimizes an RR measure using only their directional derivatives, it then requires tracking the evolution of this finite-dimensional sufficient statistic, allowing them to be implemented efficiently.

\begin{example}[Regular G-derivatives]\rm{
\label{example:RR-directional-derivatives}
The directional derivative of an RR measure is completely characterized in terms of only a few finite-dimensional quantities that depend on the moments of the distribution, and the functions \mmode{\risk } and \mmode{L}.
\begin{enumerate}[label = {\rm (\alph*)}, itemsep = 0mm, topsep = 0mm, leftmargin = *]
\item \emph{Variance}: With the underlying inner product \mmode{\inprod{(\Sigma , \mu)}{(\Sigma' , \mu')} \define \trace (\Sigma\transp\Sigma') + \mu\transp\mu'}, we recall from \eqref{eq:variance-L-r} that \mmode{\nabla \risk (\Sigma, \mu) = \Big( \identity{n} , -2\mu \Big) }. Then the \mmode{\Q}-directional derivative of the variance \mmode{V}, at \mmode{\P} can be calculated from \eqref{eq:RR-directional-derivative} as
\begin{equation}
\label{eq:variance-G-derivative}
	\Vgrad{\P}{\Q} = \trace \big( \Sigma_{\Q} - \Sigma_{\P} \big) - 2 \mu_{\P}\transp (\mu_{\Q} - \mu_{\P}) .
\end{equation}
	
\item \emph{\mmode{\ER}-measure}: With the canonical inner product \mmode{\inprod{z}{\widehat{z}} = z\transp \widehat{z}}, we recall from \eqref{eq:er-risk-L-r} that \mmode{\nabla \risk (z) = \big({(\theta_1 z_1)^{-1}} , {(\theta_2 z_2)^{-1}} , \ldots , {(\theta_n z_n)^{-1}}\big)} at every \mmode{z \in \R{n}_+}. Thus, its \mmode{\Q}-directional derivative at \mmode{\P} is
\begin{equation}
\label{eq:er-risk-G-derivative}
\ERgrad{\P}{\Q} = \summ{j = 1}{n}{ \frac{\EE{\Q_j - \P_j} [e^{-\theta_j \xi_j}] }{ \theta_j \EE{\P_j} [e^{-\theta_j \xi_j}]} }.    
\end{equation}

\item \emph{Finite support}: Recall from \eqref{eq:finite-support-L-r} that \mmode{\EE{\P} [L(\xi)] = M\cdot \P} for all \mmode{ \P \in \amb }. Then with the canonical inner product \mmode{\inprod{z}{z'} = z\transp z'}, we have \mmode{\nabla \risk (z) = (M^{-1})^{\transp} \nabla R (M^{-1}z) } for all \mmode{ z \in M(\amb) }. Substituting these quantities in \eqref{eq:RR-directional-derivative} and simplifying, we get
\begin{equation}
\label{eq:finite-support-G-derivative}
\FSgrad{\P}{\Q} = (\nabla \FS{\P} )\transp (\Q - \P) .
\end{equation}
It is to be observed that the matrix \mmode{M} has no relevance in the G-derivative as one would expect since the risk measure \mmode{R} is defined independent of the matrix \mmode{M}.
\end{enumerate}
}\end{example}

The G-derivative enjoys inherent properties that will be helpful to devise computational solutions. 
\begin{proposition}[G-derivative: properties]
\label{proposition:derivative-properties}
For any \mmode{\P , \Q \in \amb}, and \mmode{\stepsize \in [0,1]}, let \mmode{\P_{\stepsize} = \P + \stepsize (\Q - \P)}, then the directional derivative \mmode{\RPgrad{\P}{\cdot}} in Definition \ref{def:directional-derivative} satisfies
\begin{subequations}
\begin{align}
\label{eq:positive-homogeniety}
\text{\rm{Positively homogeneous: }}& \quad
\RPgrad{\P}{\P_{\stepsize}} \symbolspace{=}{\;} \stepsize \RPgrad{\P}{\Q} \\
\label{eq:concavity-linear-bound}
\text{\rm{Upper bound for concave risk measures: }}& \quad
\RP{\Q} - \RP{\P} \symbolspace{\leq}{\;} \RPgrad{\P}{\Q} .
\end{align}
\end{subequations}
\end{proposition}

\begin{remark}[Optimality conditions]
\label{remark:optimality conditions}
Similar to KKT conditions, F-derivative based first-order optimality conditions for both constrained and unconstrained versions of \eqref{eq:main-opt} are given in \cite[Section 3]{lanzetti2022first}. With G-derivatives, the upper bound~\eqref{eq:concavity-linear-bound} for the concave risk measure \mmode{R} immediately gives sufficient optimality conditions for \eqref{eq:main-opt}. More precisely, if \mmode{\P\opt \in \argmax_{\Q \in \amb} \ \RPgrad{\P\opt}{\Q} }, then applying \eqref{eq:concavity-linear-bound} for any \mmode{\Q \in \amb} yields
\[
\RP{\Q} \symbolspace{\leq}{\;} \RP{\P\opt} + \RPgrad{\P\opt}{\Q} \symbolspace{\leq}{\;} \RP{\P\opt} + \RPgrad{\P\opt}{\P\opt} \symbolspace{=}{\;} \RP{\P\opt} ,
\]
which establishes the optimality \mmode{\P\opt \in \argmax_{\Q \in \amb} \ \RP{\Q}}.
\end{remark}

\subsection{Smoothness}
The proposed approach to solve \eqref{eq:main-opt} builds on the Frank-Wolfe (FW) algorithm for finite-dimensional convex optimization problems \cite{frank1956algorithm}. It is a well-known fact that the primal sub-optimality in FW algorithm converges at a sub-linear rate O(\nicefrac{1}{k}) for optimization of ``smooth'' objective functions over compact feasible sets in finite-dimensional problems. A convex function is said to be smooth if it has Lipschitz continuous gradients with respect to some norm. The choice of norm in an infinite-dimensional setting can be problematic, as all the norms are not equivalent (unlike the finite-dimensional setting). To extend such convergence attributes for a FW like algorithm in the infinite-dimensional setting of \eqref{eq:main-opt}, we first propose an appropriate notion of smoothness in terms of the  directional derivatives (as in Definition \ref{def:directional-derivative}) that is also norm-independent.

\begin{definition}[G-smoothness]
\label{def:smoothness}
The risk measure \( R \) is \mmode{\smooth}-smooth if there exists a constant \mmode{\smooth \geq 0} such that for all \( \P, \Q \in \amb \) and \( \stepsize \in [0, 1] \), we have the inequality
\begin{equation}
\label{eq:G-smoothness}
\RPgrad{\P_{\stepsize}}{\P} + \RPgrad{\P}{\P_{\stepsize}} \symbolspace{\leq}{\;} \stepsize^2 \smooth \quad \text{where \( \P_{\stepsize} = \P + \stepsize \big( \Q - \P \big) \) } .
\end{equation}
\end{definition}

\paragraph{{\bf Connection to the existing notions of smoothness}}
The notion of smoothness in Definition \ref{def:smoothness} is a generalization of the canonical smoothness condition of Lipschitz continuous gradients \cite[Section 2.1.5]{nesterov2003introductory}. Formally, a function \mmode{f : \mathcal{D} \rightarrow \R{}} with a finite-dimensional domain \mmode{(\mathcal{D}, \norm{\cdot})} and gradients \mmode{\nabla f } is said to be \mmode{\beta}-smooth if 
\begin{equation}
\label{eq:canonical-smoothness}
\dualnorm{\nabla f (x) - \nabla f (y)} \symbolspace{\leq}{\;} \beta \norm{x - y} \quad \text{ holds for all \mmode{x,y \in \mathcal{D}}.}
\end{equation}
It is further generalized (or relaxed) by the notion of Holder-smoothness (particularly, \mmode{1}-Holder smooth) where it is required that
\begin{equation}
\label{eq:holder-smoothness}
\inner{\nabla f (x) - \nabla f (y)}{y - x} \symbolspace{\leq}{\;} \beta \norm{x - y}^{2} \quad \text{holds for all \mmode{x,y \in \mathcal{D}}}.
\end{equation}
Furthermore, if the set \mmode{\mathcal{D}} is \mmode{\pnorm{\cdot}{}}-bounded in addition, then the notion of Holder-smoothness \eqref{eq:holder-smoothness} is sufficient to the requirement that there exists some constant \mmode{\smooth \geq 0} such that for every \mmode{x,y \in \mathcal{D}}, \mmode{\gamma \in [0,1]}, the function \mmode{f} satisfies the inequality
\begin{equation}
\label{eq:holder-smoothness-in-stepsize}
\inner{\nabla f (x) - \nabla f (x_{\stepsize})}{x_{\stepsize} - x} \symbolspace{\leq}{\;} \stepsize^2 \smooth \quad \text{ where \mmode{ \ x_{\stepsize} \define x + \stepsize (y - x)}}.
\end{equation}
Expanding the left-hand side of \eqref{eq:holder-smoothness-in-stepsize} as \mmode{\inner{\nabla f (x) - \nabla f (x_{\stepsize})}{x_{\stepsize} - x} = \inner{\nabla f(x)}{x_{\stepsize} - x} + \inner{\nabla f (x_{\stepsize})}{x - x_{\stepsize}} }, we see that the individual terms are simply the directional derivatives of \mmode{f}; Definition \ref{def:smoothness} of the proposed notion of smoothness becomes apparent at once. We also highlight that the notion of smoothness in Definition \ref{def:smoothness} is closely related to the notion of ``curvature coefficient'' used in \cite[(3)]{jaggi2013revisiting}. In fact, it can be easily shown that a function has a finite curvature coefficient if it is smooth in the sense of Definition~\ref{def:smoothness}.

\paragraph{{\bf Concave quadratic lower bound via smoothness}}
\begin{wrapfigure}{r}{0.365\linewidth}
    \centering
    \includegraphics[width=0.3325\textwidth]{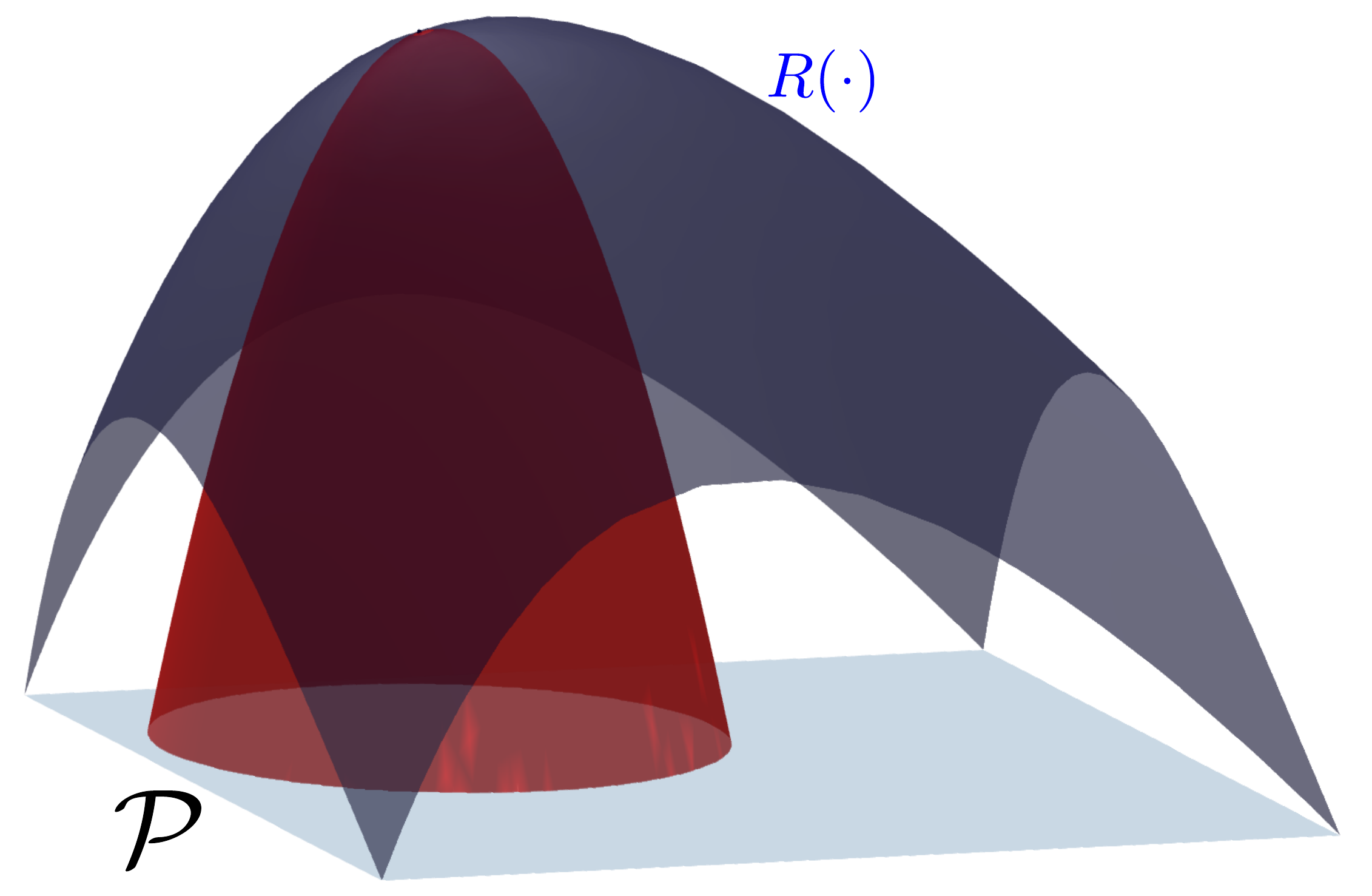}
    \caption{Concave-quadratic lower bound via smoothness.} 
    \label{fig:smoothness}
\end{wrapfigure}
It is well established that in a finite-dimensional setting, smoothness conditions like \eqref{eq:canonical-smoothness} and \eqref{eq:holder-smoothness} give rise to a global convex (resp.\,concave) quadratic upper bound (resp.\,lower bound). For the risk measure \mmode{R} specifically, a sample representation of such a quadratic lower bound is shown in Figure \ref{fig:smoothness}. The existence of such bounds guarantees that the curvature of the function is at most that of the quadratic bounds, which is crucial in concluding the convergence of the FW algorithm. In other words, if a function \mmode{f} is \mmode{\beta}-canonically smooth (as in \eqref{eq:canonical-smoothness}), then the following inequality holds: 
\[
-\frac{\beta}{2} \norm{y - x}^2 \leq f(y) - f(x) - \inprod{\nabla f (x)}{y - x} \leq \frac{\beta}{2} \norm{y - x}^2 \quad \text{for all } x,y \in \mathcal{D}.
\]
The notion of smoothness in Definition \ref{def:smoothness} imposes similar quadratic bounds, but without using any norm. This is done by enforcing ``quadratic-like'' bounds to hold uniformly over all directions. The following Lemma establishes the smoothness of the class of RR measures.

\begin{lemma}[Regular G-smoothness]
\label{lemma:RR-smoothness}
Suppose the RR measure \mmode{\RP{\P} = \risk (\EE{\P}[L(\xi)] ) } satisfies

\begin{enumerate}[label = {\rm (\roman*)}, itemsep = 0mm, topsep = 0mm, leftmargin = *]
    \item bounded diameter in expectations with \mmode{ \diameter \define \sup\limits_{\P, \Q \in \amb} \; \norm{\EE{\P} [L (\xi)] \; - \; \EE{\Q} [L (\xi)]} \symbolspace{<}{\;}   +\infty }, and
    
    \item it admits a smooth \mmode{\risk} in the canonical sense \eqref{eq:canonical-smoothness}, i.e., 
    \[
       \pnorm{\nabla \risk (u) - \nabla \risk (v)}{*} \leq \beta \norm{u - v} \quad \text{for all } u, v \in \{ \EE{\P} [L(\xi)] : \P \in \amb \} .
    \]
\end{enumerate}
Then, it is \mmode{( \beta d^2)}-smooth in the sense of Definition \ref{def:smoothness}.
\end{lemma}

\begin{example}[Regular G-smoothness]\rm{
For the RR measures in Examples \ref{example:RR-measures} and \ref{example:RR-directional-derivatives} with the same underlying norms, the smoothness constant \mmode{\beta} of the respective \mmode{\risk}-function is as follows:

\begin{enumerate}[label = {\rm (\alph*)}, itemsep = 0mm, topsep = 0mm, leftmargin = *]
    \item \emph{Variance}: Recalling the function \mmode{\risk} from \eqref{eq:variance-L-r}, and that \mmode{\nabla \risk \big( \Sigma, \mu \big) = \big( \identity{n} , -2\mu  \big)}, we get \mmode{\beta = 2}.

    \item \emph{Entropic risk}: Assume that there exists \mmode{b > 0} such that \mmode{ b \leq \EE{\P_j}[e^{-\theta_j \xi_j}] } for all \mmode{j = 1,2,\ldots,n}, and \mmode{\P = (\P_j)_j \in \amb}, and letting \mmode{ \theta_{\min} \define \min_{j\le n} \theta_j }, we see that the function \mmode{r : [b, +\infty)^n \; \to \R{} } as given in \eqref{eq:er-risk-L-r} is \mmode{\beta}-canonically smooth for \mmode{\beta = (b^2 \theta_{\min})^{-1} }.

    \item \emph{Finite support}: Assuming that the risk measure \mmode{R} has  \mmode{\beta'}-Lipschitz continuous gradients on \mmode{\simplex{N}} with respect to some norm \mmode{\norm{\cdot}}. It follows from \eqref{eq:finite-support-L-r}, that \mmode{\beta = \beta' \pnorm{M^{-1}}{o}^2}, where, \mmode{\pnorm{M^{-1}}{o} = \sup_{\norm{u} \leq 1} \; \norm{M^{-1}u} }. 
\end{enumerate} 
}\end{example}
One of the advantages of the proposed notion of smoothness is that it allows us to employ inequalities and bound sets in a finite-dimensional space to guarantee the smoothness of the risk measures. Since all norms on finite-dimensional spaces are equivalent, establishing that \mmode{\beta, d < +\infty} with respect to the same norm is not restrictive, even though this may give rise to dimension-dependent constants~\mmode{\beta} and \mmode{d}.

\subsection{Frechet derivatives}
\label{subsection:frechet-derivatives}
An important observation to be made is that the \mmode{\Q}-directional derivative \mmode{ \RPgrad{\P}{\Q}} of regular-risk measures is {\em affine} in \mmode{\Q}.\footnote{A function \mmode{f: \mathcal{S} \rightarrow \R{}} is said to affine if \mmode{f(x + \theta(y - x)) = f(x) + \theta (f(y) - f(x))} for every \mmode{x,y \in \mathcal{S}} and \mmode{\theta \in [0,1]}.} In principle, without any further assumptions on the risk measure \mmode{R}, its \mmode{\Q}-directional derivative need not be affine in \mmode{\Q}. Counter examples of functions that have well defined G-derivatives in all directions but that are nonlinear in the direction exist even among functions defined on \mmode{\R{2}}, let alone the infinite-dimensional setting of \mmode{\amb}. A sufficient condition for the directional derivative to be affine in \mmode{\Q} is the existence of a stronger notion of derivative called the \emph{Frechet-derivative} or \emph{F-derivative}. 
\begin{definition}[Frechet-derivative]
The Frechet(F)-derivative of the risk measure \mmode{R} at \mmode{\P \in \amb} associated with  a given norm \mmode{\pnorm{\cdot}{\amb}} on \mmode{\amb}, is a function \mmode{\ell_{\P} : \Xi \rightarrow \R{}} such that the mapping \mmode{\amb \ni \P' \mapsto \EE{\P'} [\ell_{\P}(\xi)] } is continuous w.r.t.\,\mmode{\pnorm{\cdot}{\amb}} and satisfies
\begin{equation}
\label{eq:frechet-derivative}
0 = \lim\limits_{\pnorm{\P' - \P}{\amb} \downarrow 0} \frac{ \RP{\P'} - \RP{\P} - \EE{\P' - \P} [\ell_{\P}(\xi)] }{\pnorm{\P' - \P}{\amb}} .
\end{equation}
\end{definition}

\paragraph{{\bf Smoothness with F-derivatives}} If the F-derivative \mmode{\ell_{\P}} were to exist at every \mmode{\P \in \amb}, the canonical notion of smoothness \eqref{eq:continuous-partial-derivatives} can be naturally extended into the infinite-dimensional setting. The risk measure \mmode{R} is \emph{F-smooth} if there exists some \mmode{\beta \geq 0} such that its F-derivative \mmode{\ell_{\P}} satisfies
\begin{equation}
\label{eq:frechet-smoothness}
\text{F-smoothness: } \quad \pnorm{\ell_{\P} - \ell_{\P'}}{\amb\opt} \symbolspace{\leq}{\;} \beta \pnorm{\P - \P'}{\amb} \quad \text{for all } \P,\P' \in \amb ,
\end{equation}
where \mmode{\pnorm{\cdot}{\amb\opt}} is the dual norm of \mmode{\pnorm{\cdot}{\amb}}. If the risk measure in \eqref{eq:main-opt} is F-smooth, most of the convergence analysis due to smoothness in finite-dimensional convex problems simply carries forward to the infinite-dimensional setting right away. This provides a natural recipe to extend the FW algorithm into the infinite-dimensional setting of probability spaces and establish their convergence under F-smoothness. F-derivative based FW-algorithms in probability spaces have already been studied in the literature \cite{kent2021frank} with a slightly more general notion of smoothness than \eqref{eq:frechet-smoothness}, and for a slightly more general class of risk measures than the usual concavity assumption. Our approach differs from \cite{kent2021frank} in the fact that we only make use of G-derivatives in both the development of the FW-algorithm and also in establishing its convergence based on only G-derivative based regularity conditions, which are simpler to deal with than F-derivatives.

\paragraph{{\bf Comparison with G-derivatives}}
It is not necessary for a function to have F-derivatives even if it has affine directional derivatives in all directions, such counterexamples exist even in a finite-dimensional setting. On the contrary, if the risk measure \mmode{R} has a well-defined F-derivative \mmode{\ell_{\P}}, then it can be shown that its \mmode{\Q}-directional derivatives also exist in all directions \mmode{\Q \in \amb}. This is easily seen by considering \mmode{\P' = \P + \stepsize (\Q - \P ) } for \mmode{\stepsize \in [0,1]} in the definition \eqref{eq:frechet-derivative} and observing that \mmode{\EE{\P' - \P} [\ell_{\P} (\xi)] = \stepsize \EE{\Q - \P}[\ell_{\P} (\xi)] }. When \mmode{\Q \neq \P}, the limit \mmode{\pnorm{\P' - \P}{\amb} \downarrow 0} is achieved if and only if \mmode{\stepsize \downarrow 0}. Then it follows from \eqref{eq:frechet-derivative} that
\[
0 = \lim_{\stepsize \downarrow 0} \; \frac{\RP{\pg} - \RP{\P} - \stepsize \EE{\Q - \P} [\ell_{\P}(\xi)]}{\stepsize \pnorm{\Q - \P}{\amb}} = \frac{1}{\pnorm{\Q - \P}{\amb}} \Big( \RPgrad{\P}{\Q} - \EE{\Q - \P} [\ell_{\P}(\xi)] \Big) .
\]
Since \mmode{\Q \neq \P}, we conclude that the \mmode{\Q}-directional derivative exists and \mmode{\RPgrad{\P}{\Q} \, = \, \EE{\Q - \P} [\ell_{\P}(\xi)]}. Moreover, this equality holds even if \mmode{\Q = \P} since \mmode{ \RPgrad{\P}{\P} = 0 = \EE{\P - \P} [\ell_{\P} (\xi)]}.

The notion of F-derivative relies heavily on the underlying metric structure on \mmode{\amb}, whereas, the notion (Definition \ref{def:directional-derivative}) of G-derivatives is independent of it. To compare, for a G-derivative to exist along a given direction, it is only required for the limit in \eqref{eq:derivative-definition} to exist. However, the existence of an F-derivative requires that the limit in \eqref{eq:derivative-definition} is achieved uniformly over all possible directions.

\paragraph{{\bf Challenges with F-derivatives}} Even though the notion of F-smoothness \eqref{eq:frechet-smoothness} is a natural extension of canonical smoothness \eqref{eq:continuous-partial-derivatives} in an infinite-dimensional setting, working with F-derivatives is potentially challenging due to the following reasons:

\begin{enumerate}[label = {\rm (\roman*)}, itemsep = 0mm, topsep = 0mm, leftmargin = *]
\item \textbf{Existence}: It is a stronger requirement that the limit in \eqref{eq:frechet-smoothness} converges uniformly (w.r.t.\,\mmode{\pnorm{\cdot}{\amb}}) in all directions, which often implies that an F-derivative might not even exist.

\item \textbf{Finite representability}: An F-derivative is a function \mmode{\ell_{\P}: \Xi \rightarrow \R{}} which is an infinite-dimensional object, so apriori, it is not clear as to whether it can be characterized in terms of a few finite-dimensional quantities.

\item \textbf{Norm consistency}: Most importantly, it is often very difficult to establish the smoothness condition \eqref{eq:frechet-smoothness} of the F-derivatives in a specific norm. To elaborate further, we know that the FW algorithm in a finite-dimensional convex problem converges sub-linearly, if the feasible set is bounded and the objective function is smooth. Since all norms on finite-dimensional vector spaces are equivalent, the choices of norms for establishing the smoothness of the objective function, and the boundedness of the feasibility set are irrelevant (even though this could potentially give rise to dimensionally dependent constants). However, since no such equivalence exists between norms on an infinite-dimensional space, it becomes then necessary that the risk measure is F-smooth w.r.t.\,the same norm under which the ambiguity set is bounded, which is a much stronger condition to expect.
\end{enumerate}

We close this discussion by providing an example wherein a risk measure has directional derivatives and is G-smooth, yet its F-derivatives do not exist. In particular, we argue that all RR measures that satisfy conditions of Lemma~\ref{lemma:RR-smoothness} are G-smooth, whereas their F-derivative does not exist if the corresponding~\mmode{L} function (i.e., sufficient statistic) is discontinuous.

\begin{example}[G-derivative vs F-derivative]\label{example:G-F}\rm{
Suppose the support set is~\mmode{\Xi = [-1, +1]}, the ambiguity set~\mmode{\amb} contains all possible distributions supported on~\mmode{\Xi} and it is equipped with a Wasserstein metric. Let the regular risk be defined by the sign function~\mmode{L(\xi) = \sgn (\xi)} (with the convention that \mmode{\sgn (0) = 0}) and \mmode{ r(z) = z}, i.e., \mmode{\RP{\P} = \EE{\P} [\sgn (\xi)]}.

\begin{enumerate}[label = {\rm (\alph*)}, itemsep = 0mm, topsep = 0mm, leftmargin = *]
\item {\em Existence of G-smoothness:} We conclude from Lemma \ref{lemma:RR-directional derivatives} that the G-derivatives of the risk measure exist and are given by \mmode{\RPgrad{\P}{\Q} = \EE{\Q - \P} [\sgn (\xi)]}. Moreover, since the support \mmode{\Xi} is bounded, condition (i) of Lemma \ref{lemma:RR-smoothness} is satisfied with \mmode{d = 2}, and since \mmode{\nabla r(z) = 1} for all \mmode{z \in [-1, +1]}, condition (ii) of Lemma \ref{lemma:RR-smoothness} is also satisfied with \mmode{\beta = 0}. Consequently, the given regular risk is \mmode{0}-smooth in the sense of Definition \ref{def:smoothness}.

\item {\em Non-existence of F-derivatives:} Since \mmode{\RPgrad{\P}{\Q} = \EE{\Q - \P} [\sgn (\xi)]}, if the F-derivative of the risk measure \mmode{R} were to exist, the mapping \mmode{\Q \longmapsto \EE{\Q - \P} [\sgn (\xi)] } must be continuous. However, for the sequence of distributions \mmode{\Q_n (\xi) = \delta (\xi - \nicefrac{1}{n})}, we see that \mmode{ \EE{\Q_n - \P} [\sgn (\xi)] = 1 - \EE{\P} [\sgn (\xi)] } for all \mmode{n = 1,2,,\ldots}, and they converge to the distribution \mmode{\Q = \delta(\xi)} in the Wasserstein metric, for which we have \mmode{ \EE{\Q - \P} [\sgn (\xi)] = - \EE{\P} [\sgn (\xi)] \neq 1 - \EE{\P} [\sgn (\xi)] = \EE{\Q_n - \P} [\sgn (\xi)] }. Therefore, the F-derivative of the risk measure does not exist.
\end{enumerate}}
\end{example}

More sophisticated examples of risk measures and ambiguity sets can be constructed following the same underlying idea of discontinuity of function \mmode{L}. Example~\ref{example:G-F} highlights the relevance of different notions of derivatives and the resulting regularity. The existence of F-derivatives and associated smoothness requires the derivative object to be continuous w.r.t.\,more variations of distributions whereas the notion of G-derivatives and smoothness considers variations only along a line joining any pair of distributions. Luckily, since the FW algorithm operates by taking convex combinations at each iteration, we only require the latter bounds which are considerably less restrictive.

\subsection{G-derivatives in non-flat spaces}
In the following, we consider a slightly general setting wherein the ambiguity set \mmode{\amb} may not be convex, e.g., when \mmode{\amb \subset \mathcal{N}(\R{n})} is the set of $n$ multivariate Gaussian distributions. In such a setting, given two distributions \mmode{\P, \Q \in \amb}, the line joining them  \mmode{\pg = \P + \stepsize (\Q - \P)} is not contained in~\mmode{\amb}. Even though such examples cannot be directly handled in our setting, it turns out that a slight modification of the definition of the derivative allows us to take care of such scenarios. To this end, we assume that the ambiguity set~\mmode{\amb} is equipped with
\begin{enumerate}[label = {\rm (\alph*)}, itemsep = 0mm, topsep = 0mm, leftmargin = *]
\item {\em Metric}: \mmode{d : \amb \times \amb \longrightarrow [0, +\infty[}

\item {\em Geodesics}: For any \mmode{\P , \Q \in \amb}, there exists a parametric curve \mmode{[0,1] \ni \stepsize \longmapsto \P^d_{\stepsize} \in \amb} such that \mmode{\P^d_0 = \P}, \mmode{\P^d_1 = \Q}, and \mmode{d(\pg^d, \P^d_{\stepsize'}) = \abs{ \stepsize - \stepsize' } d(\P, \Q) }.
\end{enumerate}
For example, if \mmode{\amb = \mathcal{N} (\R{n})}, then each distribution is uniquely identified by its first and second moments \mmode{\mu} and \mmode{\Sigma} respectively. Then, for given two distributions \mmode{\P = \mathcal{N} (\mu , \Sigma)} and \mmode{\P' = \mathcal{N} (\mu' , \Sigma')}, we define the metric \mmode{ d (\P , \P') \define \sqrt{ \pnorm{\mu - \mu'}{2}^2 + \pnorm{\Sigma - \Sigma'}{F}^2 } }, and the associated geodesic is \mmode{ \pg^d = \mathcal{N} (\mu_{\stepsize} , \Sigma_{\stepsize}) }, where  \mmode{(\mu_{\stepsize} , \Sigma_{\stepsize}) = (1 - \stepsize) (\mu, \Sigma) + \stepsize (\mu' , \Sigma')}.

{\bf G-derivatives and smoothness.} Given a geodesic structure on \mmode{\amb}, we define the associated G-derivatives by
\[
\RPgrad{\P}{\Q} \define \lim_{\stepsize \downarrow 0} \frac{1}{\stepsize} \left( \RP{\pg^d} - \RP{\P} \right) .
\]
Observe that the above definition is slightly different from that of \eqref{eq:derivative-definition} where the convex combination \mmode{\P_{\stepsize}} of distributions \mmode{\P} and \mmode{\Q} is replaced with its geodesic counterpart. Consequently, we can also define G-smoothness on non-flat spaces in a geodesic sense similar to Definition \ref{def:smoothness}. We say that the risk measure \mmode{R} is \mmode{\smooth}-smooth if there exists some \mmode{\smooth \geq 0} such that the inequality holds
\[
\RPgrad{\P}{\pg^d} + \RPgrad{\pg^d}{\P} \leq \stepsize^2 \smooth , \quad \text{for any \mmode{\P, \Q \in \amb} and \mmode{\stepsize \in [0,1]}.}
\]

{\bf Geodesic concavity}. We can define a notion of concavity using the geodesic structure that is slightly more general than the usual notion. We say that a risk measure \mmode{R} is geodesically concave on \mmode{\amb} if the mapping \mmode{[0,1] \ni \stepsize \longmapsto \RP{\pg^d} } is concave for every \mmode{\P, \Q \in \amb }. It turns out that both of the results from Proposition \ref{proposition:derivative-properties} hold true for derivatives defined in a geodesic sense. This allows us to consider DRO problems like 
\[
\min_{x \in \setx} \sup_{\P \in \mathcal{N}} \EE{\P} [f(x, \xi)] - \lambda G(\P, \Pref) ,
\]
where \mmode{G(\P, \Pref)} is the Gelbrich distance (or any other moment-based distance). We wish to emphasize that the FW algorithm must also be adapted to work with the geodesic derivatives instead.

\section{The Frank-Wolfe Algorithm}
\label{section:FW-algorithm}
Given the notion of G-derivative as in Definition \ref{def:directional-derivative}, the Frank-Wolfe (FW) algorithm for \eqref{eq:main-opt} is an iterative procedure that involves solving the optimization problem
\begin{equation}
\label{eq:main-FW}
\sup_{\Q \in \amb} \; \RPgrad{\P}{\Q} ,
\end{equation}
\begin{wrapfigure}{r}{0.4\linewidth}
    \centering
    \includegraphics[width=0.38\textwidth]{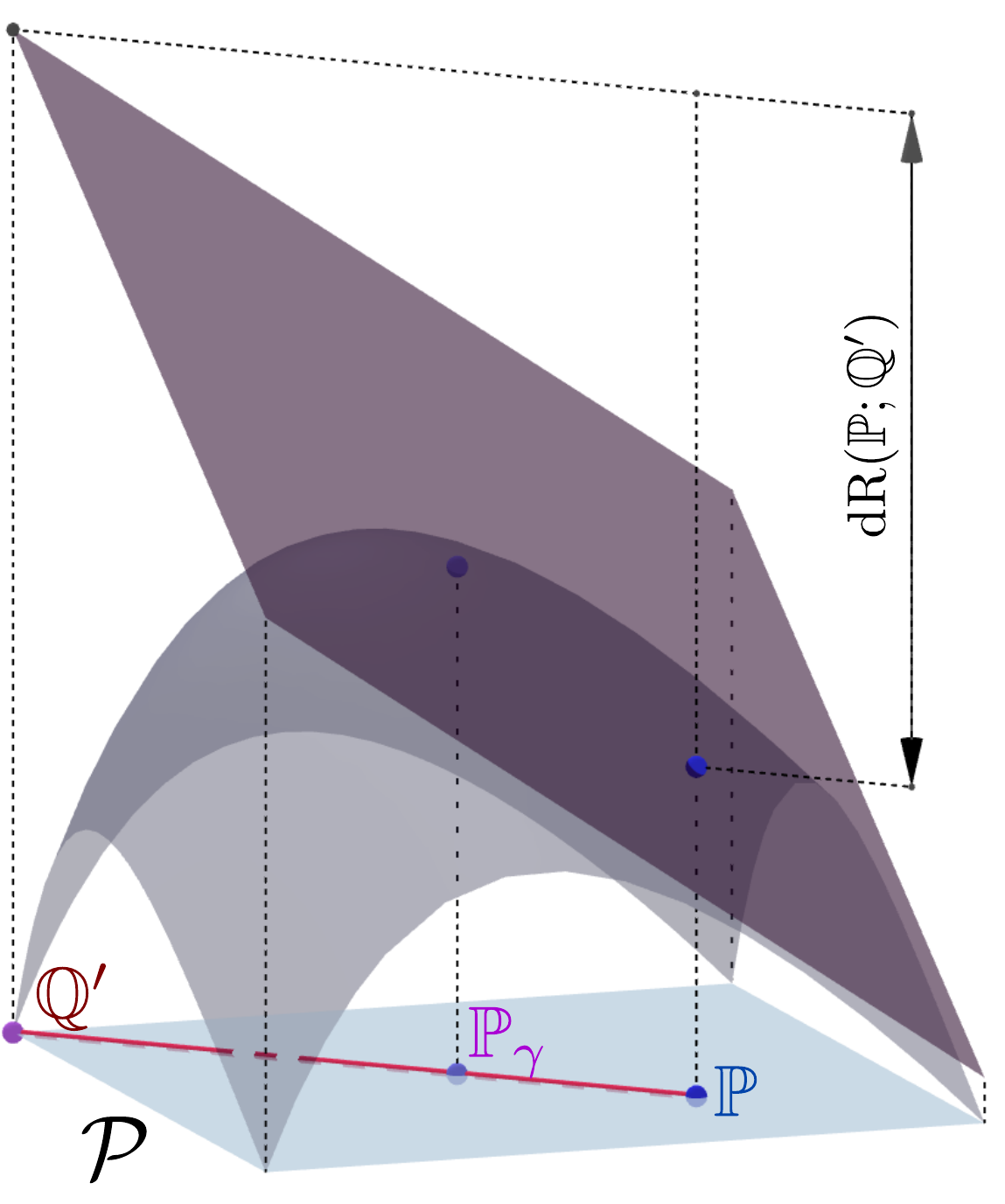}
    \caption{FW-oracle.} 
    \label{fig:single-FW-update}
\end{wrapfigure}
for a given \mmode{\P \in \amb}, at each iteration. The FW-problem \eqref{eq:main-FW} is linear if and only if the G-derivative \mmode{\RPgrad{\P}{\Q}} is affine in \mmode{\Q}, for every \mmode{\P \in \amb}, translating the problem into the linear DRO class in~\eqref{LDRO}. This is indeed the case for many interesting risk measures as seen in Example \ref{example:RR-directional-derivatives}. Moreover, in all finite-dimensional optimization problems, the corresponding FW-problem \eqref{eq:main-FW} is always linear, which need not be the case for a generic infinite-dimensional optimization problem like \eqref{eq:main-opt}.

Implementation of the FW algorithm only requires a well-defined notion of the G-derivative, and as seen in Lemma~\ref{lemma:RR-directional derivatives}, such objects can be computed by means of a few finite-dimensional quantities in several problems of interest. Moreover, even though the \emph{FW-problem}~\eqref{eq:main-FW} is infinite-dimensional in nature, it admits tractable finite-dimensional convex reformulations for many relevant applications similar to linear DROs~\cite{kuhn2019wasserstein}. This is a compelling reason to investigate FW methods for optimization problems over probability distributions, particularly in the case of nonlinear risk measures as in~\eqref{eq:main-opt}.

\subsection{FW oracle}
The FW-oracle is a set-valued mapping \mmode{\mathcal{F} : \amb \times [0,1] \rightarrow 2^\amb } defined as
\begin{equation}
\label{eq:FW-oracle}
\fworacle{\P}{\stepsize} \define \left\{ \Q' \in \amb \text{ such that } \sup_{\Q \in \amb} \; \RPgrad{\P}{\Q} \leq \stepsize \delta \smooth + \RPgrad{\P}{\Q'} \right\}.
\end{equation}
Given the current iterate \mmode{\P \in \amb}, the FW algorithm involves solving the FW-problem \eqref{eq:main-FW} to obtain its approximate solution \mmode{\Q' \in \amb}, then the current iterate \mmode{\P} is updated by moving it towards \mmode{\Q'} as shown in Figure~\ref{fig:single-FW-update}. 

\paragraph{{\bf Additive accuracy of the oracle}} The parameter \mmode{\delta > 0} is an arbitrary positive number signifying the accuracy of the FW oracle, and \mmode{\smooth} is the smoothness constant of the risk measure \mmode{R} as in Definition \ref{def:smoothness}. It must be observed that in an iterative scheme to solve \eqref{eq:main-opt}, it is typical that the stepsize sequence (\mmode{\stepsize_k}) is monotonically decreasing and converges to \mmode{0}. Therefore, it is also required that the FW-oracle solves the sub-problems \eqref{eq:main-FW} up to a greater precision as the iterations progress.

\paragraph{{\bf The Frank-Wolfe-gap}} We refer to the quantity \mmode{\sup_{\Q \in \amb} \; \RPgrad{\P}{\Q}} as the \emph{Frank-Wolfe(FW)-gap} at \mmode{\P}, and it is crucial in defining aposteriori stopping criteria for the FW-algorithm. Along with the distribution \mmode{\Q'}, we assume that the FW-oracle also provides access to the quantity \mmode{ \RPgrad{\P}{\Q'}}, which is an approximate of the FW-gap at \mmode{\P}.

\begin{assumption}[Accessibility of FW oracle]
\label{a:FW trac}
We assume that the FW oracle~\eqref{eq:FW-oracle} is computationally available, i.e., there exists a tractable approach to find a feasible solution from the set~\eqref{eq:FW-oracle}.
\end{assumption}

To develop our algorithm for NDRO problems, we assume the accessibility of the FW oracle in Assumption~\ref{a:FW trac}. However, given a DRO problem at hand, one needs to ensure that the corresponding FW oracle is indeed computationally available. When the directional derivatives are linear in \mmode{\Q}, the corresponding FW problem is a linear worst-case distribution problem, akin to the linear DRO problems~\eqref{LDRO}. The properties of the loss function~$\ell(\xi)$ (suppressing the decision variable~$x$ in~\eqref{LDRO}) under which the respective linear DRO enjoys a tractable reformulation have been extensively studied in the literature. Next, we provide an example of this kind.  

\begin{remark}[Tractable FW oracle]
\label{rem:FW trac}
For the RR measures in Definition~\ref{def:RR-measure}, the FW oracle is a lienar DRO~\eqref{LDRO} with the loss function~$\ell(\xi) = \inprod{ \nabla \risk (\EE{\P} [L(\xi)]) }{  L(\xi)}$ (cf. \eqref{eq:RR-directional-derivative}). If this loss can, for instance, be described as a sum of pointwise maximum of concave functions (i.e., $\ell(\xi) = \sum_{t\le T} \max_{k \le K}\ell_{tk}(\xi)$), we know that the linear DRO~\eqref{LDRO} under the Wasserstein ambiguity set~$\amb$ and the norm transportation cost~$\|\cdot\|$ has a tractable reformulation~\cite[Theorem~6.1]{mohajerin2018data}. In this light, a sufficient condition for the RR measure is when each element of the sufficient statistic vector~$L(\xi)$ constitutes a pointwise maximum of finitely many concave functions and $\nabla r \ge 0$. 
\end{remark}

We note that the sufficient statistics of the variance and entropic risks in Example~\ref{example:RR-measures} do not meet the piecewise concavity condition of Remark~\ref{rem:FW trac}; they are convex quadratic and exponential function, respectively, over the desired support sets. It is interesting to see that in the case of entropic risk, the positivity $\nabla r \ge 0$ is, however, fulfilled (see~Example~\ref{example:RR-directional-derivatives}(b)). Besides the tractability setting of Remark~\ref{rem:FW trac}, the literature includes various other combinations of ambiguity sets and functionals of probability distributions supported on finite support~\cite{postek2016computationally}, the moment-based ambiguity sets~\cite{delage2010distributionally, ref:goh2010distributionally, ref:wiesemann2014distributionally}, and the metric-based ambiguity set~\cite{gao2022wasserstein, kuhn2019wasserstein}, to name but a few.

\begin{lemma}[FW-one-step-bounds]
\label{lemma:FW-one-step-bounds}
Consider \eqref{eq:main-opt} with a risk measure that is \mmode{\smooth}-smooth in the sense of Definition~\ref{def:smoothness} for some \mmode{\smooth \geq 0}, and let \mmode{R\opt} be its optimal value. Let \mmode{\mathcal{F}} be the  corresponding FW-oracle as in \eqref{eq:FW-oracle} with an arbitrary accuracy parameter \mmode{\delta \geq 0}. For any \mmode{\P \in \amb}, \mmode{\stepsize \in [0,1]}, and \mmode{\Q' \in \fworacle{\P}{\stepsize}}, let \mmode{ \P_{\stepsize} = \P + \stepsize (\Q' - \P) } be the one-step-ahead FW update from \mmode{\P} with a stepsize of \mmode{\stepsize}. Then we have
\begin{equation}
\label{eq:fw-one-step-ineq}
R\opt - \RP{\P} \symbolspace{\leq}{\;} \sup_{\Q \in \amb} \; \RPgrad{\P}{\Q} \symbolspace{\leq}{\;} \frac{1}{\stepsize} \Big( \RP{\P_{\stepsize}} - \RP{\P} \Big) + \stepsize \smooth (1+\delta) .
\end{equation}
\end{lemma}
Rearranging the inequality \eqref{eq:fw-one-step-ineq}, the one-step improvement in sub-optimality is seen to be
\begin{equation}
\label{eq:fw-one-step-sub-optimality}
R\opt - \RP{\P_{\stepsize}} \symbolspace{\leq}{\;} \stepsize^2 \smooth (1 + \delta) + (1 - \stepsize) \big( R\opt - \RP{\P} \big) .
\end{equation}

\subsection{FW convergence guarantees}
The FW-algorithm seeks to solve the optimization problem \eqref{eq:main-opt} by iteratively solving the FW-problem \eqref{eq:main-FW} using a FW-oracle as in \eqref{eq:FW-oracle}. More precisely, given a step size sequence \mmode{(\stepsize_k)_k \subset [0,1] } and a distribution \mmode{\P_0 \in \amb}, the FW-algorithm generates a sequence of distributions \mmode{(\P_k)_k} such that 
\begin{equation}
\label{eq:FW-iterations-main}
\text{FW-algorithm:} \quad
\P_{k+1} = \P_k + \stepsize_k ( \Q_k - \P_k ),\quad \text{where} \quad \Q_k \in \fworacle{\P_k}{\stepsize_k} .
\end{equation}
It must be observed that the implementation of the FW algorithm \eqref{eq:FW-iterations-main} does not depend on the choice of a norm on the ambiguity set \mmode{\amb}. Therefore, it is desirable to have a \emph{norm-independent} analysis of the FW algorithm. To this end, we take inspiration from \cite{jaggi2013revisiting}, which has a similar analysis for the convergence of FW-algorithm for finite-dimensional problems by working with the notion of \emph{curvature co-efficient} instead of the canonical smoothness \eqref{eq:canonical-smoothness} used in \cite{frank1956algorithm}. It turns out that the notion of smoothness defined in Definition \ref{def:smoothness} is also amenable to a similar analysis of the FW-algorithm; with the advantage that \eqref{eq:G-smoothness} is more in-line as a generalized notion of smoothness from \eqref{eq:canonical-smoothness} and \eqref{eq:holder-smoothness}.

\begin{proposition}[Apriori bounds]
\label{proposition:FW-convergence} 
Consider \eqref{eq:main-opt} with a risk measure that is \mmode{\smooth}-smooth in the sense of Definition \ref{def:smoothness} for some \mmode{\smooth \in [0, +\infty)}, and let \mmode{R\opt} be its optimal value. For \mmode{k = 1,2,\ldots,} let \mmode{(\P_k)_k} be the sequence of distributions obtained from the FW-algorithm \eqref{eq:FW-iterations-main} with a step size sequence \mmode{\stepsize_k = \frac{2}{k+2}} and some \mmode{\P_0 \in \amb}. Then
\begin{equation}
\label{eq:fw-primal-convergence}
R \opt - \RP{\P_k} \symbolspace{\leq}{\;} \frac{4\smooth}{k + 2} (1 + \delta) \quad \text{ for all } k \geq 1 .
\end{equation}
\end{proposition}

Regarding the FW algorithm~\eqref{eq:FW-iterations-main}, an important observation to be made is that the ``complexity" of the distribution~$\P_k$ may, in general, increase with each update of $FW$-iteration. To make this more clear, suppose that the distributions of $\P_k$ and $\Q_k$ are both discrete with potentially different support sets. Then, it is straightforward to see that the support of the next iteration~$\P_{k+1}$ is the union of the two supports, and as such, its cardinality increases in each iteration. However, this issue can be avoided in the case of regular risk measures.  

\begin{remark}[Sufficient statistics \& reduced FW update]
\label{rem:suff-stat}
When the risk measure is regular in the sense of Definition~\ref{def:RR-measure} (i.e., $\RP{\P} = \risk \big( \EE{\P} [L(\xi)] \big)$), we recall from \eqref{eq:RR-directional-derivative} that the G-derivative $\RPgrad{\P}{\Q}$ is characterized entirely in terms of the finite-dimensional sufficient statistic vector~$\EE{\P} [L(\xi)]$. The FW algorithm~\eqref{eq:FW-iterations-main} reduces to the finite-dimensional update 
\begin{align}
\label{FW-update-reduced}
\mu_{k+1} = \mu_k + \stepsize_k \big( \nu_k - \mu_k \big), \quad \text{where} \quad \nu_k \in \big\{ \EE{\Q} [L(\xi)] \in \R{m}~:~ \Q \in \fworacle{\P}{\stepsize} \big\}
\end{align}
in which~$\fworacle{\P}{\gamma}$ is the FW-oracle~\eqref{eq:FW-oracle}, and the initial condition is~\mmode{\mu_0 = \EE{\Pref} [L(\xi)]}. When the FW-oracle~\eqref{eq:FW-oracle} is computationally available~(e.g., the tractable cases in Remark~\ref{rem:FW trac}), it suffices to follow the reduced finite-dimensional update rule~\eqref{FW-update-reduced}, instead of the infinite-dimensional update~\eqref{eq:FW-iterations-main}. In other words, the worst-case risk~\eqref{eq:main-opt} can be solved via $\sup_{\P \in \amb} \RP{\P} = \risk(\mu\opt)$, where $\mu\opt$ is the convergent point of the reduced FW-algorithm in~\eqref{FW-update-reduced}. Thus, the finite vector~$\EE{\P} [L(\xi)]$ is indeed ``sufficient" for computing the worst-case risk via the FW algorithm.
\end{remark}
We note that the sufficient statistics in Remark~\eqref{rem:suff-stat} becomes more involved when an additional decision such as $x$ in \eqref{NDRO} must be determined. This shall be addressed in the next section.

\subsection{FW-gap based termination and aposteriori bounds}
Suppose that we know some \mmode{\smooth' \geq 0} that satisfies the smoothness condition \eqref{eq:G-smoothness}. Then for any given \mmode{\eps>0}, if the FW algorithm is run for \mmode{K \geq \ceil*{ \frac{4\smooth' (1 + \delta)}{\eps} } - 2 } iterations with the stepsize sequence \mmode{\stepsize_k = \frac{2}{k + 2}}, then the last iterate \( \P_K \) is guaranteed to be \mmode{\eps} sub-optimal in objective value. Thus, in principle, it suffices to only know some upper bound \mmode{\smooth' \geq \smooth} for the smoothness constant. However, if finding \mmode{\smooth} exactly is challenging, and the known upper bound \mmode{\smooth'} is not tight; the theoretically guaranteed number of iterations required for \mmode{\eps}-sub-optimality may not be practical. In such a setting, it turns out that the optimal value of the FW-problem \eqref{eq:main-FW} called the \emph{FW-gap} provides a good measure to define aposteriori stopping criteria. Moreover, we will also see later for NDRO problems that terminating the FW algorithm when the FW-gap is small provides worst-case performance bounds in the context of DRO problems. We follow the analysis of \cite{jaggi2013revisiting,clarkson2010coresets} by considering the FW algorithm under two regimes of stepsize sequence to obtain provable upper bounds on the FW-gap towards later iterations.

Since the oracle employed to solve the linear minimization sub-problems at each iteration is only accurate to some specified precision, the actual value of the FW-gap \mmode{\sup_{\Q \in \amb} \, \RPgrad{\P}{\Q}} is never known exactly. However, at each iteration \mmode{k}, the FW-oracle does provide its approximate estimate \mmode{ \RPgrad{\P}{\Q_k}}, which satisfies the inequality \mmode{\sup_{\Q \in \amb} \, \RPgrad{\P}{\Q} \leq \delta \stepsize_k \smooth + \RPgrad{\P}{\Q_k}}. Thus, for a given value of \mmode{\eps > 0}, we terminate the FW procedure by examining the quantity \mmode{\RPgrad{\P}{\Q_k}} such that the desired FW-gap inequality: \mmode{\sup_{\Q \in \amb} \, \RPgrad{\P}{\Q} \leq \eps}, is satisfied.

\begin{proposition}[Aposteriori bounds]
\label{proposition:FW-gap}
Consider \eqref{eq:main-opt} with a risk measure that is \mmode{\smooth}-smooth in the sense of Definition \ref{def:smoothness} for some \mmode{\smooth \geq 0}, and let \mmode{R\opt} be its optimal value. Let \mmode{K \geq 1} and let \mmode{(\P_k)_k} be the sequence obtained from the FW-algorithm \eqref{eq:FW-iterations-main} using a diminishing stepsize \mmode{\stepsize_k \define \frac{2}{k + 2}} for \mmode{k = 0,1,\ldots, K - 1}, and then a constant stepsize \mmode{\stepsize_k = \frac{2}{K + 2}} for \mmode{k = K, K+1, \ldots, 2K + 1}. Finally, let \mmode{\fwgap_k =  \RPgrad{\P}{\Q_k} } for \mmode{k = 1,2,\ldots, 2K + 1 } be the sequence of approximate FW-gaps. There exists \mmode{ \khat \in \{K, K+1, \ldots, 2K + 1\}} such that
\begin{equation}
\label{eq:FW-gap-sub-optimality}
R\opt - \RP{\P_{\khat}} \symbolspace{\leq}{\;} \sup_{\Q \in \amb} \RPgrad{\P_\khat}{\Q} \symbolspace{\leq}{\;} \frac{2\smooth (2 + 3 \delta)}{K + 2} ,
\end{equation}
and every such \mmode{\khat} is recognised by verifying the inequality \mmode{\fwgap_{\khat} \leq \frac{4\smooth}{K + 2} (1+\delta)}.    
\end{proposition}

\begin{remark}[Explicit error bounds]
\label{remark:FW-eps-stopping}
Consider the setting of Proposition \ref{proposition:FW-gap} with \mmode{K = \keps \define \ceil*{\frac{2\smooth}{\eps} (2 + 3\delta)} - 2} for an \mmode{\eps > 0}. Then, there exists \mmode{\khat \in \{ \keps, \keps + 1, \ldots, 2\keps + 1 \} } such that
\[
\sup_{\Q \in \amb} \ \RPgrad{\P_\khat}{\Q} \leq \eps ,
\]
and every such \mmode{\khat} is recognised by verifying the inequality
\mmode{\fwgap_{\khat} \leq \eps \frac{2 + 2\delta}{2 + 3\delta}}.
\end{remark}

\subsection{FW stepsize selection}
The last part of this section discusses two particular features of the stepsize rule (two-regimes and diminishing behavior) in Proposition~\ref{proposition:FW-gap}.  

\paragraph{\bf Two-regimes stepsize}
The two-regimes for \mmode{\stepsize_k} in Proposition~\ref{proposition:FW-gap} turns out to be crucial to obtain provable guarantees that the FW-gap is bounded above in the later iterations. Even though such certificates are of independent interest in their own right, having such upper bounds is also essential in the context of DRO problems. An upper bound on the FW-gap at iteration \mmode{k} ensures that the performance of the decision \mmode{x_k} for the worst-case distribution is not ``too-bad''.

\paragraph{\bf Different stepsize selection for the FW algorithm}
For generic feasible sets, the Frank-Wolfe algorithm requires that the stepsize sequence \mmode{(\stepsize_k)_k} be diminishing. Even though the risk measure is smooth, the direction \mmode{\Q'_k} may change dis-continuously around the optimal solution if the ambiguity set \mmode{\amb} has flat faces (like Wasserstein-\mmode{1} balls). Therefore, the Frank-Wolfe algorithm in general requires diminishing stepsize to converge, and might not converge with constant stepsize unlike gradient descent algorithms. Even with the standard diminishing stepsize of \mmode{\stepsize_k = \nicefrac{2}{k + 2}}, the canonical FW algorithm is plagued with the {\em zig-zag} phenomena where the iterates keep oscillating around the optimal solution. To remedy these challenges of FW in finite-dimensional convex problems, various adaptive stepsize sequences have been proposed in the literature with provably better guarantees under some additional assumptions. In the following, we describe the main ideas of these stepsize selection rules in the context of \eqref{eq:main-opt}.
\begin{enumerate}[label = {\rm (\alph*)}, itemsep = 0mm, topsep = 0mm, leftmargin = *]
    \item {\em Demyanov-Rubinov (DR) stepsize}: A very interesting stepsize selection is due to \cite{demyanov1970approximate, dem1968minimization}
    \begin{equation}
    \label{eq:Demyanov-Rubinov-stepsize}
    \stepsize_k = \min \left\{ \frac{\fwgap_k}{2 \smooth} , 1  \right\} ,
    \end{equation}
    which adaptively tunes the stepsize \mmode{\stepsize_k} using the current value of the FW-gap.

    \item {\em Backtracking}: Suppose, it is relatively easy to evaluate the risk measure \mmode{R}, then one can employ the backtracking based DR step size due to \cite{pedregosa2020linearly, nguyen2023bridging}. This stepsize selection rule adaptive tunes the smoothness constant \mmode{C} locally along the line segment \mmode{(\pk , \Q'_k)}. This allows the backtracking step size to take larger steps than other methods. Specific way to adaptively tune the smoothness constant may vary, the specific rule in \cite{pedregosa2020linearly} can be summarised as follows. For fixed constants \mmode{\eta \in (0, 1), \tau \geq 1}, at each iteration the stepsize rule is 
    \begin{equation}
    \label{eq:backtracking-stepsize}
    \begin{cases}
    \begin{aligned}
    \stepsize_k &= \min \left\{ \frac{\fwgap_t}{2 \tau^t \smooth_k} , 1 \right\} 
 \text{ for the largest \mmode{t = 1,2,\ldots} such that } \\
    \RP{\P_{k +1}} & \geq \RP{\pk} + \stepsize_k \fwgap_k - \stepsize^2 \tau^t \smooth_k .
    \end{aligned}
    \end{cases}
    \end{equation}
The smoothness constant is also updated as \mmode{\smooth_{k + 1} = \eta \tau^t \smooth_k}. Another variant \cite[Assumption 6.1, Algorithm 1] {nguyen2023bridging} of backtracking-based step size selection in the FW algorithm gives rise to linear/geometric convergence with further assumptions on the feasible set.

 \item {\em Exact line search}: If it is easier to optimize the risk measure \mmode{R} on the line segment \mmode{(\pk , \Q'_k)}, then one can also select the stepsize by exactly maximizing 
 \begin{equation}
\label{eq:exact-line-search-stepsize}
\stepsize_k = \argmax_{\stepsize \in [0,1]} \quad \RP{\pk + \stepsize (\Q'_k - \pk)} .
 \end{equation}
\end{enumerate}

\section{Nonlinear Distributionally Robust Optimization}
\label{section:NDRO}

Let us recall that a generic DRO problem is formulated as the min-max problem
\begin{equation}
\label{eq:NDRO}
F\opt \define \inf_{x \in \setx} \sup_{\P \in \amb} \FP{x}{\P}\,, 
\end{equation}
where \( \setx \subset \R{n} \) is a closed convex set, denoting the set of feasible decisions, and \( \amb \) denotes a given ambiguity set of distributions. A DRO problem is said to be \emph{feasible} if \mmode{F\opt < +\infty}, which happens if and only if there exists some \mmode{x \in \setx} such that \mmode{\sup_{\P \in \amb} \FP{x}{\P}  < +\infty}. Our objective is to develop a framework to solve a generic DRO problem \eqref{eq:NDRO}. Particularly, with emphasis on the case when \mmode{ \FP{x}{\P} } is nonlinear in \mmode{\P} for every \( x \in \setx \), to which we refer to \eqref{eq:NDRO} as a \emph{nonlinear distributionally robust optimization} (NDRO) problem.

The way the variable~$x$ enters the risk~$F$ in~\eqref{eq:NDRO} may have an impact on the scalability of the proposed FW~algorithm. In particular, in the case of the regular risks in Definition~\ref{def:RR-measure}, the important feature is whether the decision~$x$ influences the sufficient statistics of the risk (cf.\,Remark~\ref{rem:suff-stat}). This consideration leads to two classes of regular risk measures: 
\begin{align}\label{eq:NDRO-regular}
    \text{(i)} \, F(x,\P) = r\big(x, \EE{\P}[L(\xi)]\big) \qquad \text{and} \quad  
    \text{(ii)} \, F(x,\P) = r\big(\EE{\P}[L(x,\xi)]\big). 
\end{align}
It should be noted that class~(i) in \eqref{eq:NDRO-regular} is a special form of class~(ii). For regular risk measures in class~(i), the sufficient statistic~$\EE{\P}[L(\xi)]$ is {\it not} influenced by the decision \mmode{x}, whereas this is not the case for general regular risk measures in class~(ii). This subtle difference makes a significant impact on whether the FW~algorithm can benefit from the notion of sufficient statistic described in Remark~\ref{rem:suff-stat}. The variance and entropic risks investigated in this article are indeed different in view of this feature.

\begin{example}[Regular NDRO examples]
\label{example:NDRO}
\rm{
The analogous examples of Example~\ref{example:RR-measures} in the NDRO context are the following:
\begin{enumerate}[label = {\rm (\alph*)}, itemsep = 0mm, topsep = 0mm, leftmargin = *]
\item \emph{Variance}: The variance \mmode{\vrisk{x}{\P}} (cf.\,\eqref{eq:variance-L-r}) associated with the distribution \mmode{\P} and a decision \mmode{x} is  
\begin{equation}
\label{NDRO_var}
\vrisk{x}{\P} \define x\transp \big( \Sigma_{\P} - \mup{\P} \mup{\P}\transp \big) x , \quad \text{where}  \quad \Sigma_{\P} \define \EE{\P} [ \xi \xi\transp ]~ \text{and}~ \mup{\P} \define \EE{\P}[\xi].
\end{equation}

\item \emph{Entropic risk measure}: The entropic risk measure~\mmode{ \ERP{x,\P}} (cf.\,\eqref{eq:er-risk-L-r}) of a multivariate distribution \mmode{\P} and decision \mmode{x}. If \mmode{\P} is a distribution with marginals \mmode{\P_j} for \mmode{j = 1,2,\ldots,n}, (i.e., the \mmode{j}-th component \mmode{\xi_j} is \mmode{\P_j} distributed). For a given set of positive \emph{risk-aversion parameters} \mmode{(\theta_j)_{j = 1}^n$ in $(0, +\infty)}, the associated risk is defined as
\begin{equation}
\label{NDRO_entropic}
\ERP{x,\P} \define \summ{j = 1}{n}{\frac{1}{\theta_j} \log \left( \EE{\P_j} [e^{-\theta_j x_j\xi_j}] \right)}.
\end{equation}
\end{enumerate}}
It is straightforward to see that the NDRO variance risk~\eqref{NDRO_var} belongs to the class (i) in \eqref{eq:NDRO-regular}, and thus keeping its sufficient statistic (i.e., the first two moments) intact, while the NDRO entropic risk~\eqref{NDRO_entropic} belongs to the class (ii) in \eqref{eq:NDRO-regular} where the vector~$L(x,\xi)$ depends inseparably on~$x$. Looking ahead at \eqref{eq:min-variance-tractable-FW}, we see that the proposed FW Algorithm \ref{algo:NDRO-min-max} applied to variance risk measure simplifies to iterations over only sufficient statistic \mmode{(\mu, \Sigma)} and \mmode{x}.

\end{example}

\paragraph{{\bf NDRO Dual problem}} Associated with the DRO problem \eqref{eq:NDRO} is its dual problem:
\begin{equation}
\label{eq:NDRO-dual}
\text{Dual problem:} \quad \dualoptval \define \sup_{\P \in \amb} \inf_{x \in \setx} \FP{x}{\P} .
\end{equation}
In general, we have \emph{weak-duality}~\mmode{\dualoptval \leq \primaloptval} relating the optimal values of the primal problem \eqref{eq:NDRO} and its dual \eqref{eq:NDRO-dual}. If \mmode{\dualoptval = \primaloptval} specifically, we say that \emph{strong duality} holds between \eqref{eq:NDRO} and \eqref{eq:NDRO-dual}. Moreover, suppose the DRO problem \eqref{eq:NDRO} and its dual \eqref{eq:NDRO-dual} admit the solutions~\mmode{x\opt} and \mmode{\P\opt}, i.e.,
\[
x\opt \in \argmin_{x \in \setx} \sup_{\P \in \amb} \FP{x}{\P} \quad \text{and} \quad \P\opt \in \argmax_{\P \in \amb} \inf_{x \in \setx} \FP{x}{\P} ,
\]
Then, the pair \mmode{(x\opt , \P\opt)} is said to be a \emph{saddle point} solution to the problems \eqref{eq:NDRO} and \eqref{eq:NDRO-dual}, which is also characterized by the condition
\[
\max_{\P \in \amb} \FP{\xopt}{\P} \symbolspace{=}{\;} \FP{\xopt}{\popt} \symbolspace{=}{\;} \min_{x \in \setx} \FP{x}{\popt} .
\]
The existence of a saddle point is sufficient for strong duality to hold, however, it is not necessary. Therefore, whenever strong duality holds, we consider the slightly relaxed notion of an \mmode{\eps}-sub-optimal saddle points as a solution concept for \eqref{eq:NDRO}.

\begin{definition}[\mmode{\eps}-saddle point]
\label{def:dual-gap-eps-solution} \rm{
Given \mmode{\eps \geq 0}, a pair \mmode{\epssaddle \in \setx \times \amb } is an \mmode{\eps}-saddle point of the DRO problem \eqref{eq:NDRO}, if it satisfies 
\begin{equation}
\label{eq:NDRO-eps-solution}
    -\eps + \sup_{\P \in \amb} \FP{\xeps}{\P} \symbolspace{\leq}{\;} \FP{\xeps}{\peps} \symbolspace{\leq}{\;} \eps + \inf_{x \in \setx} \FP{x}{\peps} .
\end{equation}
}
\end{definition}
It is worth noting that if \mmode{\epssaddle} is an \mmode{\eps}-saddle point, then we have the inequalities:
\[
\begin{aligned}
\sup\limits_{\Q \in \amb} \FP{\xeps}{\Q} & \symbolspace{\leq}{\;} 2\eps + \inf\limits_{y \in \setx} \FP{y}{\peps} \symbolspace{\leq}{\;} 2\eps + \inf\limits_{y \in \setx} \sup_{\Q \in \amb} \FP{y}{\Q} \ &&= \ 2\eps + \primaloptval , \quad \text{and} \\
\inf\limits_{y \in \setx} \FP{y}{\peps} & \symbolspace{\geq}{\;} -2\eps + \sup_{\Q \in \amb} \FP{\xeps}{\Q} \symbolspace{\geq}{\;} -2\eps + \sup_{\Q \in \amb} \inf\limits_{y \in \setx} \FP{y}{\Q} &&= -2\eps + \dualoptval .
\end{aligned}
\]
In other words, if \mmode{(\xeps , \peps)} is an \mmode{\eps}-saddle point, then both \mmode{x_{\eps}} and \mmode{\peps} are at most \mmode{2\eps}-sub-optimal to the DRO problem \eqref{eq:NDRO} and its dual \eqref{eq:NDRO-dual} respectively. Consequently, the decision \mmode{\xeps} is guaranteed to be at most \mmode{2\eps} worse from the best decision that could have been made in a DRO framework.

Solving the minimization over \mmode{x} in the dual-problem \eqref{eq:NDRO-dual} results in a maximization problem (potentially nonlinear) over the distributions
\begin{equation}
\label{eq:ndro-maximization-problem}
\sup_{\P \in \amb} \RP{\P} , \quad \text{ where } \RP{\P} \define \inf\limits_{x \in \setx} \FP{x}{\P} .
\end{equation}
Denoting \mmode{x(\P) \define \argmin_{x \in \setx} \FP{x}{\P} } (whenever a minimizer exists), for every \mmode{\P \in \amb}, the proposed method to compute an \mmode{\eps}-saddle point of the DRO problem generates a sequence \mmode{(\xk, \pk)_k}, where \mmode{\xk \in \xp{\pk}} for each \mmode{k}, and the sequence of distributions \mmode{(\pk)_k} is obtained by applying the FW algorithm to the maximization problem \eqref{eq:ndro-maximization-problem}. If a pair \mmode{(x' , \P')} satisfies \mmode{\P' \in \argmax_{\P \in \amb} \; \RP{\P}} and \mmode{x' \in x(\P') }, then it is not guaranteed that \mmode{x' \in \argmin_{y \in \setx} \; \sup_{\P \in \amb} \; \FP{y}{\P} } unless \mmode{x(\P')} is unique. It so turns out that the regularity assumptions on \mmode{F} and \mmode{\amb}, required for the algorithm convergence, also ensure uniqueness.

\subsection{NDRO: continuity, derivatives and smoothness}
Our proposed method to solve the NDRO problem \eqref{eq:NDRO} by applying the FW algorithm to \eqref{eq:ndro-maximization-problem} requires that the risk measure \mmode{R} therein has well defined G-derivatives that are also smooth in the sense of Definition \ref{def:smoothness}. This is not guaranteed apriori. In the following, we impose some regularity assumptions on the function \mmode{F} that guarantee the required smoothness of the risk measure \mmode{R}.

\begin{assumption}[NDRO smoothness]
\label{assumption:NDRO-smoothness}
Let \mmode{ \pg \define \P + \stepsize \big( \Q - \P \big)} as in~\eqref{eq:derivative-definition} and denote \mmode{\fxp{x}{\cdot} \define \FP{x}{\cdot} }, for every \mmode{x \in \setx}. We assume that there exists positive constants \mmode{\alpha, \smooth_1, \smooth_2 } such that
\begin{enumerate}[label = {\rm (\roman*)}, itemsep = 0mm, topsep = 0mm, leftmargin = *]

\item {\it Continuous function:} The mapping \mmode{\setx \times [0,1] \ni (x,\stepsize) \mapsto \FP{x}{\pg}} is proper, convex-concave, and continuous for all \mmode{\P, \Q \in \amb}.

\item \emph{Continuous derivatives}: The function \( \fxp{x}{\P} \) is directionally differentiable on \( \amb \), and the G-derivative \mmode{ \FPgrad{x}{\P}{\Q} } is \mmode{\smooth_1}-Lipschitz continuous in \mmode{x} uniformly over \mmode{\P , \Q \in \amb}, i.e.,
\begin{equation}
\label{eq:continuous-partial-derivatives}
\FPgrad{x}{\P}{\Q} - \FPgrad{y}{\P}{\Q} \symbolspace{\leq}{\;} \smooth_1 \norm{x - y}, \qquad \forall\, x,y \in \setx \quad \forall\, \P,\Q\in \amb. 
\end{equation}

\item \emph{Smoothness}: The function \( \fxp{x}{\P} \) is \mmode{\smooth_2}-smooth in the sense of Definition \ref{def:smoothness}, uniformly over \( x \in \setx \).

\item \emph{Strong convexity}: The function \( \FP{x}{\P} \) is \( \alpha \)-strongly convex in \mmode{x} w.r.t.\,the norm \mmode{\norm{\cdot}}, uniformly over all \( \P \in \amb \), i.e.,
\begin{equation}
\label{eq:strong-convexity-assumption}
\frac{\alpha}{2} \norm{x-y}^2 \symbolspace{\leq}{} \FP{y}{\P} - \FP{x}{\P} - \inprod{ \nabla_1 \FP{x}{\P}}{y-x}, \quad \forall\, x,y \in \setx, \quad \forall\, \P,\Q\in \amb.\footnote{\mmode{\nabla_1 F(x,\P) \define \big( \nicefrac{\partial F}{\partial x} \big) (x, \P)} denotes the partial derivative of \mmode{F} w.r.t. \mmode{x} evaluated at \mmode{(x,\P)}. }
\end{equation}

\end{enumerate}
\end{assumption}
\begin{remark}[Choice of norm on \mmode{\setx}]
It must be noted that the norm \mmode{\norm{\cdot}} considered in the strong convexity assumption \eqref{eq:strong-convexity-assumption} and the continuity assumption \eqref{eq:continuous-partial-derivatives} is identical. Considering an identical norm is not restrictive since all norms on \mmode{\setx} are equivalent. However, using such equivalence often makes the resulting constants \mmode{\alpha, \smooth_1} to be dimension dependent (of \mmode{\setx}). We emphasize here that the smoothness constant given in Lemma \ref{lemma:NDRO-smoothness}, requires that the constants \mmode{\alpha}, and \mmode{\smooth_1} that satisfy conditions \eqref{eq:continuous-partial-derivatives} and \eqref{eq:strong-convexity-assumption}, to satisfy with an identical norm.
\end{remark}

For now, we assume that these conditions for the abstract problem \eqref{eq:NDRO} are satisfied. However, for specific problems like the entropic or variance risk minimization (Sections~\ref{section:er-risk-portfolio-selection} and \ref{section:min-variance-portfolio-selection}, respectively), we will determine verifiable conditions whenever possible so that the conditions in Assumption~\ref{assumption:NDRO-smoothness} are indeed satisfied, (see Lemmas \ref{lemma:er-risk-smoothness-assumptions} and \ref{lemma:variance-smoothness-assumptions}). Informally, condition \eqref{eq:continuous-partial-derivatives} together with the smoothness condition is akin to saying that the directional derivatives of \mmode{\fxp{x}{\P}} are similar to being ``Lipschitz continuous'' with respect to both \mmode{x} and \mmode{\P}. Moreover, the strong convexity assumption implies that the mapping \mmode{\P \mapsto x(\P)} is also similar to being ``Lipschitz continuous''. These two consequences together, yield the smoothness of \mmode{R}. The following lemma formally establishes this deduction in a norm-independent (in \mmode{\P}) analysis.
\begin{lemma}[NDRO-derivative properties]
\label{lemma:NDRO-smoothness}
Let the function \( F \) satisfy Assumptions~\ref{assumption:NDRO-smoothness} with constants \mmode{\alpha >0}, and \mmode{ \smooth_1, \smooth_2 \geq 0},  then the following holds for the risk measure \( R \) as defined in \eqref{eq:ndro-maximization-problem} 
\begin{enumerate}[label = {\rm (\roman*)}, itemsep = 0mm, topsep = 0mm, leftmargin = *]
    \item \emph{Danskin's theorem:} The risk measure \mmode{R} is directionally differentiable on \( \amb \), and for any \mmode{\P,\Q \in \amb} its \( \Q \)-directional derivative \( \RPgrad{\P}{\Q} \) at \( \P \), is given by
\begin{equation}
\label{eq:danskin}
\RPgrad{\P}{\Q} = \FPgrad{x(\P)}{\P}{\Q} .
\end{equation}

    \item \emph{Smoothness:} The risk measure \mmode{R} is \mmode{\smooth}-smooth in the sense of Definition \ref{def:smoothness} for 
    \[
    \smooth = \smooth_2 + \frac{\smooth_1}{2\alpha} \Big( \smooth_1 + \sqrt{\smooth_1^2 + 4\alpha \smooth_2} \Big) .
    \]
\end{enumerate}
\end{lemma}
 
\subsection{Frank-Wolfe based algorithm for the NDRO problem.}
Let \mmode{\P_k} for \mmode{k = 0,1,2,\ldots,} be the sequence of iterates generated by the FW-algorithm \eqref{eq:FW-iterations-main} for \mmode{R} as defined in \eqref{eq:ndro-maximization-problem}. Assume that the FW oracle is \mmode{\delta}-accurate, for some arbitrary \mmode{\delta>0}. To produce a solution \mmode{x_\eps} to the NDRO problem \eqref{eq:NDRO} for a given \mmode{\eps > 0}, one must decide the number of iterations \mmode{\keps}, the stepsize sequence \mmode{\stepsize_k} for \mmode{k = 0,1,\ldots, \keps}, and the stopping criteria for the FW-algorithm. 
Algorithm~\ref{algo:NDRO-min-max} and the related next theorem present the main result of the article that provides a solution to the NDRO problem \eqref{eq:NDRO}.
\begin{algorithm}[h]
\caption{FW algorithm for NDRO problem \eqref{eq:NDRO}}
\label{algo:NDRO-min-max}
\KwIn{A distribution \( \P \in \amb \), positive real numbers \( \eps \) and \mmode{\smooth}, and access to a FW-oracle corresponding to the function \mmode{F}.}
\KwOut{The final decision variable \( x_\eps \) and worst case distribution \( \P_{\eps} \) .}

\nl \textbf{Initialization}: \mmode{\P_0 \define \P} 

\emph{diminishing stepsize regime}

\nl\textbf{for}: \mmode{k =0,1,2,\ldots, \keps \define \ceil*{\frac{2\smooth}{\eps} (2 + 3\delta)} - 2}

\quad \textbf{stepsize:} \mmode{\stepsize_k \define \frac{2}{k + 2}}

\quad find \mmode{(x_k, \Q_k) \in \setx \times \amb} such that
\begin{equation}
\label{eq:FW-oracle-NDRO-regime-1}
\FP{x_k}{\P_k} \symbolspace{=}{\;} \min\limits_{x \in \setx} \FP{x}{\P_k} \quad \text{and} \quad
\sup\limits_{\Q \in \amb} \FPgrad{x_k}{\P_{k}}{\Q} \symbolspace{\leq}{\;} \delta \stepsize_k \smooth + \FPgrad{x_k}{\P_k}{\Q_k} 
\end{equation}

\quad \mmode{\fwval_k = \FPgrad{x_k}{\P_k}{\Q_k}}
\quad \text{and} \quad \mmode{\P_{k+1} \define \P_k + \stepsize_k ( \Q_k - \P_k )}

\nl\textbf{end for}

\emph{constant stepsize regime}

\nl\textbf{for} \mmode{k = \keps + 1, \ldots, 2\keps + 1 }

\quad \textbf{stepsize:} \mmode{\stepsize_k \define \frac{2}{\keps + 2}}

\quad \textbf{if} \mmode{\fwval_k > \eps \frac{2+2\delta}{2+3\delta}}, \textbf{do}

\quad \quad find \mmode{(x_k, \Q_k) \in \setx \times \amb} based on \eqref{eq:FW-oracle-NDRO-regime-1}

\quad \quad \mmode{\fwval_k = \FPgrad{x_k}{\P_k}{\Q_k}}
\quad \text{and} \quad \mmode{\P_{k+1} \define \P_k + \stepsize_k ( \Q_k - \P_k )}

\quad \textbf{else}

\quad \quad \textbf{Output}: \( x_\eps \define x_k \) and \( \P_\eps \define \P_k \) and \textbf{end for}.

\nl\textbf{end for}
\end{algorithm}

\begin{theorem}[NDRO solution]
\label{theorem:NDRO-FW}
Consider the Nonlinear DRO problem \eqref{eq:NDRO} under the setting of Assumptions~\ref{assumption:NDRO-smoothness}. Then the following holds
\begin{enumerate}[label = {\rm (\roman*)}, itemsep = 0mm, topsep = 0mm, leftmargin = *]
\item \emph{Strong duality:} 
\begin{equation}
\label{eq:ndro-strong-duality}
\primaloptval \symbolspace{=}{\;} \min\limits_{x \in \setx} \sup\limits_{\P \in \amb} \FP{x}{\P} \symbolspace{=}{\;} \sup\limits_{\P \in \amb} \min\limits_{x \in \setx} \FP{x}{\P} \symbolspace{=}{\;} \dualoptval
\end{equation}
\item {\em Saddle point computation:} Given any \mmode{\eps > 0}, the pair \mmode{(x_\eps , \P_\eps)} computed from Algorithm \ref{algo:NDRO-min-max} is an \mmode{\eps}-saddle point in the sense of Definition \ref{def:dual-gap-eps-solution}.
\end{enumerate}
\end{theorem}

Approaching an NDRO problem via the FW-based approach in Algorithm~\ref{algo:NDRO-min-max} requires the following three key ingredients:

\begin{enumerate}[label = {\rm (\roman*)}, itemsep = 0mm, topsep = 0mm, leftmargin = *]

\item {\em Regularity conditions of Assumption~\eqref{assumption:NDRO-smoothness}}: To ensure the convergence of the FW algorithm, we must ensure that the corresponding G-derivatives exist, and satisfy regularity conditions of Assumption~\ref{assumption:NDRO-smoothness}, particularly (ii)~\textit{continuous derivatives} and (iii)~\textit{smoothness}.

\item {\em Solver for \mmode{\min_{x \in \setx} \FP{x}{\P}} in~\eqref{eq:FW-oracle-NDRO-regime-1}}: Since it is a finite-dimensional convex problem, any well known first-order methods like ISTA, FISTA \cite{beck2009fast}, Nesterov's Accelerated Gradient Descent \cite{nesterov1983method}, Extra Gradient \cite{korpelevich1976extragradient} can be applied. Moreover, since the risk function \mmode{F} is assumed to be strongly convex in \mmode{x} (or at least with a regulariser), these first-order methods solve the corresponding minimization problems with geometric convergence.

\item {\em Feasibility of the FW oracle~\eqref{eq:FW-oracle-NDRO-regime-1}}: The FW worst-case distribution problem must admit tractable reformulations. In fact, when the directional derivatives are linear in the distribution, the corresponding FW problem is indeed tractable for several interesting choices of risk measures and ambiguity sets, as discussed in \cite{postek2016computationally, kuhn2019wasserstein, gao2022wasserstein}.

\end{enumerate}

\subsection{Slower convergence without strong-convexity}
\label{subsection:slow-convergence-without-strong-convexity}
In this case, we assume that the function \mmode{F} satisfies conditions (ii) and (iii) of Assumption \ref{assumption:NDRO-smoothness}. However, it may not be necessarily strongly convex in \mmode{x}. For example, the variance risk measure is strongly convex if and only if the smallest eigenvalue of the matrix \mmode{\big( \Sigma_{\P} - \mup{\P}\mup{\P}\transp \big) } is bounded away from \mmode{0} uniformly over \mmode{\P \in \amb}, which might not be the case. Even in such a setting, we desire to develop methods that compute an \mmode{\eps}-saddle point of \mmode{F} using the setup of Algorithm \ref{algo:NDRO-min-max} for any \mmode{\eps>0}. We take inspiration from the smoothing techniques in the optimization literature \cite{nesterov2005smooth} for smoothing a non-smooth convex function and devise similar techniques that work with a suitable strongly-convex approximation \mmode{F_{\eps}}, of \mmode{F}, and still use Algorithm \ref{algo:NDRO-min-max} to compute an \mmode{\eps}-saddle point of \mmode{F}. To this end, we assume that the set \mmode{\setx} is also bounded in addition to being closed and, thus, compact. Many common examples of \mmode{\setx} like the simplex \mmode{\simplex{n}}, satisfy the compactness assumption.

For any given an \mmode{\eps > 0}, \mmode{x \in \setx}, and \mmode{\P \in \amb}, let \mmode{\fapprox{x}{\P} \define \FP{x}{\P} + \big( \nicefrac{\eps}{\boundx^2} \big) \norm{x}^2}. We propose to solve the following min-max problem in place of \eqref{eq:NDRO}
\begin{equation}
   \label{eq:NDRO-approx-problem}
   \min_{x \in \setx} \sup_{\P \in \amb} \fapprox{x}{\P} ,
\end{equation}
where \mmode{B_x \define \max_{x \in X} \norm{x}}. It is apparent at once that \mmode{\fapprox{x}{\P}} is \mmode{(\nicefrac{2\eps}{\boundx^{2}}) }-strongly convex in \mmode{x}, uniformly over \mmode{\P \in \amb}, and consequently satisfies all the conditions in Assumption \ref{assumption:NDRO-smoothness}. Thus, employing Algorithm \ref{algo:NDRO-min-max} with \mmode{F_{\eps}} computes a pair \mmode{(x' , \P')} that satisfies the inequalities
\begin{equation}
\label{eq:approx-saddle-point-ineq}
-\eps + \sup_{\Q \in \amb} \fapprox{x'}{\Q} \symbolspace{\leq}{\;} \fapprox{x'}{\P'} \symbolspace{\leq}{\;} \eps + \inf_{y \in \setx} \fapprox{y}{\P'} .
\end{equation}
It turns out that such a pair \mmode{(x' , \P')} is an \mmode{\eps}-saddle point of \mmode{F} as well. This is easily seen by observing that on the one hand, we have
\[
\begin{aligned}
    \big( \nicefrac{\eps}{\boundx^2} \big) \norm{x'}^2 + \sup\limits_{\Q \in \amb} \FP{x'}{\Q} \symbolspace{&=}{\;} \sup\limits_{\Q \in \amb} \fapprox{x'}{\Q} \symbolspace{\leq}{\;} \eps + \fapprox{x'}{\P'} \quad \text{(from \eqref{eq:approx-saddle-point-ineq})} \\
    & = \ \eps + \big( \nicefrac{\eps}{\boundx^2} \big) \norm{x'}^2 + \FP{x'}{\P'}.
\end{aligned}
\]
Thus, we immediately get \mmode{\sup_{\Q \in \amb} \ \FP{x'}{\Q} \leq \eps + \FP{x'}{\P'}}. On the other hand, since the minimization over \mmode{x} is solved exactly in Algorithm \ref{algo:NDRO-min-max}, we have 
\[
\begin{aligned}
\fapprox{x'}{\P'} \symbolspace{&=}{\;} \min_{y \in \setx} \fapprox{y}{\P'} \symbolspace{=}{\;} \min_{y \in \setx} \FP{y}{\P'} + \big( \nicefrac{\eps}{\boundx^2} \big) \norm{y}^2 \\
& \leq \ \eps + \min_{y \in \setx} \FP{y}{\P'} \quad \text{since \mmode{\norm{y} \leq \boundx} for all \mmode{y \in \setx} } .
\end{aligned}
\]
Thus, it is apparent that the pair \mmode{(x' , \P')} is an \mmode{\eps}-saddle point of the function \mmode{F} as well.

\begin{corollary}[Slower convergence]
Consider a function \mmode{\FP{x}{\P}} that is not necessarily strongly convex in \mmode{x}. Then for any desired precision \mmode{\eps > 0}, an \mmode{\eps}-saddle point of \mmode{F} can be computed by applying Algorithm \ref{algo:NDRO-min-max} with \mmode{\keps = \ceil*{\frac{2\smooth_{\eps}}{\eps} (2 + 3\delta)} - 2 }, to the strongly convex approximate function \mmode{F^\eps} , where \mmode{ \smooth_\eps = \smooth_2 + \frac{\smooth_1 \boundx^2}{4\eps} \Big( \smooth_1 + \sqrt{\smooth_1^2 + \frac{8\eps}{\boundx^2} \smooth_2 } \Big) }.
\end{corollary}

Since the strong convexity parameter of \mmode{F_{\eps}} itself depends on \mmode{\eps}, we conclude from Lemma~\ref{lemma:NDRO-smoothness} that the risk measure \mmode{R_{\eps}(\P) \define \min_{y \in \setx} \fapprox{y}{\P}}, also has an \mmode{\eps} dependent smoothness constant \mmode{\smooth_\eps = \smooth_2 + \frac{\smooth_1 \boundx^2}{4\eps} \Big( \smooth_1 + \sqrt{\smooth_1^2 + \frac{8\eps}{\boundx^2} \smooth_2 } \Big) }. Now, Algorithm \ref{algo:NDRO-min-max} terminates in \mmode{O(\keps)} iterations, where \mmode{\keps = O(\nicefrac{\smooth_\eps}{\eps})}. Since \mmode{\smooth_\eps = O(\nicefrac{1}{\eps})}, it is easily seen that for non-strongly convex functions \mmode{F}, applying Algorithm \ref{algo:NDRO-min-max} to its regularized strongly-convex approximation \mmode{F_{\eps}} takes \mmode{O(\nicefrac{1}{\eps^2})} iterations to compute an \mmode{\eps}-saddle point of \mmode{F}. To compare, recall that for a strongly-convex function \mmode{F}, Algorithm \ref{algo:NDRO-min-max} takes \mmode{O(\nicefrac{1}{\eps})} iterations to compute an \mmode{\eps}-saddle point. This trade-off between speed of convergence and precision in the approximation is a typical occurrence in standard smoothing techniques as well.

\section{Entropic Risk Portfolio Selection}
\label{section:er-risk-portfolio-selection}

This section is dedicated to study the proposed methodology and its required conditions for the entropic risk~\eqref{NDRO_entropic}. 
To this end, let \mmode{{\xi}(t)}, \mmode{t = 1,2,\ldots,T} be i.i.d. samples drawn from some unknown underlying distribution \mmode{\P_o} supported on \mmode{\Xi = \R{n}}. We assume that the individual components \mmode{\xi_i (t)}, \mmode{i = 1,2,\ldots,n}, are independently distributed from each other. Let \mmode{\Pref_i \define \frac{1}{T}\summ{t = 1}{T} \delta (\xi_i (t)) } for \mmode{i = 1,2,\ldots,n}, and let \mmode{\Pref = \Pi_{i = 1}^n \Pref_i}. For \mmode{c > 0} and any two distributions \mmode{\P, \Q} supported on \mmode{\R{}}, let \mmode{\Wdist{c} (\P, \Q)} be a Wasserstein distance between them defined as
\begin{equation}
\label{eq:er-risk-wasserstein-distance}
\Wdist{c} (\P, \Q) \define 
\begin{cases}
\begin{aligned}
& \sup_{\pi} && \EE{\pi} \left[  e^{c \abs{u - v}} \right]  \quad \text{where \mmode{(u,v)} are \mmode{\pi}-jointly distributed } \\
& \sbjto && \P(u) = \int_v \pi(u,v) dv \quad \text{ and } \quad \Q(v) = \int_u \pi(u,v) du .
\end{aligned}
\end{cases}
\end{equation}
For each \mmode{j = 1, 2 \ldots, n}, consider \mmode{\erwamb{j} \define \{ \P : \Wdist{c} ( \P, \Pref_j ) \leq \ambradius \}} and let \mmode{\amb_c = \Pi_{j = 1}^n \erwamb{j}}.
Let \mmode{\theta_j \in (0, 1)} for \mmode{j = 1,2,\ldots, n} be the risk aversion parameters, and for \mmode{c > \max\{ \theta_j : j = 1,2,\ldots,n \} }, we seek to solve
\begin{equation}
\label{eq:er-risk-portfolio-selection}
\min_{x \in \simplex{n}} \sup_{\P \in \amb_c} \quad \errisk{x}{\P} \symbolspace{\define}{} \summ{j = 1}{n}{\frac{1}{\theta_j} \log \left( \EE{\P_j} [e^{-\theta_j x_j \xi_j}] \right) }  .
\end{equation}
We will study~\eqref{eq:er-risk-portfolio-selection} in detail for the various notions of the directional derivative, smoothness, and the resulting FW-oracle with its tractable formulations.

\subsection{Regularity conditions}
Recalling the entropic risk~\eqref{eq:er-risk-L-r}, we note that in the optimization~\eqref{eq:er-risk-portfolio-selection} the mapping \mmode{\P \mapsto \errisk{x}{\P}} is an RR measure in the sense of Definition~\ref{def:RR-measure} with functions \mmode{L} and \mmode{r} given by
\[
L(x,\xi) = \Big( e^{-\theta_1 x_1 \xi_1} , e^{-\theta_2 x_2 \xi_2} , \ldots , e^{-\theta_n x_m \xi_n} \Big) \quad \text{and} \quad
\risk (z) = \summ{j = 1}{n} \frac{1}{\theta_j} \log(z_j) .
\]
Thus, in view of Lemma \ref{lemma:RR-directional derivatives} and using the short-hand notation \mmode{\erxpshrt{x}{\cdot} \define \errisk{x}{\cdot}}, we can write the directional derivatives of the risk measure \mmode{\ER_x} as
\begin{equation}
\label{eq:er-risk-derivative}
\erxpgrad{x}{\P}{\Q} = \summ{j = 1}{n}{ \frac{\EE{\Q_j - \P_j} [e^{-\theta_j x_j \xi_j}] }{ \theta_j \EE{\P_j} [e^{-\theta_j x_j \xi_j}]} } \quad \text{for every } \P, \Q \in \amb .
\end{equation}
It turns out that for any saddle point \mmode{(x\opt, \P\opt)} of the problem \eqref{eq:er-risk-portfolio-selection}, \mmode{\P\opt} belongs to a strictly smaller set \mmode{\amb_c'} contained in \mmode{\amb_c} and is of bounded support. The following Definition characterizes the smaller ambiguity set and the subsequent Lemma \ref{lemma:er-risk-smaller-ambiguity} formalizes this assertion.

\begin{definition}[Restricted ambiguity set]
For each \mmode{j = 1,2,\ldots, n}, let 
\[
\ximin{j} \define -\ambradius - \frac{\log(T)}{c} + \min_{t = 1,\ldots,T} \xi_j (t) \quad \text{ and } \quad \ximax{j} = \max_{t = 1,\ldots,T} \xi_j (t) ,
\]
let \mmode{ \mathcal{W}_j' \subset \erwamb{j} } be the collection of distributions supported on \mmode{[\ximin{j} , \ximax{j}]}, and let \mmode{\amb_c' \define \Pi_{j = 1}^n \mathcal{W}_j'}.
\end{definition}

\begin{lemma}[Restricted ambiguity set]
\label{lemma:er-risk-smaller-ambiguity}
A pair \mmode{(x\opt, \P\opt)} is a saddle point of \eqref{eq:er-risk-portfolio-selection} if and only if it is also a saddle point of 
\begin{equation}
\label{eq:er-risk-problem-smaller-ambiguity}
\min_{x \in \setx} \sup_{\P \in \amb_c'} \errisk{x}{\P} ,
\end{equation}
\end{lemma}
The crucial consequence of Lemma \ref{lemma:er-risk-smaller-ambiguity} is that it allows us to conclude regularity conditions for the risk measure \mmode{\errisk{x}{\P}} by restricting our analysis to the smaller ambiguity set \mmode{\amb_c'}. It turns out that the \emph{continuous derivatives} and \emph{smoothness} conditions (ii) and (iii) of Assumption \ref{assumption:NDRO-smoothness} respectively, are not satisfied for the entropic risk portfolio optimization problem \eqref{eq:er-risk-portfolio-selection} on the entire ambiguity set \mmode{\amb} but only on the smaller set \mmode{\amb_c'}. Lemma \ref{lemma:er-risk-smaller-ambiguity} ensures that we can indeed restrict the ambiguity set from \mmode{\amb_c} to \mmode{\amb_c'} without losing any optimal solution. The following lemma establishes the regularity conditions.

\begin{lemma}[Entropic risk regularity conditions]
\label{lemma:er-risk-smoothness-assumptions}
Consider the entropic risk portfolio optimization problem \eqref{eq:er-risk-portfolio-selection}. Let \mmode{\mathcal{E}_x (\cdot) \define \errisk{x}{\cdot}} for every \mmode{x \in \setx}, then the following assertions hold
\begin{enumerate}[label = {\rm (\roman*)}, itemsep = 0mm, topsep = 0mm, leftmargin = *]

\item \emph{Continuous derivatives:} The directional derivatives \mmode{\erxpgrad{x}{\P}{\Q}}, satisfy
\[
\erxpgrad{x}{\P}{\Q} - \erxpgrad{y}{\Q}{\P} \symbolspace{\leq}{\;} \pnorm{x - y}{2} \sqrt{ \summ{j = 1}{n} (\ximax{j} - \ximin{j})^2 e^{4 \theta_j (\ximax{j} - \ximin{j})} }, \quad \forall\, x,y \in \setx \quad \forall\,\P,\Q \in \amb_c'. 
\]

\item \emph{Smoothness:} The risk measure \mmode{\mathcal{E}_x} is \mmode{\smooth}-smooth in the sense of Definition \ref{def:smoothness} on \mmode{\amb_c'} and uniformly over \mmode{x \in \setx} for \mmode{\smooth =  \summ{j = 1}{n} \frac{1}{\theta_j} \big( e^{\theta_j (\ximax{j} - \ximin{j})} - 1 \big)^2 } .

\end{enumerate}
\end{lemma}
Besides the \emph{smoothness} and \emph{continuity} conditions, the uniform strong-convexity assumption is extremely difficult to verify for the entropic risk measure. Therefore, as discussed in Section \ref{subsection:slow-convergence-without-strong-convexity}, we add a strongly-convex regularizer of \mmode{x} to apply Algorithm \ref{algo:NDRO-min-max}.

\subsection{The FW Oracle}
The FW problem for minimum entropic risk portfolio selection \eqref{eq:er-risk-portfolio-selection}, at a given distribution \mmode{\P} is
\begin{equation}
\label{eq:er-risk-FW-problem}
\sup_{\Q \in \amb_c} \quad \summ{j = 1}{n} \frac{1}{\theta_j} \frac{ \EE{\Q_j - \P_j} [e^{-\theta_j x_j \xi_j}] }{ \EE{\P_j} [e^{-\theta_j x_j \xi_j}] } .
\end{equation}
Since \mmode{\Q = \Pi_{j = 1}^n \Q_j } and \mmode{\Q_j \in \erwamb{j}}, the FW problem \eqref{eq:er-risk-FW-problem} for entropic risk is separable into \mmode{n} different problems, and for each \mmode{j = 1,2,\ldots,n}, we have
\begin{equation}
\label{eq:er-risk-FW-problem-elementwise}
\sup_{\Q_j \in \erwamb{j}} \EE{\Q_j} [e^{- \theta_j x_j \xi_j}] .
\end{equation}

\begin{lemma}[Entropic risk FW oracle]
\label{lemma:er-risk-FW-oracle}
Let \mmode{\theta > 0}, \mmode{x \in [0, 1]}, and \mmode{z(t)}, \mmode{t = 1,2,\ldots,T,} be any arbitrary collection of real numbers and let \mmode{\Pref = \frac{1}{T} \summ{t = 1}{T} \delta(z(t)) } be a given discrete distribution. Consider the following problem
\begin{equation}
\label{eq:er-risk-FW-scalar-problem}
\sup_{ \Q \in \erwamb{} } \EE{\Q} [e^{- \theta x z}] .
\end{equation}
Let \mmode{\big(Z(t)\big)_t} be a non-increasing permutation of \mmode{\big(z(t)\big)_t}, i.e., we have \mmode{Z(1) \geq Z(2) \geq \cdots \geq Z(T)}. Let \mmode{T' \in \{1, 2, \ldots, T \} } be the smallest integer such that \mmode{ (T e^{c \ambradius} - T' ) e^{\frac{- c \theta x Z(T') }{c - \theta x}}\geq \summ{t = T' + 1}{T} e^{\frac{- c \theta x Z(t) }{c - \theta x}} }, define
\begin{equation}
\label{eq:er-risk-scalar-problem-solution}
\begin{cases}
\begin{aligned}
\eta\opt & \define \frac{\theta x}{c} \left( \frac{1}{T e^{c\ambradius} - T'} \summ{t = T' + 1}{T} e^{ \frac{- c \theta x Z(t)}{c - \theta x} } \right)^{\frac{c - \theta x}{\theta x}} \text{ and } \\
z\opt(t) & \define \min \left\{ z(t) , \frac{ c z(t) + \log (\nicefrac{c \eta\opt}{\theta x}) }{ c - \theta x } \right\} \quad \text{ for } t = 1,2,\ldots,T .
\end{aligned}
\end{cases}
\end{equation}
\begin{enumerate}[label = {\rm (\roman*)}, itemsep = 0mm, topsep = 0mm, leftmargin = *]
\item {\em Optimal solution}: The discrete distribution \mmode{\Q\opt = \frac{1}{T} \summ{t = 1}{T} \delta (z\opt (t)) } is the optimal solution \mmode{\Q\opt} to the linear worst-case distribution problem \eqref{eq:er-risk-FW-scalar-problem}.

\item {\em Lower and upper bounds}: For any \mmode{t = 1, 2,\ldots, T} we have 
    \[
    \underline{z} \define  - \ambradius - \frac{\log(T)}{c} + \min_{t = 1,\ldots,T} z(t) \leq z\opt (t) \leq \overline{z} \define \max_{t = 1 ,\ldots, T} z(t) .
    \]
\end{enumerate}
\end{lemma}

\begin{corollary}[FW restricted ambiguity set]
For any \mmode{x \in \simplex{n}} and \mmode{\P \in \amb}, we have
\begin{equation}
\label{eq:er-risk-FW-arg-max-equality}
\arg\max_{\Q \in \amb_c} \erxpgrad{x}{\P}{\Q} \quad = \quad \arg\max_{\Q' \in \amb_c'} \erxpgrad{x}{\P}{\Q'} .
\end{equation}
\end{corollary}

\subsection{Simulation results}
We validate the convergence properties of Algorithm \ref{algo:NDRO-min-max} for the NDRO problem of the entropic risk portfolio selection \eqref{eq:er-risk-portfolio-selection} with unrestricted support (i.e., \mmode{\Xi = \R{n} }). For each \mmode{j = 1,2,\ldots, n}, the samples \mmode{\xi_j(t)}, for \mmode{t = 1,2,\ldots, T = 2n}, are drawn randomly from bi-exponential distribution with density \mmode{f_j (\cdot) = \lambda_j e^{- \lambda_j \abs{\cdot}}} with each parameter \mmode{\lambda_j} drawn uniformly from \mmode{[0,1]}, and the distribution \mmode{\Pref} is uniformly distributed over the samples drawn. The risk aversion parameter \mmode{\theta_j} are drawn uniformly from \mmode{[0,1]} for each \mmode{j = 1,2,\ldots,n}, and we select \mmode{c = 1} to define the Wasserstein distance in \eqref{eq:er-risk-wasserstein-distance} and the resulting ambiguity set \mmode{\amb}. For \mmode{n = 250} and various values of \mmode{\ambradius} (the radius of the ambiguity set), we solve \eqref{eq:er-risk-portfolio-selection}. The proposed FW-based method (Algorithm \ref{algo:NDRO-min-max} with the FW-oracle of Proposition \ref{lemma:er-risk-FW-oracle}) is implemented in {\fontfamily{pcr}\selectfont MATLAB} on a Macbook Air (M1 with 8GB RAM), wherein the minimization problem: \mmode{\min_{x \in \simplex{n}} \ \errisk{x}{\P_k}} is solved by running the FISTA algorithm \cite{beck2009fast} until convergence.

For three distinct values of \mmode{\ambradius = 1, 5}, and \mmode{10}, Figure \ref{fig:er-risk-simulation-plots-noreg} shows the convergence of Algorithm \ref{algo:NDRO-min-max} with \mmode{\keps = 350}, on the entropic risk portfolio selection problem \eqref{eq:er-risk-portfolio-selection}. The plots on the top show the evolution of the {\em Primal} and {\em Dual} functions: \mmode{\min_{x \in \simplex{n}} \errisk{x}{\P_k}} and \mmode{\sup_{\P \in \amb_c} \errisk{x_k}{\P}} respectively, w.r.t.\,the iteration \mmode{k} of the algorithm. Whereas, the plots at the bottom of the figure show the evolution of the two sub-optimality metrics:
\[
\begin{cases}
\begin{aligned}
\text{\emph{primal sub-optimality}} \quad &: \quad \quad \ER\opt - \min\limits_{x \in \simplex{n}} \errisk{x}{\P_k} \quad \text{and} \\
\text{\emph{duality gap}} \quad &:  \quad \sup\limits_{\P \in \amb_c} \errisk{x_k}{\P} \symbolspace{-}{\;} \min\limits_{x \in \simplex{n}} \ \errisk{x}{\P_k}   .
\end{aligned}
\end{cases}
\]
Since Algorithm \ref{algo:NDRO-min-max} explicitly solves the minimization over \mmode{x} in each iteration (see \eqref{eq:FW-oracle-NDRO-regime-1}), the primal function is readily available. However, this is not the case with the dual function which needs to be computed independently at each iteration by solving \mmode{\sup_{\P \in \amb_c} \ \erxpshrt{x_k}{\P}} while keeping the current iterate \mmode{x_k} fixed.
\begin{figure}[t]
    \centering
    \begin{subfigure}[b]{0.32\textwidth}
    \raggedleft
\includegraphics[width=  5cm]{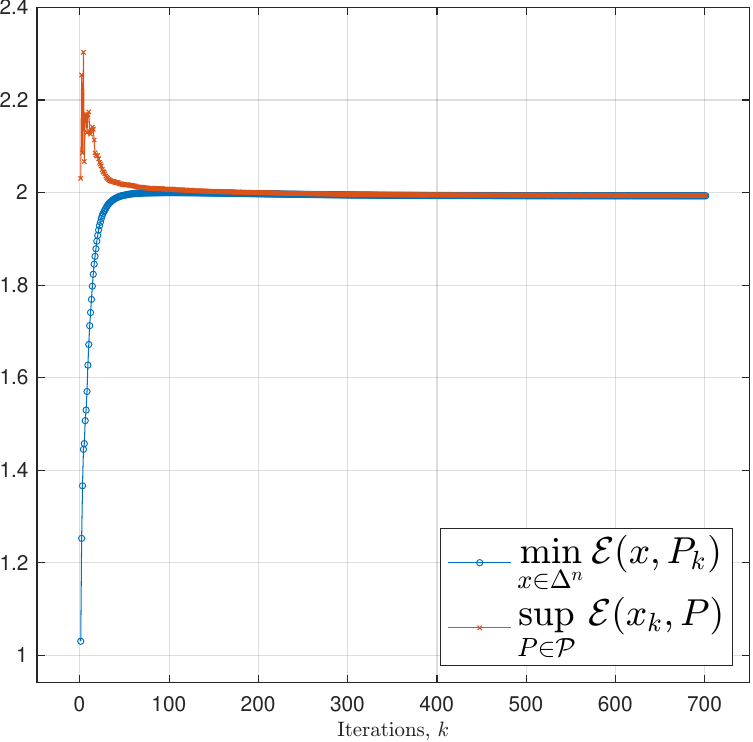}
\vspace{0.2cm}
\includegraphics[width=  5cm]{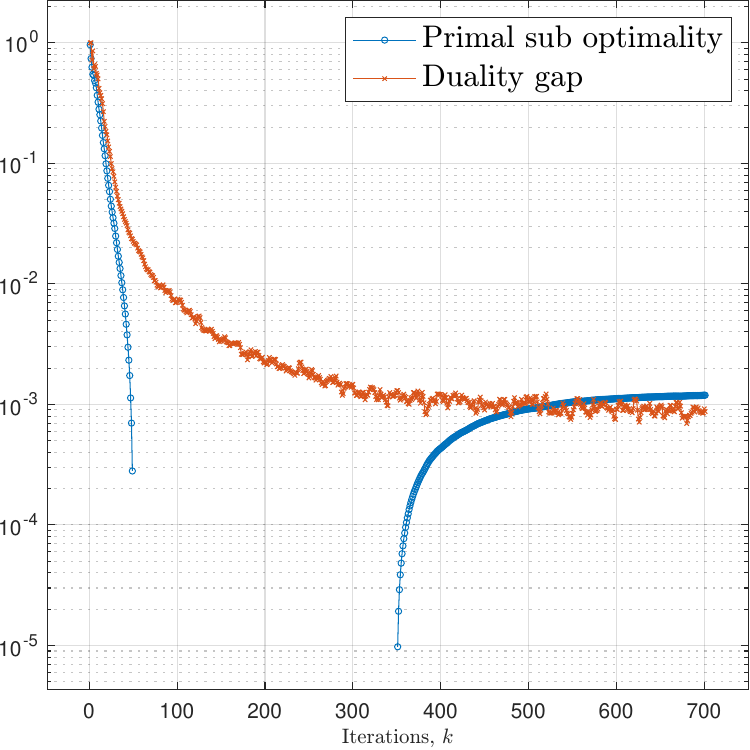}
    \caption{\mmode{\ambradius = 1}, \mmode{\alpha = 0}}
    \label{fig:er-risk-rho1alpha0}
    \end{subfigure}
    \hfill
    \begin{subfigure}[b]{0.32\textwidth}
    \raggedleft
\includegraphics[width=  5cm]{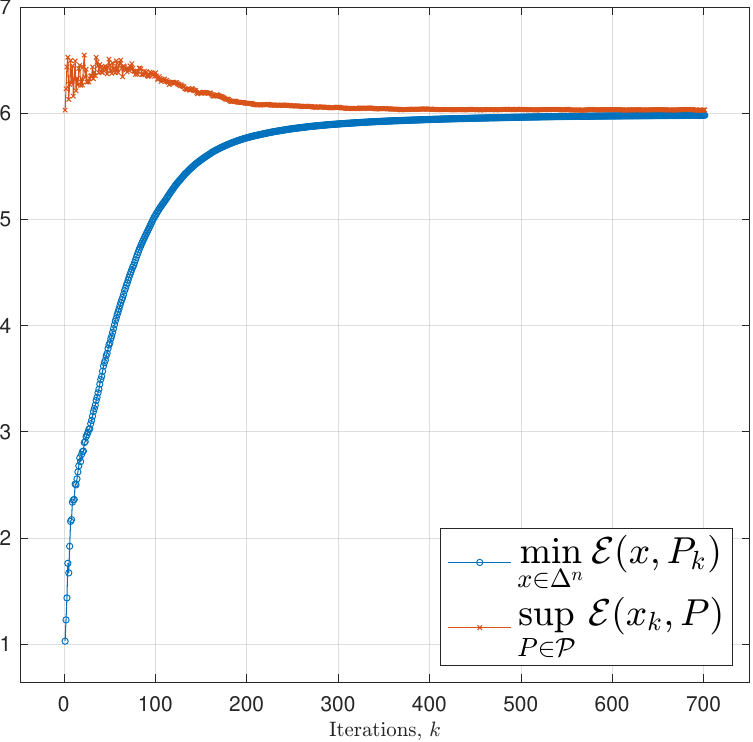}
\vspace{0.2cm}
\includegraphics[width=  5cm]{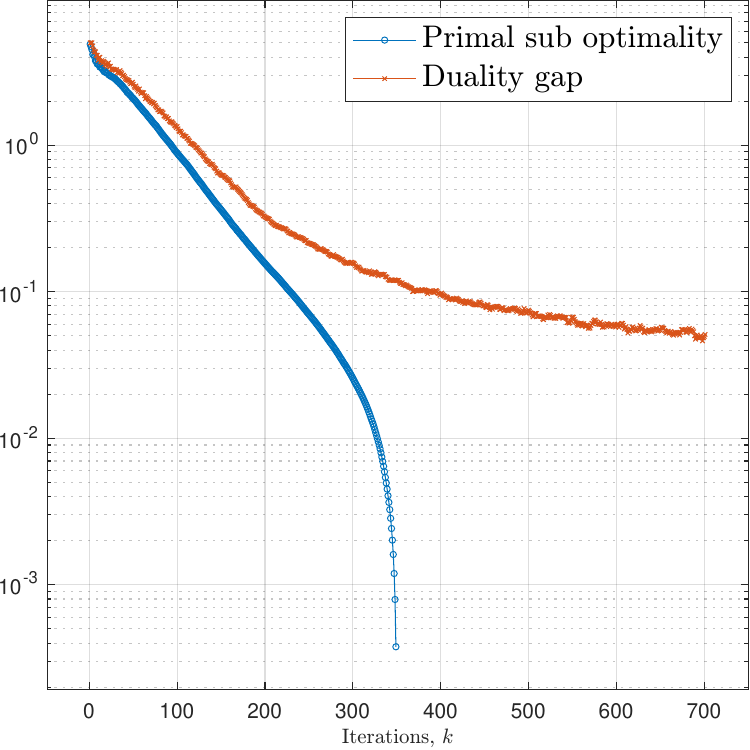}
    \caption{\mmode{\ambradius = 5}, \mmode{\alpha = 0}}
    \label{fig:er-risk-rho5alpha0}
    \end{subfigure}
    \hfill
    \begin{subfigure}[b]{0.32\textwidth}
    \raggedleft
\includegraphics[width=  5cm]{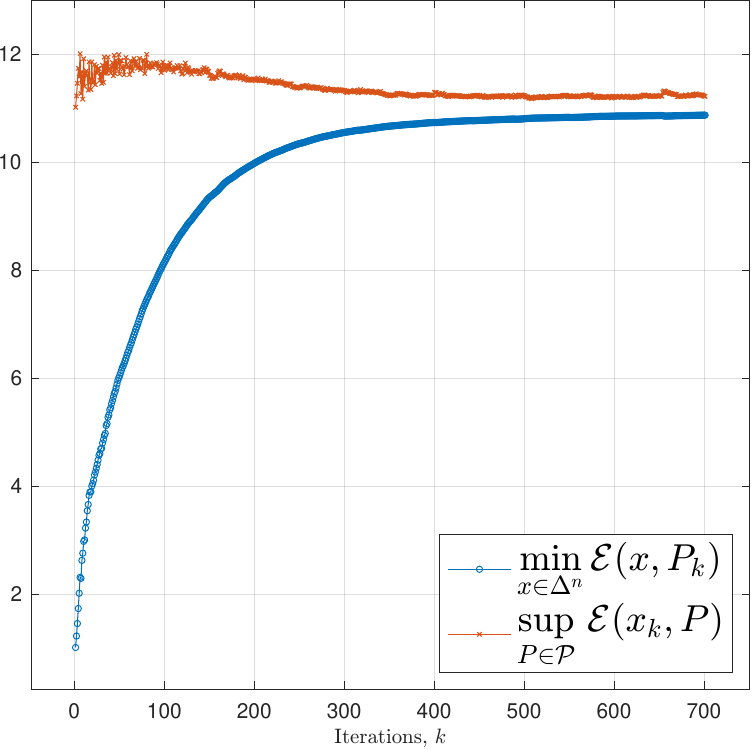}
\vspace{0.2cm}
\includegraphics[width=  5cm]{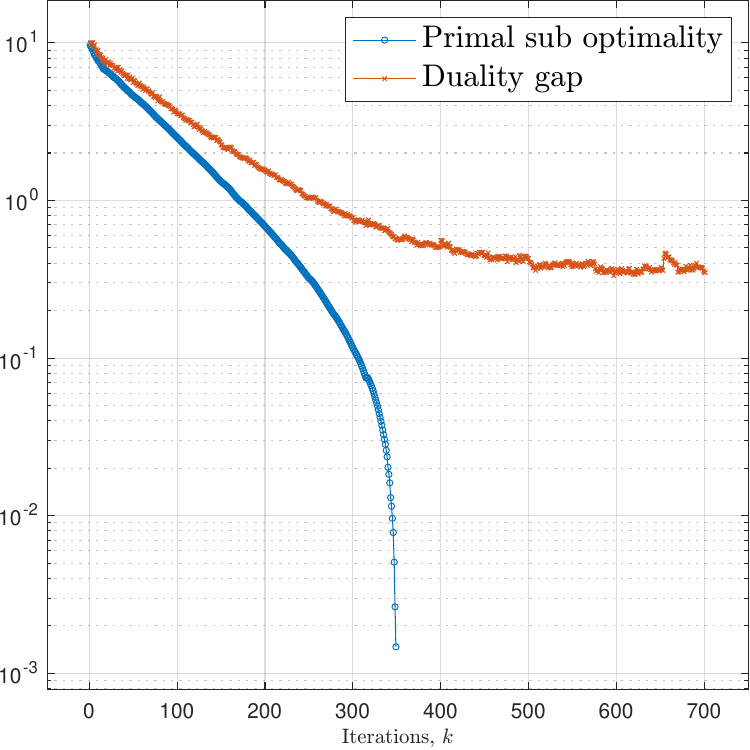}
    \caption{\mmode{\ambradius = 10}, \mmode{\alpha = 0}}
    \label{fig:er-risk-rho10alpha0}
    \end{subfigure}
    \caption{Convergence plots for Algorithm \ref{algo:NDRO-min-max} with \mmode{\keps = 350}, applied to \eqref{eq:er-risk-portfolio-selection}.}
    \label{fig:er-risk-simulation-plots-noreg}
\end{figure}

\paragraph{{\bf With explicit regularization of \mmode{x}}} We observe in Figure \ref{fig:er-risk-simulation-plots-noreg} that, as the value of \mmode{\ambradius} increases, the convergence of the algorithm becomes extremely slow. This can be attributed to the fact that the strong-convexity parameter becomes extremely small, and thus the resulting smoothness constant is prohibitively large, making the convergence slow. To remedy this, as suggested in Section \ref{subsection:slow-convergence-without-strong-convexity}, we compute an approximate saddle point of \eqref{eq:er-risk-portfolio-selection} by applying the FW-algorithm to the explicitly regularized min-sup problem
\begin{equation}
\label{eq:er-risk-simulation-problem-with-reg}
\min_{x \in \simplex{n}} \sup_{\P \in \amb_c} \frac{\alpha}{2} \pnorm{x}{2}^2 + \errisk{x}{\P} .  
\end{equation}
Since \mmode{\pnorm{x}{2} \leq \pnorm{x}{1} = 1} for all \mmode{x \in \simplex{n}}, we know that an \mmode{\eps}-saddle point of \eqref{eq:min-var-ellips-simulation-problem} can be computed for any \mmode{\eps > 0} by applying Algorithm \ref{algo:NDRO-min-max} to \eqref{eq:min-var-ellips-simulation-problem-with-reg} with \mmode{\alpha = 2\eps}.
\begin{figure}[t]
    \centering
    \begin{subfigure}[b]{0.32\textwidth}
    \raggedleft
\includegraphics[width=  5cm]{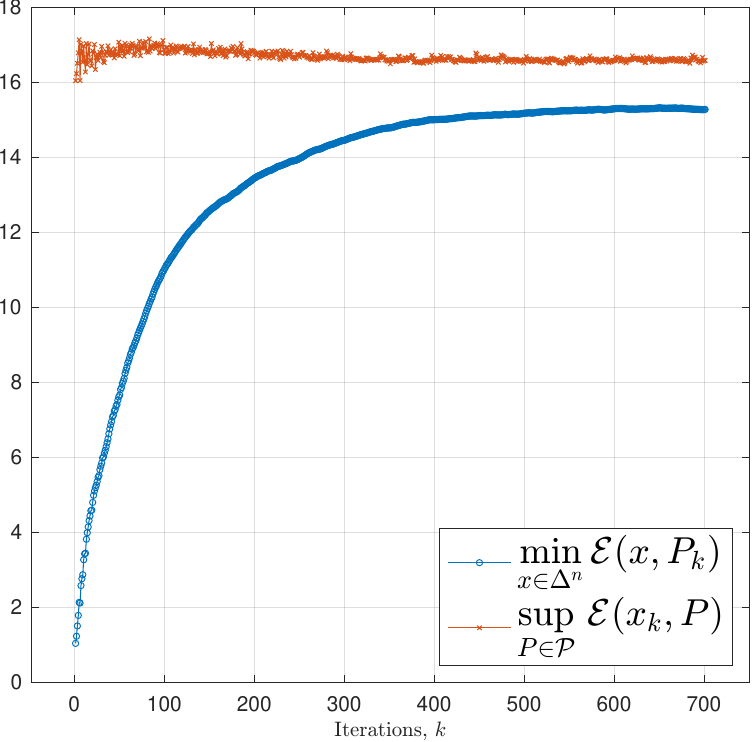}
\vspace{0.2cm}
\includegraphics[width=  5cm]{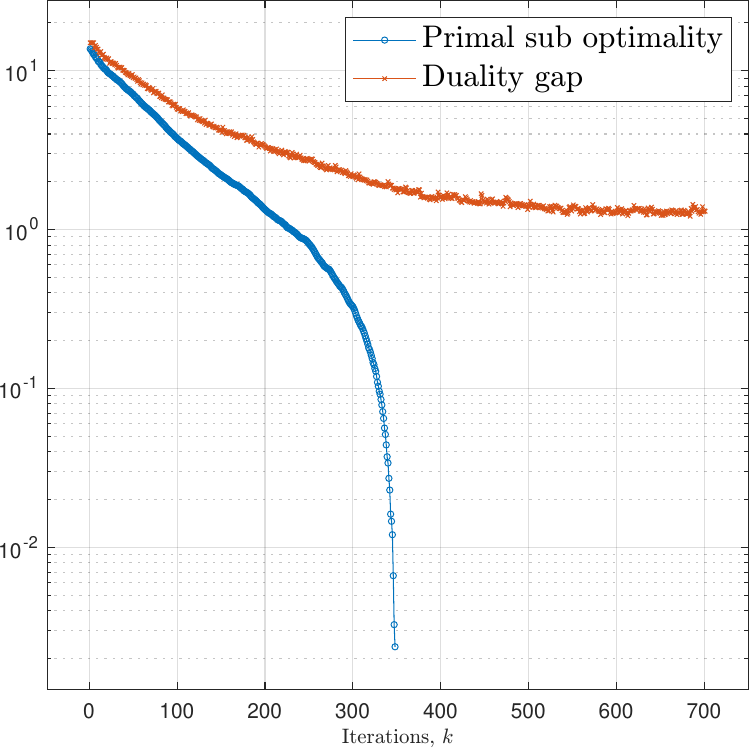}
    \caption{\mmode{\ambradius = 15}, \mmode{\alpha = 0}}
    \label{fig:er-risk-rho15alpha0}
    \end{subfigure}
    \hfill
    \begin{subfigure}[b]{0.32\textwidth}
    \raggedleft
\includegraphics[width=  5cm]{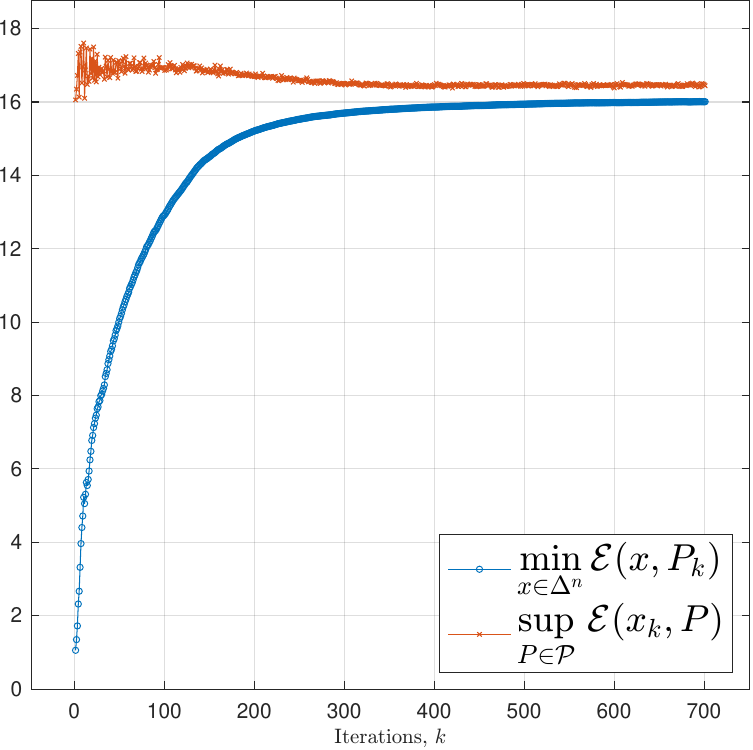}
\vspace{0.2cm}
\includegraphics[width=  5cm]{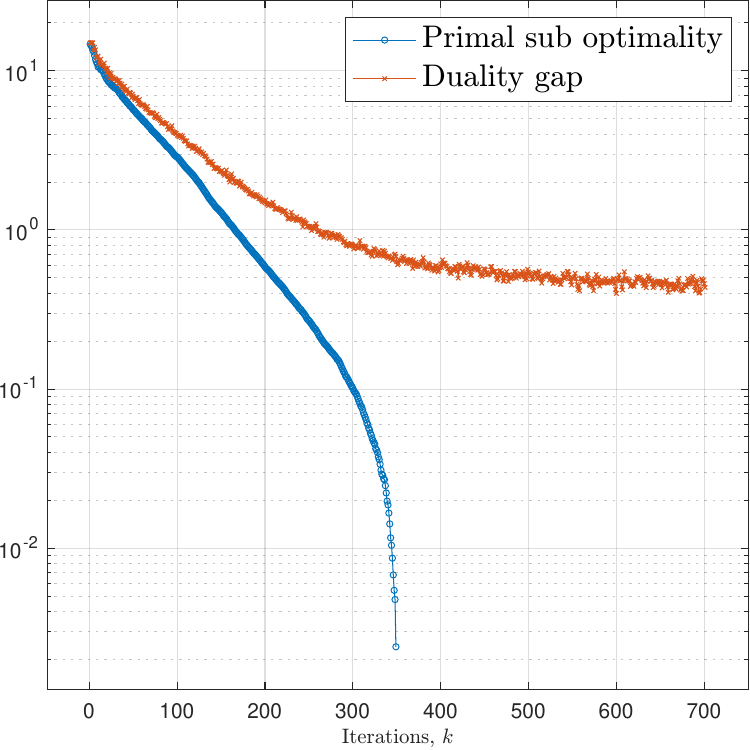}
    \caption{\mmode{\ambradius = 15}, \mmode{\alpha = 5}}
    \label{fig:er-risk-rho15alpha5}
    \end{subfigure}
    \hfill
    \begin{subfigure}[b]{0.32\textwidth}
    \raggedleft
\includegraphics[width=  5cm]{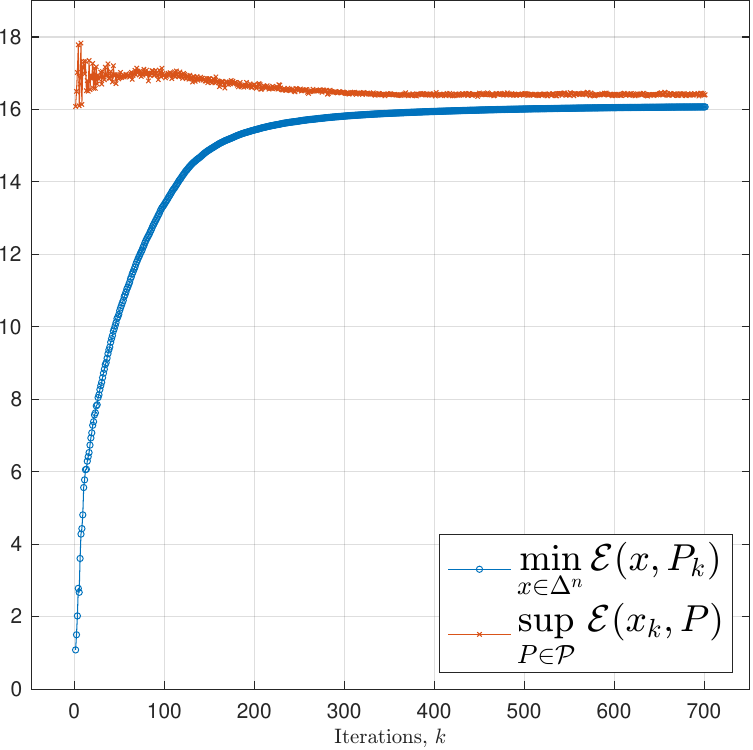}
\vspace{0.2cm}
\includegraphics[width=  5cm]{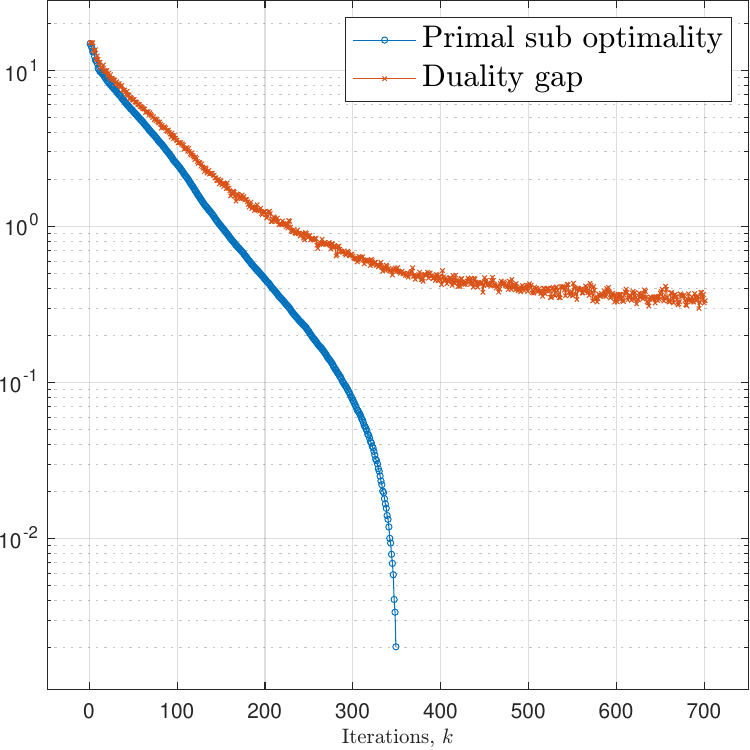}
    \caption{\mmode{\ambradius = 15}, \mmode{\alpha = 10}}
    \label{fig:er-risk-rho15alpha10}
    \end{subfigure}
    \caption{Convergence plots for Algorithm \ref{algo:NDRO-min-max} with \mmode{\keps = 350}, applied to \eqref{eq:er-risk-simulation-problem-with-reg}.}
    \label{fig:er-risk-simulation-plots-withreg}
\end{figure}

In Figure \ref{fig:er-risk-simulation-plots-withreg}, we show the convergence plots of Algorithm \ref{algo:NDRO-min-max} when applied to the explicitly regularized problem \eqref{eq:er-risk-simulation-problem-with-reg}. We consider the same data set from Figure \ref{fig:er-risk-simulation-plots-noreg} but with a slightly larger value of \mmode{\ambradius = 15}, making the conditioning of the problem even worse than that for \mmode{\ambradius = 10}. We show the convergence plots of the algorithm for three different values of \mmode{\alpha = 0, 5 , 10}. Clearly, for \mmode{\alpha = 0}, the problem is not regularized, and the convergence is bad (even worse than that for \mmode{\ambradius = 10} from Figure \ref{fig:er-risk-rho10alpha0}). Then the effect of explicit regularization and how it improves the regularity and convergence can be clearly seen in Figure \ref{fig:er-risk-rho15alpha5} and \ref{fig:er-risk-rho15alpha10}.

\section{Minimum Variance Portfolio Selection}
\label{section:min-variance-portfolio-selection}

The minimum variance portfolio selection~\eqref{NDRO_var} is one of the textbook examples of NDRO~\cite{markowitz1952portfolio, pflug2007ambiguity, blanchet2022distributionally,ref:nguyen2021mean}. This section is dedicated to studying the NDRO regularity of this example and the different aspects of the proposed algorithm in this context. To this end, let \mmode{\widehat{\xi}_i}: \mmode{i = 1,2,\ldots,N} be i.i.d. samples drawn from some unknown underlying distribution \mmode{\P_o}, and let \mmode{\Pref} be the nominal distribution that is uniformly distributed over the samples \mmode{(\widehat{\xi}_i)_i}. Let \mmode{\amb = \wamb{m} \define \{ \Q : \Wdist{m} (\Pref, \Q) \leq \ambradius \}}, where \mmode{\Wdist{m}} is the \mmode{m}-th order Wasserstein distance between the distributions, induced by the transportation cost coming from a norm \mmode{\norm{\cdot}} on \mmode{\Xi}. More specifically, we have
\[
\Wdist{m} (\P, \Q) \define 
\begin{cases}
\begin{aligned}
& \sup_{\pi} && \Big( \EE{\pi} [\norm{u - v}^m] \Big)^{\nicefrac{1}{m}} \quad \text{where \mmode{(u,v)} are \mmode{\pi}-jointly distributed } \\
& \sbjto && \P(u) = \int_v \pi(u,v) dv \quad \text{ and } \quad \Q(v) = \int_u \pi(u,v) du .
\end{aligned}
\end{cases}
\]
Associated with the distribution \mmode{\P}, let \mmode{\sgmap{\P} \define \EE{\P} [\xi \xi\transp]} and \mmode{\mup{\P} \define \EE{\P} [\xi] } denote its first and second moments respectively. Finally, let \mmode{\setx \subset \simplex{n} } denote the set of feasible actions, then we seek to investigate the min-sup problem:
\begin{equation}
\label{eq:min-variance-portfolio-selection}
\min_{x \in \setx} \sup_{\P \in \wamb{m}} \vrisk{x}{\P} \symbolspace{\define}{} x\transp \big( \sgmap{\P} - \mup{\P} \mup{\P}\transp \big) x .
\end{equation}
We shall study in detail the notion of the directional derivative, smoothness, and the resulting FW-oracle with its tractable formulations for~\eqref{eq:min-variance-portfolio-selection} under different conditions.

\subsection{Regularity conditions}
We recall that the variance risk~\eqref{eq:variance-L-r} and note that the mapping \mmode{\P \mapsto \vxpshrt{x}{\P}} is an RR measure in the sense of Definition \ref{def:RR-measure} with the respective sufficient statistic~\mmode{L} and the function~\mmode{r}
\[
L (\xi) = \; (\xi \xi\transp , \, \xi) \quad \text{and} \quad \risk \big(x, (\Sigma, \mu)\big) = \; x\transp \Sigma x - \big( x\transp \mu \big)^2 .
\]
Unlike the risk entropic in Section~\ref{section:er-risk-portfolio-selection}, the variance sufficient statistic~\mmode{L} above is not influenced by the decision~$x$ (cf.\,the two regular risk classes in \eqref{eq:NDRO-regular}). This is a particularly useful feature to control the complexity of the FW iteration~\eqref{eq:FW-iterations-main} by only tracing the finite-dimensional sufficient statistic; see also Remark~\ref{rem:suff-stat}. In view of Lemma~\ref{lemma:RR-directional derivatives} and using the risk functions the short-hand notation \mmode{\vxpshrt{x}{\cdot} \define \vrisk{x}{\cdot}}, we can describe the directional derivatives of \mmode{V_x} as
\begin{equation}
\label{eq:variance-derivative}
\vxpgrad{x}{\P}{\Q} \, = \,  \EE{\Q - \P} \big[ x\transp(\xi - 2 \mup{\P})\xi\transp x \big] \quad \text{for every } \P, \Q \in \wamb{m} .
\end{equation}

\begin{assumption}[Variance moments bound]
\label{assumption:min-variance-bounded-moments}
There exists constants \mmode{B_{\Sigma} , B_{\mu} \geq 0} such that
\begin{equation}
\label{eq:bounded-mean-variance}
\pnorm{\Sigma_{\Q} - \Sigma_{\P}}{o}  \leq B_{\Sigma} \quad \text{ and } \quad \pnorm{\mu_{\Q} - \mu_{\P}}{2}  \leq B_{\mu} \quad \text{ for all } \quad \P,\Q \in \wamb{m} ,
\end{equation}
where \mmode{\Sigma_{\P} = \EE{\P} [\xi\xi\transp]} and \mmode{\mu_{\P} = \EE{\P} [\xi]}.
\end{assumption}

\begin{lemma}[Variance regularity conditions]
\label{lemma:variance-smoothness-assumptions}
Consider the minimum variance portfolio optimization problem \eqref{eq:min-variance-portfolio-selection}, and suppose that Assumption \ref{assumption:min-variance-bounded-moments} holds with constants \mmode{B_{\Sigma} , B_{\mu}}. Let \mmode{V_x (\cdot) \define \vrisk{x}{\cdot}} for every \mmode{x \in \setx}, then the following assertions hold
\begin{enumerate}[label = {\rm (\roman*)}, itemsep = 0mm, topsep = 0mm, leftmargin = *]

\item \emph{Continuous derivatives:} For \mmode{\smooth_1 = 2 \Big( B_{\Sigma} + 2 \big( B_{\mu} + \pnorm{\muref}{2} \big)^2 + B_{\mu}^2 \Big)}, the directional derivatives \mmode{\vxpgrad{x}{\P}{\Q}}, satisfy
\[
\vxpgrad{x}{\P}{\Q} - \vxpgrad{y}{\Q}{\P} \symbolspace{\leq}{\;} \smooth_1 \pnorm{x - y}{2}, \qquad \forall\, x,y \in \setx \quad \forall\,\P,\Q \in \amb. 
\]

\item \emph{Smoothness:} The risk measure \mmode{V_x} is \mmode{\big( 2 B_{\mu}^2 \big)}-smooth in the sense of Definition \ref{def:smoothness}, uniformly over \mmode{x \in \setx}.

\end{enumerate}
\end{lemma}

If Assumption \ref{assumption:min-variance-bounded-moments} holds, then Lemma \ref{lemma:variance-smoothness-assumptions} ensures that the \emph{continuous derivatives} and \emph{smoothness} conditions (ii) and (iii) of Assumption \ref{assumption:NDRO-smoothness} respectively, are satisfied for the minimum variance portfolio optimization problem \eqref{eq:min-variance-portfolio-selection}. However, the uniform strong-convexity assumption need not hold in general. For instance, if the ambiguity set is large (i.e., if \mmode{\ambradius} is large), then all the data points \mmode{\xi_i} can be perturbed within their respective \mmode{\ambradius}-neighborhoods such that the variance corresponding to the perturbed points is rank deficient. Thus, the strong-convexity condition in Assumption \ref{assumption:NDRO-smoothness} is not satisfied at the distribution supported over the perturbed points. In such cases, adding an explicit strongly-convex regularizer to apply Algorithm \ref{algo:NDRO-min-max} (as discussed in Section \ref{subsection:slow-convergence-without-strong-convexity}) not only provides theoretical guarantees for convergence but also improves the speed of the algorithm in our observation; see the numerical simulations concerning Figure~\ref{fig:min-variance-withreg}.

\subsection{The Frank-Wolfe algorithm}
\paragraph{{\bf The FW-Oracle}} Let us consider the FW-problem arising in the NDRO problem of minimum variance portfolio selection \eqref{eq:min-variance-portfolio-selection},
\begin{equation}
\label{eq:min-variance-FW-problem}
\sup\limits_{\Q \in \wamb{m}} \vxpgrad{x}{\P}{\Q} .
\end{equation}
Since the directional derivatives \mmode{\vxpgrad{x}{\P}{\Q}} are affine in \mmode{\Q} for every pair \mmode{(x,\P)}, the corresponding FW-problem is linear. However, the existence and characterization results of the solution to the FW-problem \eqref{eq:min-variance-FW-problem} change depending on the interplay of (i) the Wasserstein distance type~$m$, (ii) the transportation cost~$\|\cdot\|$, and (iii) the support set~$\Xi$ (unbounded, or compact). Therefore, the specific settings for which the corresponding FW-oracle is easy to describe are discussed later in the section. To this end, we simplify \eqref{eq:min-variance-FW-problem}, and study its dual problem which is used later in the characterization of a solution to \eqref{eq:min-variance-FW-problem}.

It is a simple algebraic exercise to verify that
\[
\begin{aligned}
\vxpgrad{x}{\P}{\Q} \; &= \; \EE{\Q - \P} \Big[ x\transp \big( (\xi - 2 \mup{\P})\xi\transp \big)x \big] \\
&= \; \EE{\Q} \left[ (\xi - \mu_{\P})\transp \big(xx\transp\big) (\xi - \mu_{\P}) \right] \symbolspace{-}{\;} x\transp \big( \sgmap{\P} - \mup{\P} \mup{\P}\transp  \big)x .
\end{aligned}
\]
Since the distribution \mmode{\P} is constant in \eqref{eq:min-variance-FW-problem}, eliminating terms that only depend on \mmode{\P} does not affect the set of maximizers. Thus, we have
\begin{equation}
\label{eq:min-variance-FW-simplification}
\argmax\limits_{\Q \in \wamb{m}} \vxpgrad{x}{\P}{\Q} \ = \argmax\limits_{\Q \in \wamb{m}} \EE{\Q} \left[ (\xi - \mu_{\P})\transp \big(xx\transp\big) (\xi - \mu_{\P}) \right] .
\end{equation}
We now focus on a generic version of the linear worst-case distribution problem:
\begin{equation}
\label{eq:min-variance-ldro-general}
\sup_{\Q \in \wamb{m}} \EE{\Q} [(\xi- v)\transp \big(xx\transp\big) (\xi - v)] ,
\end{equation}
for any given \mmode{x, v \in \R{n}}. We known that \cite[Theorem 7]{kuhn2019wasserstein} \eqref{eq:min-variance-ldro-general} admits an equivalent dual formulation given by
\begin{equation}
\label{eq:min-variance-dual}
\inf_{\eta \geq 0} \; \eta \ambradius^m + \frac{1}{N} \summ{i = 1}{N}
\begin{cases}
\begin{aligned}
& \sup_{q_i} && (q_i + \xi_i - v)\transp \big(xx\transp\big) (q_i + \xi_i - v) - \eta \norm{q_i}^m \\
& \sbjto && q_i + \xi_i \in \Xi \quad \text{for all} \quad i = 1,2,\ldots,N.
\end{aligned}
\end{cases}
\end{equation}
For \mmode{\eta \geq 0} and \mmode{\xi \in \Xi}, let
\begin{equation}
\label{eq:min-var-perturbation-problem}
q (\eta , \xi, v) \define \begin{cases}
\begin{aligned}
& \argmax_{q} && \big( (x\transp q) + x\transp (\xi - v)\transp \big)^2 - \eta \norm{q}^m \\
& \sbjto && q + \xi \in \Xi ,
\end{aligned}
\end{cases}
\end{equation}
whenever an optimal solution exists. If the dual problem \eqref{eq:min-variance-dual} admits an optimal solution \mmode{\etax{x}}, and \mmode{q(\etax{x} , \xi_i , v)} exists for each \mmode{i = 1,2,\ldots,N}; then for any collection \mmode{q'_i \in q (\etax{x} , \xi_i , v)}, \mmode{ i = 1,2,\ldots,N}, the discrete distribution \mmode{\Q_x (\xi) \define \frac{1}{N} \sum_{i = 1}^{N} \delta \big( \xi - \xi_i - q_i' \big)} is a maximizer for the linear worst-case distribution problem \eqref{eq:min-variance-ldro-general}.

\paragraph{{\bf Tractable one step FW-update}} We observe that in the minimum variance problem \eqref{eq:min-variance-portfolio-selection}, the only information needed pertaining to a distribution \mmode{\Q \in \wamb{m}} is its first and second order moments \mmode{\mu_{\Q} = \EE{\Q} [\xi]}, \mmode{\Sigma_{\Q} = \EE{\Q} [\xi \xi\transp]} respectively. Thus, in order to solve the DRO problem \eqref{eq:min-variance-portfolio-selection} via the FW-algorithm (Algorithm \ref{algo:NDRO-min-max}), it is apparent that it suffices to track the evolution of these moments rather than the entire distribution (which gets increasingly difficult with the iterations). To this end, a single FW-update \eqref{eq:FW-oracle-NDRO-regime-1}, in Algorithm \ref{algo:NDRO-min-max}, for the minimum-variance problem \eqref{eq:min-variance-portfolio-selection} can be summarised in terms of the finite-dimensional quantities:
\begin{equation}
\label{eq:min-variance-tractable-FW}
\begin{cases}
\begin{aligned}
&\text{Solve for \mmode{x} } &&: \quad  x_k  \in \argmin_{x \in \setx} \; x\transp (\Sigma_k - \mu_k \mu_k\transp) x \\
&\text{FW-oracle } &&: \quad
\begin{cases}
\begin{aligned}
\eta_k &= \eta_{x_k} \quad \text{ and } \quad q_i \in q(\eta_k , \xi_i, \mu_k) \quad \text{for all} \quad i = 1,2, \ldots, N, \\
\mu'_k &= \frac{1}{N} \summ{i = 1}{N} (\xi_i + q_i) \quad \text{ and } \quad \Sigma'_k = \frac{1}{N} \summ{i = 1}{N} (\xi_i + q_i) (\xi_i + q_i)\transp ,
\end{aligned}
\end{cases} \\
&\text{FW-update } &&: \quad
\mu_{k + 1} = \mu_k + \stepsize_k \big( \mu'_k - \mu_k \big) \, , \quad \quad \Sigma_{k + 1} = \Sigma_k + \stepsize_k \big( \Sigma'_k - \Sigma_k \big) .
\end{aligned}
\end{cases}
\end{equation}
The tractable formulation \eqref{eq:min-variance-tractable-FW} of the FW algorithm for variance particularly highlights the discussion in Remark \ref{rem:suff-stat} and \eqref{eq:NDRO-regular} on the sufficient statistic for variance. In particular, it must be noted that the entire information of the distribution is characterized by \mmode{(\mu, \Sigma)}, and therefore, the FW algorithm (Algorithm \ref{algo:NDRO-min-max}) can be simplified as \eqref{eq:min-variance-tractable-FW} which is an iteration over only finite-dimensional quantities \mmode{(\mu , \Sigma)} and \mmode{x}.

We now focus on specific settings of \eqref{eq:min-variance-portfolio-selection} for which the FW-oracle is easily characterized.

{\bf Case 1: unconstrained support.} We consider the setting: \mmode{\Xi = \R{n}}, \mmode{\setx \in \simplex{n}} being any compact subset, and the transportation cost \mmode{\norm{\cdot}} to define the Wasserstein ambiguity set (and let \mmode{\dualnorm{\cdot}} be the associated dual norm).

\begin{lemma}
\label{lemma:min-variance-perturbation-solution}
Consider the maximization problem \eqref{eq:min-var-perturbation-problem} for \mmode{\Xi = \R{n}}, and \mmode{\eta \geq 0}, \mmode{x, v,\xi \in \R{n}}, and let \mmode{J} denote its optimal value. Then the following assertions hold
\begin{enumerate}[label = {\rm (\roman*)}, itemsep = 0mm, topsep = 0mm, leftmargin = *]
\item \emph{(Unbounded)}: \mmode{J = +\infty}, and no optimal solution exists for \eqref{eq:min-var-perturbation-problem} in the following cases: \mmode{(m < 2)}, \mmode{(m > 2, \eta = 0)}, \mmode{(m = 2, \eta < \dualnorm{x}^2)}, and \mmode{(m = 2, \eta = \dualnorm{x}^2, x\transp (\xi - v) \neq 0)}.

\item \emph{(Bounded)}: \mmode{J < +\infty}, and the optimal solution \mmode{q(\eta, \xi, v)} exists for \eqref{eq:min-var-perturbation-problem} in all other cases, and in particular,
\begin{enumerate}[label = {\rm (\alph*)}, itemsep = 0mm, topsep = 0mm, leftmargin = *]
    \item \mmode{(m > 2 , \eta > 0)}, both \mmode{J < +\infty}, and \mmode{q(\eta, \xi, v)} exists.

    \item \mmode{(m = 2, \eta > \dualnorm{x}^2)}, then we have
    \[
    J = \frac{\eta \abs{x\transp (\xi - v)}^2}{\eta - \pnorm{x}{*}^2} \, , \quad q(\eta , \xi, v) = \frac{\pnorm{x}{*} \big( x\transp (\xi - v) \big)} {\eta - \pnorm{x}{*}^2} \, \bar{q}_x \; \text{for } \; \bar{q}_x \in \argmin\limits_{\norm{\bar{q}} \leq 1} \ x\transp \bar{q} .
    \]

    \item \mmode{(m = 2, \eta = \dualnorm{x}^2, x\transp(\xi - v) = 0) }, then \mmode{J = 0} and \mmode{q(\eta, \xi, v) = \{ s \bar{q}_x : s \geq 0 \}}.
\end{enumerate}
\end{enumerate}
\end{lemma}

\begin{lemma}[Variance FW-oracle: unconstrained support]
\label{lemma:min-var-ldro-solution-for-m2}
Consider the linear worst-case-distribution problem \eqref{eq:min-variance-ldro-general} and its dual \eqref{eq:min-variance-dual}, under the setting, \mmode{\Xi = \R{n}}, \mmode{m = 2}, and any \mmode{x \in \setx}. The solutions \mmode{(\Q_x , \etax{x})} to \eqref{eq:min-variance-ldro-general} and its dual \eqref{eq:min-variance-dual} are given by
\begin{equation}
\label{eq:min-var-unconstrained-support-FW-oracle}
\begin{cases}
\begin{aligned}
\etax{x} &= \pnorm{x}{*}^2 + \frac{\pnorm{x}{*}}{\rho} \sqrt{\frac{1}{N} \summ{i = 1}{N} \abs{x\transp (\xi_i - v)}^2 } , \\
\Q_x (\xi) &= \frac{1}{N} \summ{i = 1}{N} \delta \big( \xi - (\xi_i + q'_i) \big) \quad \text{for any }  q'_i \in q(\etax{x} , \xi_i, v), \ i = 1,2,\ldots,N.
\end{aligned}
\end{cases}
\end{equation}
\end{lemma}
It turns out that for the special case of \mmode{m = 2}, the problem \eqref{eq:min-variance-portfolio-selection} with unconstrained support reduces to a simple empirical risk minimization problem. This was already discovered in \cite{blanchet2022distributionally} specifically for the case when \mmode{\setx = \setprtfl \define \{ x\in \simplex{n} : \EE{\Q} [x\transp \xi] \geq \alphamin, \text{ for all } \Q \in \amb \} }, for any \mmode{\alphamin}. We generalize it slightly by showing that a similar conclusion holds for any compact set \mmode{\setx \subset \simplex{n}}. Moreover, we provide an alternate proof completely based on first-order optimality conditions and the FW-oracle, for the min-max problem \eqref{eq:min-variance-portfolio-selection}. We emphasize here that since it is shown explicitly that the min-max problem \eqref{eq:min-variance-portfolio-selection} reduces to an empirical minimization problem, a solution can be computed without the need to run the FW-algorithm \eqref{eq:min-variance-tractable-FW} iteratively. More importantly, this also alleviates the need to ensure that the convex regularity assumptions hold for this special case of \eqref{eq:min-variance-portfolio-selection}.

\begin{proposition}[Variance saddle point]
\label{proposition:min-variance-blanchet}
Consider the minimum variance portfolio optimization problem \eqref{eq:min-variance-portfolio-selection} in the setting of \mmode{m = 2}, and \mmode{\Xi = \R{n}}. Let 
\begin{equation}
x\opt \in \argmin_{x \in \setx} \sqrt{\inprod{x}{(\sigmaref - \muref \muref\transp) x} } \, + \rho \pnorm{x}{*} ,
\end{equation}
then there exists some \mmode{ \bar{q}_{x\opt} \in \arg\max_{\norm{\bar{q} } \leq 1} \inprod{x\opt}{\bar{q}} } such that for 
\begin{equation}
\label{eq:popt-explicit-solution}
\P\opt \define \frac{1}{N} \summ{i = 1}{N} \delta \big( \xi - (\xi_i + q_i\opt) \big) , \quad \text{where} \quad q_i\opt = \frac{\ambradius \inprod{x\opt}{\xi_i - \muref} }{\sqrt{ \inprod{x\opt}{ (\sigmaref - \muref \muref\transp) x\opt } }} \bar{q}_{x\opt} \quad \forall i \leq N,
\end{equation}
the pair \mmode{(x\opt, \P\opt)} is a saddle point to \eqref{eq:min-variance-portfolio-selection}.
\end{proposition}

{\bf Case 2: ellipsoidal support.} We consider the setting: \mmode{\Xi = \ellipsoid \define \{ \xi : \inprod{\xi}{M \xi} \leq 1 \} } for some \mmode{M \succeq 0}, \mmode{m = 2}, and the underlying transportation cost defining the Wasserstein distance to be \mmode{\norm{\cdot} = \pnorm{\cdot}{2}}. Under this setting, the FW problem in  \eqref{eq:FW-oracle-NDRO-regime-1} involves maximizing a quadratic function subject to convex ellipsoidal constraints. Even if this problem is non-convex in general, it admits a tractable reformulation as an SDP via the celebrated \emph{S-procedure} \cite[Appendix B]{boyd2004convex}, \cite{uhlig1979recurring}.
\begin{lemma}[Variance FW-oracle: ellipsoidal support]
\label{lemma:FW-oracle-min-variance-ellipsoidal-support}
Consider the minimum variance portfolio optimization problem \eqref{eq:min-variance-portfolio-selection} under the setting of the ellipsoidal support. The corresponding dual problem \eqref{eq:min-variance-dual} of the linear worst-case-distribution problem \eqref{eq:min-variance-ldro-general} is equivalent to the SDP
\begin{equation}
\label{eq:FW-problem-ellipsoid-support-SDP}
\begin{cases}
\begin{aligned}
& \min_{\eta \in \R{}, \ \lambda, \theta \in \R{N} }    
&& \eta \rho^2 - \frac{1}{N} \summ{i = 1}{N} \theta_i \\
& \sbjto &&
\begin{cases}
\begin{aligned}
& \ \ \eta \geq 0, \text{ and } \lambda_i \geq 0 \ \forall i \leq N, \\
& \begin{bmatrix}
\eta \identity{n} - x x\transp + \lambda_i M  & (xx\transp)v - \eta \xi_i  \\
(xx\transp)v\transp - \eta \xi_i\transp &  \eta \pnorm{\xi_i}{2}^2 - (x\transp v)^2 - \lambda_i - \theta_i
\end{bmatrix}
\succeq 0 , \  \forall i \leq N .
\end{aligned}
\end{cases}
\end{aligned}
\end{cases}
\end{equation}
Moreover, for any solution \mmode{(\eta\opt , \lambda\opt , \theta\opt)}, to the SDP \eqref{eq:FW-problem-ellipsoid-support-SDP}, the pair \mmode{(\Q_x , \etax{x})} given by
\begin{equation}
\label{eq:min-var-ellipsoidal-support-FW-oracle}
\etax{x} = \eta\opt , \quad \text{and} \quad \Q_x (\xi) = \frac{1}{N} \summ{i = 1}{N} \delta \Big( \xi - \big( \eta\opt \identity{n} - xx\transp + \lambda\opt_i M \big)^{-1} \big( \eta\opt \xi_i - x x\transp v \big) \Big) ,
\end{equation}
is a solution to \eqref{eq:min-variance-ldro-general} and its dual \eqref{eq:min-variance-dual} respectively.
\end{lemma}

\subsection{Simulation results}
We validate the convergence attributes of Algorithm \ref{algo:NDRO-min-max} for the NDRO problem of the minimum variance portfolio selection \eqref{eq:min-variance-portfolio-selection} with ellipsoidal support (i.e., \mmode{\Xi = \ellipsoid \subset \R{n} }). The positive definite matrix \mmode{M} that characterizes the support \mmode{\ellipsoid}, is generated randomly and it is ensured to be reasonably well conditioned. The samples \mmode{\xi_i}, for \mmode{i = 1,2,\ldots,N = 2n}, are drawn randomly from \mmode{\ellipsoid}, and the distribution \mmode{\Pref} is uniformly distributed over the samples. Then for \mmode{n = 25} and various values of \mmode{\ambradius} (the radius of the ambiguity set), we solve
\begin{equation}
\label{eq:min-var-ellips-simulation-problem}
V\opt \define \ \min_{x \in \simplex{n}}  \sup_{\P \in \wamb{2}} \vrisk{x}{\P} .   
\end{equation}
The proposed FW-based method (Algorithm \ref{algo:NDRO-min-max} with the simplified FW-update \eqref{eq:min-variance-tractable-FW}) is implemented in {\fontfamily{pcr}\selectfont MATLAB}, wherein the SDP \eqref{eq:FW-problem-ellipsoid-support-SDP} corresponding to the FW-oracle and the minimization problem: \mmode{\min_{x \in \simplex{n}} \ \vrisk{x}{\P_k}} are solved using the {\fontfamily{pcr}\selectfont cvx} solver.

First, Figure \ref{fig:min-variance-noreg} shows the convergence of Algorithm \ref{algo:NDRO-min-max} for \mmode{\keps = 75} iterations, and for three distinct values of \mmode{\ambradius = 0.1, 0.5}, and \mmode{1}. The plots on the top show the evolution of the primal and dual functions: \mmode{\min_{x \in \simplex{n}} \vrisk{x}{\P_k}} and \mmode{\sup_{\P \in \wamb{2}} \vrisk{x_k}{\P}} respectively, w.r.t.\,the iteration \mmode{k} of the algorithm. Whereas, the plots at the bottom of the figure show the evolution of the two sub-optimality metrics:
\[
\begin{cases}
\begin{aligned}
\text{\emph{primal sub-optimality}} \quad &: \quad \quad V\opt - \min\limits_{x \in \simplex{n}} \vrisk{x}{\P_k} \quad \text{and} \\
\text{\emph{duality gap}} \quad &:  \quad \sup\limits_{\P \in \wamb{2}} \vrisk{x_k}{\P} \symbolspace{-}{\;} \min\limits_{x \in \simplex{n}} \ \vrisk{x}{\P_k}   .
\end{aligned}
\end{cases}
\]
Since Algorithm \ref{algo:NDRO-min-max} explicitly solves the minimization over \mmode{x} in each iteration (see \eqref{eq:FW-oracle-NDRO-regime-1}), the primal function is readily available. However, this is not the case with the dual function which needs to be computed independently at each iteration for the current iterate \mmode{x_k}. We compute it by running the FW-update \eqref{eq:min-variance-tractable-FW} for several iterations (\mmode{\sim\keps}) independently with \mmode{x_k} held fixed, this amounts to applying the FW-algorithm \eqref{eq:FW-iterations-main} for the risk measure \mmode{\vxpshrt{x_k}{\cdot}}.
\begin{figure}[t]
    \centering
    \begin{subfigure}[b]{0.32\textwidth}
    \raggedleft
\includegraphics[width=  5cm]{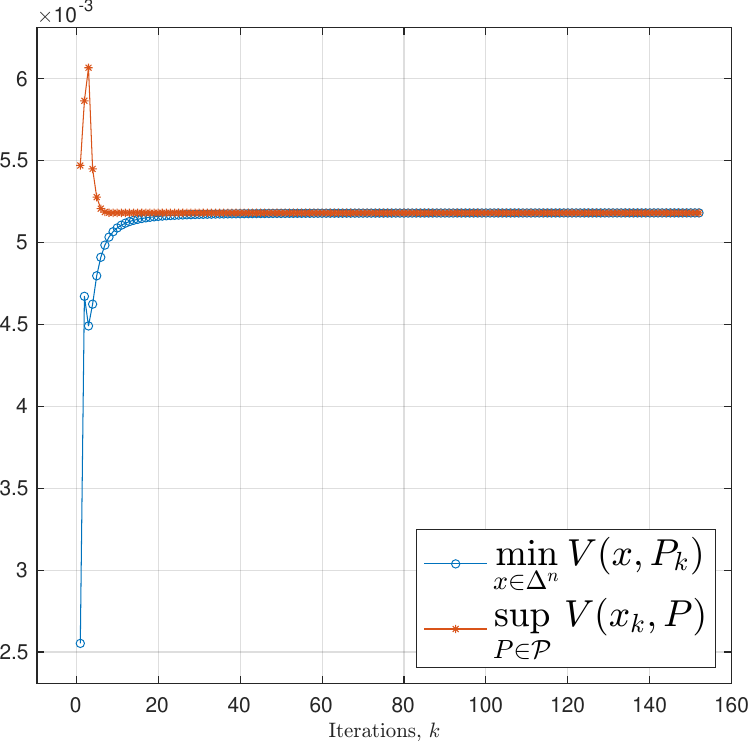}
\vspace{0.25cm}
\includegraphics[width=  5cm]{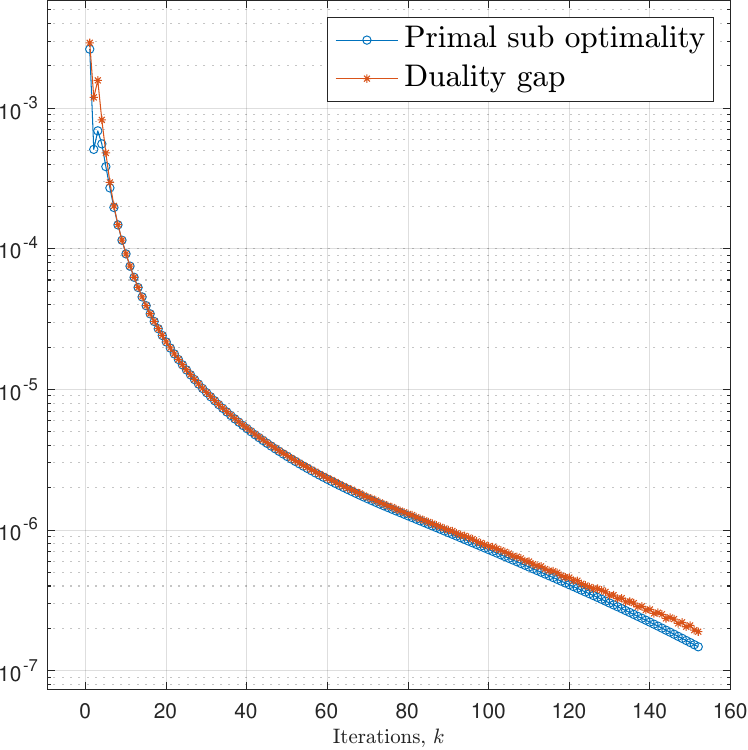}
    \caption{\mmode{\ambradius = 0.1}, \mmode{\alpha = 0}}
    \label{fig:rho-0.1}
    \end{subfigure}
    \hfill
    \begin{subfigure}[b]{0.32\textwidth}
    \raggedleft
\includegraphics[width=  5cm]{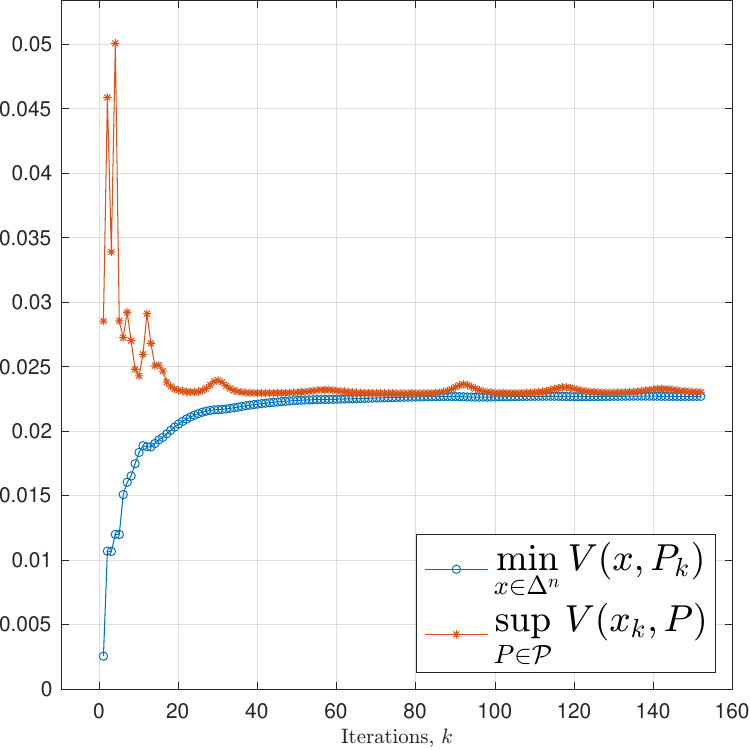}
\vspace{0.25cm}
\includegraphics[width=  5cm]{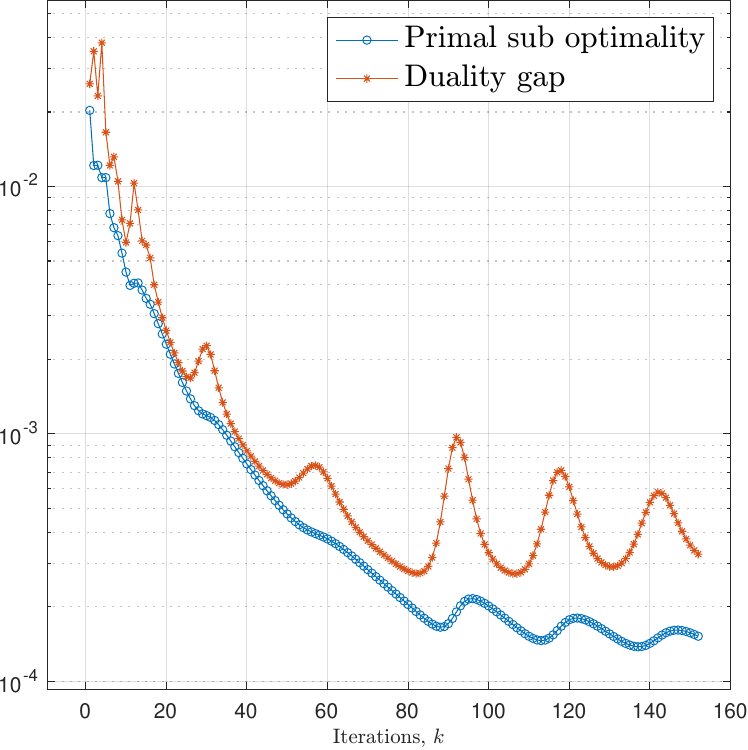}
    \caption{\mmode{\ambradius = 0.5}, \mmode{\alpha = 0}}
    \label{fig:rho-0.5}
    \end{subfigure}
    \hfill
    \begin{subfigure}[b]{0.32\textwidth}
    \raggedleft
\includegraphics[width=  5cm]{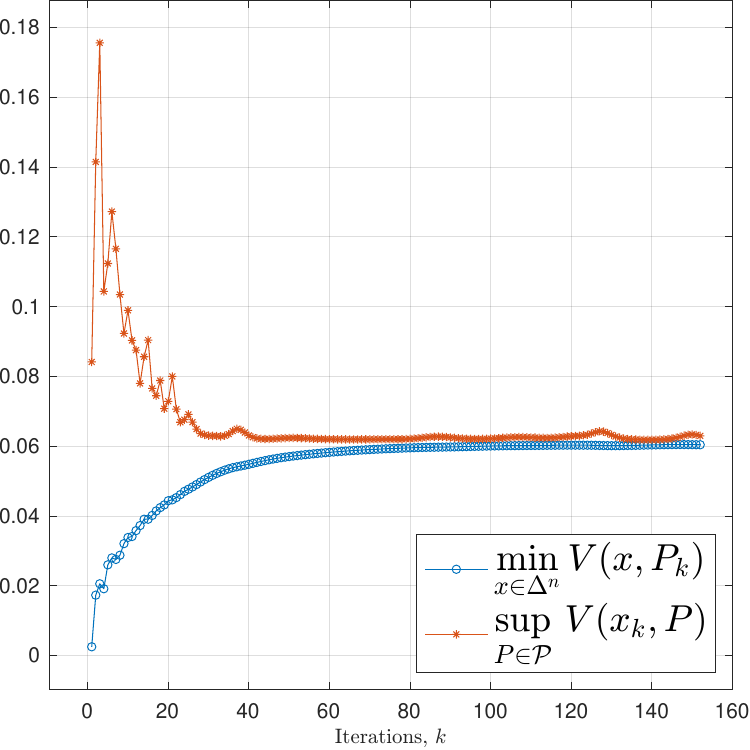}
\vspace{0.25cm}
\includegraphics[width=  5cm]{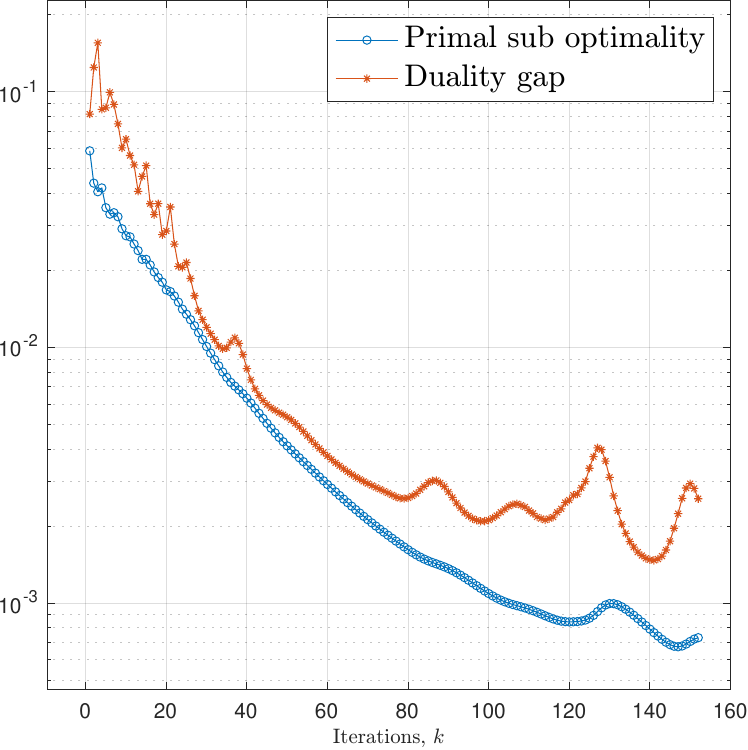}
    \caption{\mmode{\ambradius = 1}, \mmode{\alpha = 0}}
    \label{fig:rho-1.0}
    \end{subfigure}
    \caption{Convergence plots for Algorithm \ref{algo:NDRO-min-max} with \mmode{\keps = 75}, applied to \eqref{eq:min-var-ellips-simulation-problem}.}
    \label{fig:min-variance-noreg}
\end{figure}

\paragraph{{\bf With explicit regularization of \mmode{x}}} We observe in Figure \ref{fig:min-variance-noreg} that, as the value of \mmode{\ambradius} increases, the convergence of the algorithm becomes less smooth, and also slower. Perhaps, this can be attributed to the fact that the strong-convexity assumption in Assumption \ref{assumption:NDRO-smoothness} fails. This is so because, a larger value of \mmode{\ambradius} allows the samples \mmode{(\xi_i)_i} to be perturbed in such a way that the variance matrix of the perturbed points is rank deficient. Thus, the function \mmode{\vrisk{x}{\P}} is not strongly convex in \mmode{x} for \mmode{\P} corresponding to the perturbed points.

To remedy this, as suggested in Section \ref{subsection:slow-convergence-without-strong-convexity}, we compute an approximate saddle point of \eqref{eq:min-var-ellips-simulation-problem} by applying the FW-algorithm to the explicitly regularized min-sup problem
\begin{equation}
\label{eq:min-var-ellips-simulation-problem-with-reg}
\min_{x \in \simplex{n}} \sup_{\P \in \wamb{2}} \frac{\alpha}{2} \pnorm{x}{2}^2 + \vrisk{x}{\P} .  
\end{equation}
Since \mmode{\pnorm{x}{2} \leq \pnorm{x}{1} = 1} for all \mmode{x \in \simplex{n}}, we know that an \mmode{\eps}-saddle point of \eqref{eq:min-var-ellips-simulation-problem} can be computed for any \mmode{\eps > 0} by applying Algorithm \ref{algo:NDRO-min-max} to \eqref{eq:min-var-ellips-simulation-problem-with-reg} with \mmode{\alpha = 2\eps}. We emphasize that both the primal function: \mmode{\min_{x \in \simplex{n}} \vrisk{x}{\P_k}} and the dual function: \mmode{\sup_{\P \in \wamb{2}} \vrisk{x_k}{\P}} are not accessible in the implementation of Algorithm \ref{algo:NDRO-min-max} for the explicitly regularized objective function. Thus, these quantities are computed separately at each iteration of the algorithm to record the primal sub-optimality and the duality gap.

\begin{figure}[t]
    \centering
    \begin{subfigure}[b]{0.32\textwidth}
    \raggedleft
\includegraphics[width=  5cm]{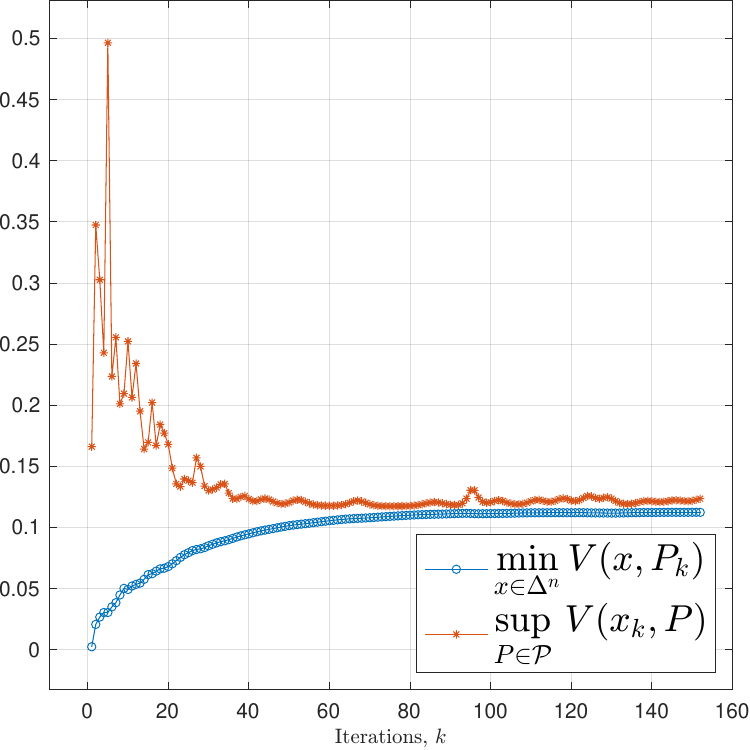}
\vspace{0.25cm}
\includegraphics[width=  5cm]{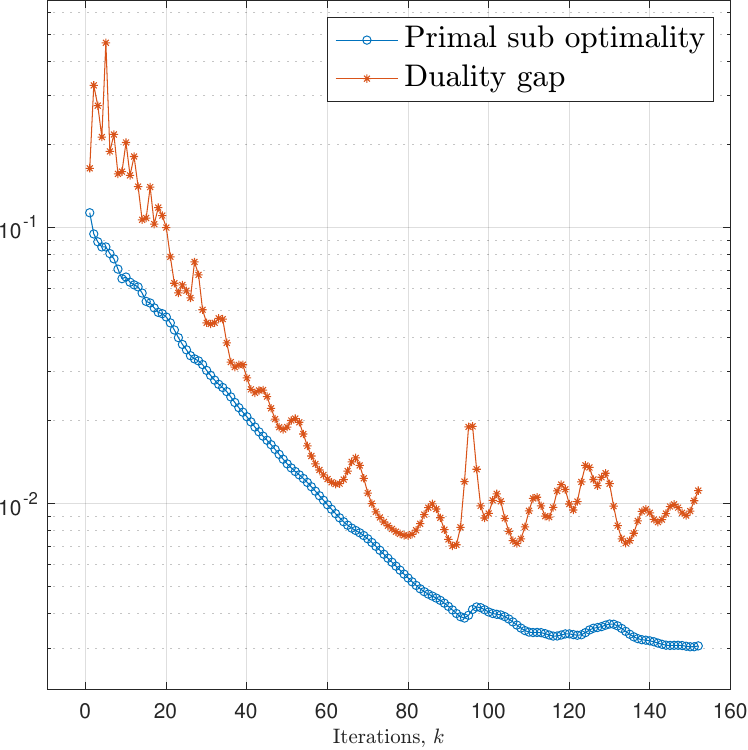}
    \caption{\mmode{\ambradius = 1.5}, \mmode{\alpha = 0}}
    \label{fig:alpha-min}
    \end{subfigure}
     \hfill
    \begin{subfigure}[b]{0.32\textwidth}
    \raggedleft
\includegraphics[width=  5cm]{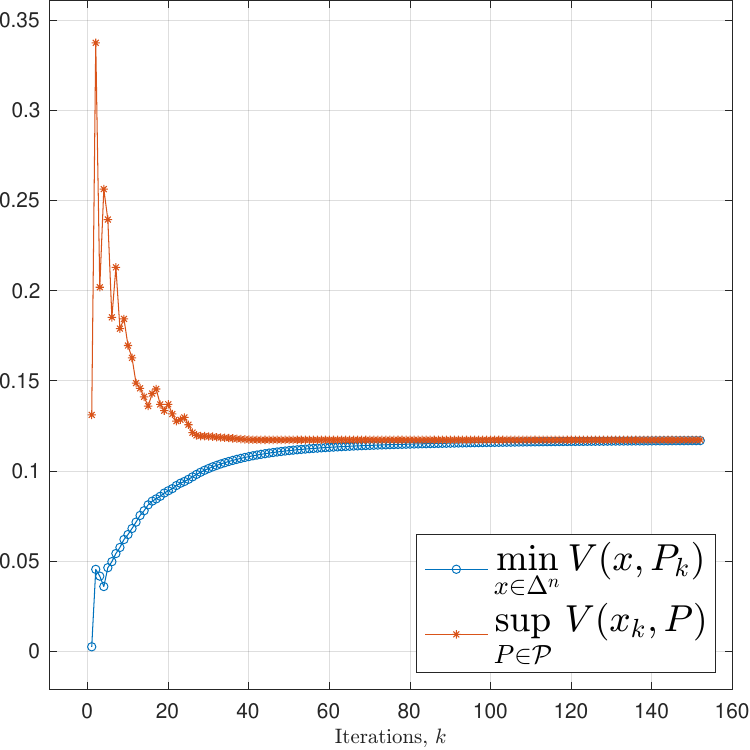}
\vspace{0.25cm}
\includegraphics[width=  5cm]{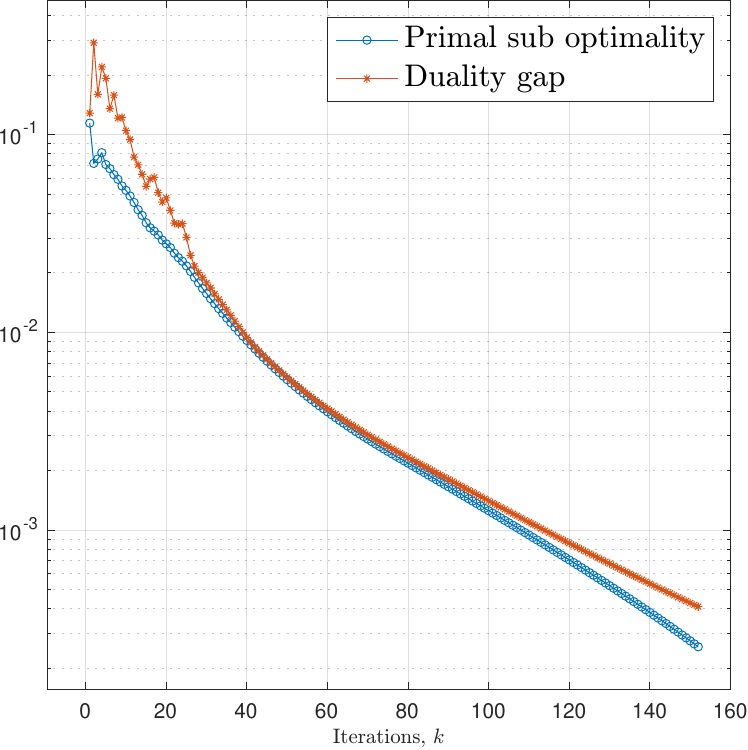}
    \caption{\mmode{\ambradius = 1.5}, \mmode{\alpha = 0.1}}
    \label{fig:alpha-mid}
    \end{subfigure}
    \hfill
    \begin{subfigure}[b]{0.32\textwidth}
    \raggedleft
\includegraphics[width=  5cm]{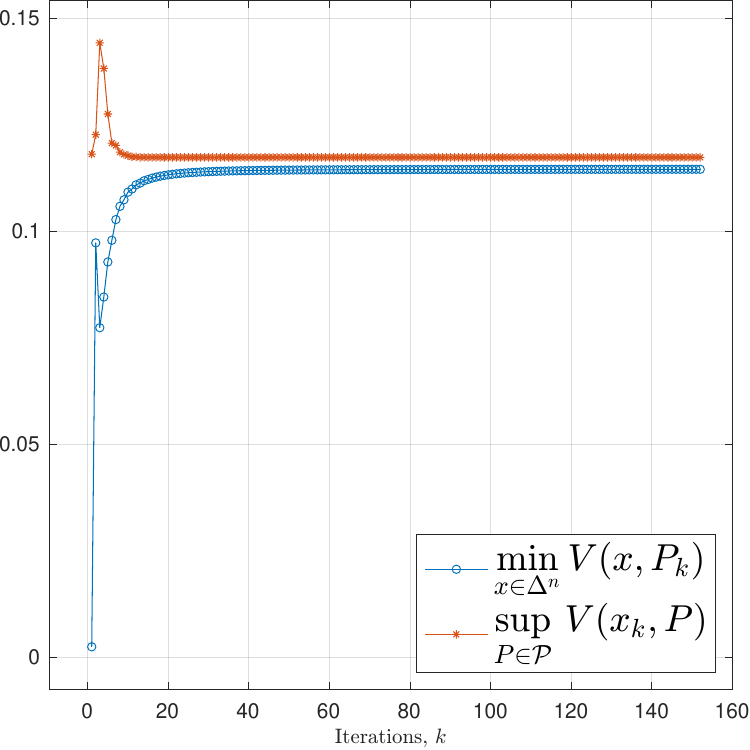}
\vspace{0.25cm}
\includegraphics[width=  5cm]{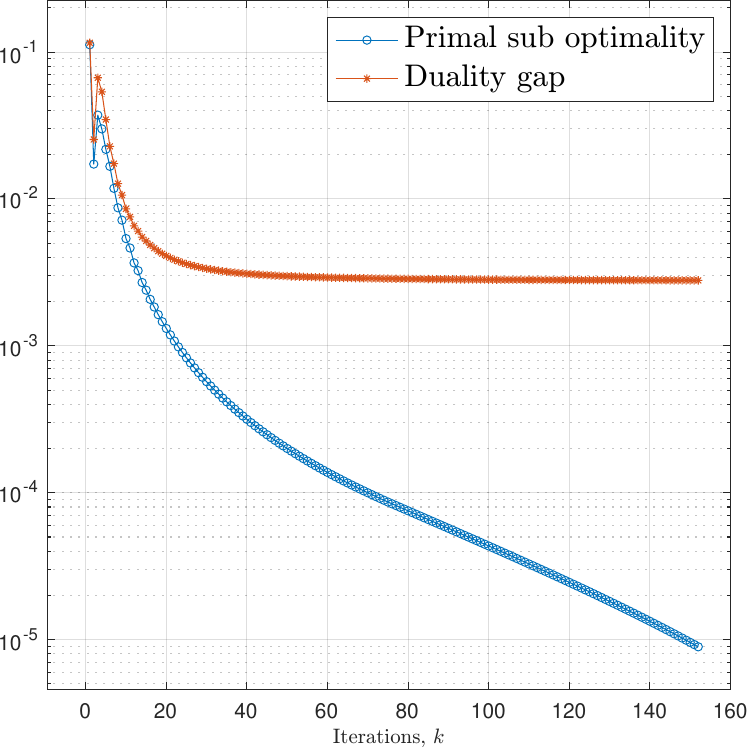}
    \caption{\mmode{\ambradius = 1.5}, \mmode{\alpha = 1}}
    \label{fig:alpha-max}
    \end{subfigure}
    \caption{Convergence plots for Algorithm \ref{algo:NDRO-min-max} with \mmode{\keps = 75}, applied to the explicitly regularized approximation \eqref{eq:min-var-ellips-simulation-problem-with-reg}.}
    \label{fig:min-variance-withreg}
\end{figure}

In Figure \ref{fig:min-variance-withreg}, we show the convergence plots of Algorithm \ref{algo:NDRO-min-max} when applied to the explicitly regularized problem \eqref{eq:min-var-ellips-simulation-problem-with-reg}. We consider the same data set from Figure \ref{fig:min-variance-noreg} but with a slightly larger value of \mmode{\ambradius = 1.5}, making the conditioning of the problem even worse than that for \mmode{\ambradius = 1}.

We show the convergence plots of the algorithm for three different values of \mmode{\alpha = 0, 0.1, 1}. Clearly, for \mmode{\alpha = 0}, the problem is not regularized, and the convergence is bad (even worse than that for \mmode{\ambradius = 1} from Figure \ref{fig:rho-1.0}). Then the effect of explicit regularization and how it improves the regularity and convergence can be clearly seen in Figure \ref{fig:min-variance-withreg}. For \mmode{\alpha = 0.1} and \mmode{1}, we see that the solution computed is an \mmode{\eps}-saddle point for the respective value of \mmode{\eps} (i.e., \mmode{\eps = 0.05} for \mmode{\alpha = 0.1} and \mmode{\eps = 0.5} for \mmode{\alpha = 1}). In addition, Figures \ref{fig:alpha-mid} and \ref{fig:alpha-max} highlight the trade-off between the speed of convergence and worst-case performance metrics. For the larger value of \mmode{\alpha = 1}, due to better regularity, the algorithm is faster which can be seen in Figure \ref{fig:alpha-max} that both the curves of primal (blue) and dual (red) functions achieve within one percent of their final value in less than 20 iterations. Whereas, with \mmode{\alpha = 0.1}, even after 60 iterations, the primal function (blue curve) is only within 3.5 percent of its final value. However, this improved speed comes at a cost as clearly evident from Figures \ref{fig:alpha-mid} and \ref{fig:alpha-max} that the worst-case cost for \mmode{\alpha = 0.1} (0.11731) is smaller than that for \mmode{\alpha = 1} (0.117403). Furthermore, even though the solution computed with both \mmode{\alpha = 0,1} and \mmode{1}, is within the respective sub-optimality levels, the duality gap however, actually goes to zero for \mmode{\alpha = 0.1}, which is not the case for \mmode{\alpha = 1}.

\section{Technical Proofs}
\label{section:technical-proofs}

\subsection{Proofs of Section \ref{section:derivatives-and-smoothness} (derivatives)}

In this part, we cover the technical proofs of the theoretical statements in Section \ref{section:derivatives-and-smoothness}.

\begin{proof}[Proof of Lemma \ref{lemma:RR-directional derivatives}]
From the definition \ref{def:directional-derivative}, we see that
\begin{align*}
	   \RPgrad{\P}{\Q} &= \lim_{\eps \downarrow 0} \frac{1}{\eps} \left( \RP{\P + \eps (\Q - \P)} - \RP{\P} \right)   \\
	   &=   \lim_{\eps \downarrow 0} \frac{1}{\eps} \left( \risk \left( \EE{\P} [L(\xi)] + \eps \big( \EE{\Q-\P} [L (\xi)] \big) \right) - \risk \left( \EE{\P} [\xi] \right) \right)   \\
	   &=   \lim_{\eps \downarrow 0} \frac{1}{\eps} \left( \risk \left( \EE{\P} [\xi] \right) + \eps \inprod{ \nabla \risk (\EE{\P} [L(\xi)]) }{ \EE{\Q-\P} [L(\xi)] } + o(\eps^2) - \risk \left( \EE{\P} [\xi] \right) \right)   \\
	   &= \inprod{ \nabla \risk (\EE{\P} [L(\xi)]) }{ \EE{\Q-\P} [L(\xi)] } \\
	   &= \EE{\Q-\P} \left[ \inprod{ \nabla \risk (\EE{\P} [L(\xi)]) }{  L(\xi)} \right].
\end{align*}
\end{proof}

\begin{proof}[Proof of Proposition \ref{proposition:derivative-properties}]
From Definition \ref{def:directional-derivative}, we have
\[
\begin{aligned}
\RPgrad{\P}{\P_{\stepsize}} &= \lim_{\eps \downarrow 0} \frac{1}{\eps} \Big( \RP{\P + \eps (\P_{\stepsize} - \P )} - \RP{\P}  \Big) = \lim_{\eps \downarrow 0} \frac{\stepsize}{(\eps \stepsize)} \Big( \RP{\P + \eps \stepsize (\Q - \P )} - \RP{\P} \Big) \\
&= \ \stepsize \lim_{\eps' \downarrow 0} \frac{1}{\eps'} \Big( \RP{\P + \eps' (\Q - \P )} - \RP{\P} \Big) \\
&= \ \stepsize \RPgrad{\P}{\Q},
\end{aligned}
\]
which establishes \eqref{eq:positive-homogeniety}. To establish \eqref{eq:concavity-linear-bound}, we see that
\[
\begin{aligned}
\RPgrad{\P}{\Q} &= \lim_{\eps \downarrow 0} \frac{1}{\eps} \Big( \RP{\P + \eps (\Q - \P )} - \RP{\P}  \Big) \ \geq \ \lim_{\eps \downarrow 0} \frac{1}{\eps} \Big( \RP{\P} + \eps \big( \RP{\Q} - \RP{\P} \big) - \RP{\P} \Big) \\
&\geq \ \RP{\Q} - \RP{\P} ,
\end{aligned}
\]
where the first inequality is due to concavity of \mmode{\P \mapsto \RP{\P}}. This completes the proof.
\end{proof}

\begin{proof}[Proof of Lemma \ref{lemma:RR-smoothness}]
For any \mmode{\P, \Q \in \amb} with \mmode{\P_{\stepsize} = \P+\stepsize(\Q - \P)} and \mmode{\stepsize \in [0,1]} we see that
\[
\begin{aligned}
& \RPgrad{\P_{\stepsize}}{\P} + \RPgrad{\P}{\P_{\stepsize}} \\
& \quad = \inprod{ \nabla \risk (\EE{\P} [L(\xi)]) }{ \EE{\P_{\stepsize}} [L(\xi)] - \EE{\P} [L(\xi)]} \; + \; \inprod{ \nabla \risk (\EE{\P_{\stepsize}} [L(\xi)]) }{ \EE{\P} [L(\xi)] - \EE{\P_{\stepsize}} [L(\xi)]} \\
& \quad = \inprod{ \nabla \risk (\EE{\P} [L(\xi)]) \; - \; \nabla \risk (\EE{\P_{\stepsize}} [L(\xi)]) }{ \EE{\P_{\stepsize}} [L(\xi)] - \EE{\P} [L(\xi)]} \\
& \quad \leq \beta \norm{ \EE{\P_{\stepsize}} [L (\xi)] - \EE{\P} [L (\xi)] }^2 \leq \stepsize^2 \beta \norm{ \EE{\Q} [L (\xi)] - \EE{\P} [L (\xi)] }^2 \leq \stepsize^2 \big(\beta d^2 \big).
\end{aligned}
\]
\end{proof}

\subsection{Proofs of Section \ref{section:FW-algorithm} (FW-algorithm)}

In this part, we cover the technical proofs of the theoretical statements in Section \ref{section:FW-algorithm}.
\begin{proof}[Proof for Lemma \ref{lemma:FW-one-step-bounds}]
For \mmode{\P} and \mmode{\P_{\stepsize}}, we see that the inequalities hold.
\begin{align*}
\RP{\P} - \RP{\P_{\stepsize}} \ 
& \leq \ \RPgrad{\P_{\stepsize}}{\P} && \text{from \eqref{eq:concavity-linear-bound}} \\
& = \ \RPgrad{\P_{\stepsize}}{\P} + \RPgrad{\P}{\P_{\stepsize}} - \RPgrad{\P}{\P_{\stepsize}} && \\
& \leq \ \stepsize^2 \smooth - \stepsize \RPgrad{\P}{\Q'} &&\text{from smoothness \eqref{eq:G-smoothness}, and \eqref{eq:positive-homogeniety}} \\
& \leq \ \stepsize^2 \smooth (1 + \delta) \ - \ \stepsize \; \RPgrad{\P}{\Q}  && \text{from \eqref{eq:FW-oracle}}.
\end{align*}
Rearranging terms and using  \eqref{eq:concavity-linear-bound}, we get
\[
R(\Q) - \RP{\P} \quad \leq \quad \RPgrad{\P}{\Q} \quad \leq \quad \frac{1}{\stepsize} \Big( \RP{\P_{\stepsize}} - \RP{\P} \Big) + \stepsize \smooth (1+\delta) ,
\]
wherein, by considering the supremum over \mmode{\Q \in \amb}, the inequalities in \eqref{eq:fw-one-step-ineq} follow at once.
\end{proof}

\begin{proof}[Proof for Proposition \ref{proposition:FW-convergence}]
Let \mmode{(\P_k)_k} be the sequence of Frank-Wolfe iterates updated as in \eqref{eq:FW-iterations-main}. From the one-step inequality \eqref{eq:fw-one-step-sub-optimality}, we have
\begin{equation}
\label{eq:FW-induction-ineq}
    R\opt - \RP{\P_{k+1}} \leq \stepsize_k^2 \smooth (1 + \delta) + (1 - \stepsize_k) \big( R\opt - \RP{\P_k} \big) \quad \text{for every } k \geq 0.
\end{equation}
Using \eqref{eq:FW-induction-ineq}, we prove the primal convergence of FW-algorithm following the classical approach \cite{jaggi2013revisiting,clarkson2010coresets} by the method of induction.

\emph{Base case}: For \mmode{k = 1}, since \mmode{\P_1 = \P_0 + \stepsize_0 (\Q_0 - \P_0)} we deduce from \eqref{eq:FW-induction-ineq} using \mmode{\stepsize = \stepsize_0 = 1} that
\[
R\opt - \RP{\P_1} \leq \smooth (1 + \delta) \ < \ \frac{4\smooth}{3} (1 + \delta) .
\]
Thus, \eqref{eq:fw-primal-convergence} holds for \mmode{k = 1}.

\emph{Induction step}: Now, assuming that \eqref{eq:fw-primal-convergence} holds for some \mmode{k \geq 1}, we deduce from \eqref{eq:FW-induction-ineq} that
\[
\begin{aligned}
R\opt - \RP{\P_{k+1}} \ 
&\leq \ \stepsize_k^2 \smooth (1 + \delta) + (1 - \stepsize_k) \big( R\opt - \RP{\P_k} \big) \quad \\
&\leq \ \frac{4\smooth(1 + \delta)}{(k + 2)^2} + \frac{k}{k + 2} \frac{4\smooth(1 + \delta)}{k + 2} = \frac{4\smooth(1 + \delta)}{k + 2} \frac{1 + k}{k + 2} \\
&\leq \ \frac{4\smooth (1 + \delta)}{k + 3}, \quad \text{since } (k + 2)^2 \leq (k + 1)(k + 3).
\end{aligned}
\]
Thus, we conclude that \eqref{eq:fw-primal-convergence} indeed holds for every \mmode{k \geq 1}. The proof is now complete.
\end{proof}

\begin{proof}[Proof for Proposition \ref{proposition:FW-gap}]
Following the ideas of \cite{clarkson2010coresets,jaggi2013revisiting}, we first prove by the method of contradiction that there exists \mmode{\khat} such that 
\begin{equation}
\label{eq:fw-sup-sub-optimality}
\sup_{\Q \in \amb} \ \RPgrad{\P_\khat}{\Q} \leq \frac{4\smooth(1 + \delta)}{K + 2} \quad \text{for some } \khat \in \{K, \ldots, 2K + 1 \} .
\end{equation}
Suppose that \mmode{\sup_{\Q \in \amb} \ \RPgrad{\P_k}{\Q} \ > \ \frac{4\smooth(1 + \delta)}{K + 2}} for every \mmode{k \in \{K, \ldots, 2K + 1\}}, then by considering \mmode{\stepsize = \frac{2}{K + 2}} in \eqref{eq:fw-one-step-ineq} and rearranging terms, we see that for every \mmode{k \in \{K, \ldots, 2K + 1\}}, we have the inequality
\[
\begin{aligned}
\RP{\P_{k}} - \RP{\P_{k + 1}} \ 
&\leq \ \frac{4}{(K + 2)^2} \smooth (1 + \delta) \ - \ \frac{2}{K + 2} \ \sup_{\Q \in \amb} \ \RPgrad{\P_k}{\Q} \\
&< \ \frac{4}{(K + 2)^2} \smooth (1 + \delta) \ - \ \frac{8}{(K + 2)^2} \smooth (1+\delta) \\
&< -\; \frac{4}{K + 2} \smooth (1+\delta) .
\end{aligned}
\]
Summing these inequalities for \mmode{k = K, \ldots, 2K + 1}, we get \mmode{ \RP{\P_K} - \RP{\P_{2K + 2}} \; < \; - \; \big( \nicefrac{4}{K + 2} \big) \smooth (1+\delta)}. Combining this with the sub-optimality of \mmode{\RP{\P_K}} from \eqref{eq:fw-primal-convergence} finally gives
\[
R\opt - \RP{\P_{2K + 2}} \ < \ - \; \frac{4}{K + 2} \smooth (1+\delta) \ + \; \frac{4}{K + 2} \smooth (1+\delta) \ < \; 0 ,
\]
which is a contradiction. Therefore, there must exist some \mmode{\khat \in \{K,\ldots, 2K + 1\}}, such that \eqref{eq:fw-sup-sub-optimality} holds.

For every \mmode{\khat} such that \eqref{eq:fw-sup-sub-optimality} holds, we also necessarily have 
\begin{equation}
\label{eq:fwgap-upper-bound}
\fwgap_{\khat} = \RPgrad{\P_\khat}{\Q_\khat} \leq \sup_{\Q \in \amb} \ \RPgrad{\P_\khat}{\Q} \leq \frac{4\smooth(1 + \delta)}{K + 2} .
\end{equation}
We emphasize that at each iteration \mmode{k}, the FW-oracle solves the FW-problems upto an additive accuracy of \mmode{\delta \stepsize_\khat \smooth}. Therefore, we do not have access to the exact value of \mmode{\sup_{\Q \in \amb} \ \RPgrad{\P_k}{\Q}}, but only its approximation \mmode{\fwgap_k}. Consequently, it is only possible to verify whether the upper bound \eqref{eq:fwgap-upper-bound} for \mmode{\fwgap_{\khat}} is satisfied for some \mmode{\khat}, and impossible to verify whether \eqref{eq:fw-sup-sub-optimality} holds. Moreover, even if \eqref{eq:fwgap-upper-bound} holds for some \mmode{\khat}, it is not necessary that \eqref{eq:fw-sup-sub-optimality} also holds since \mmode{\fwgap_{\khat} \symbolspace{\leq}{\;} \sup_{\Q \in \amb}  \RPgrad{\P_\khat}{\Q}}. However, since \mmode{\fwgap_{\khat}} is approximately equal to the FW-gap, an upper bound on \mmode{\fwgap_{\khat}} gives the following slightly worse upper bound on the FW-gap
\[
\sup_{\Q \in \amb} \; \RPgrad{\P_\khat}{\Q} \leq \delta \stepsize_\khat \smooth + \RPgrad{\P_\khat}{\Q_\khat} \leq \frac{2\delta \smooth}{K + 2} + \frac{4\smooth(1+\delta)}{K + 2} \leq \frac{2\smooth(2\smooth + 3\delta)}{K + 2}  .
\]
Finally, since \mmode{\RP{\Q} - \RP{\P} \symbolspace{\leq}{\;} \RPgrad{\P}{\Q}} for every \mmode{\Q \in \amb} (follows from \eqref{eq:concavity-linear-bound}), taking supremum over \mmode{\Q \in \amb} immediately gives the inequality
\[
R\opt - \RP{\P} \leq \sup_{\Q \in \amb} \; \RPgrad{\P_\khat}{\Q} .
\]
This completes the proof of the proposition.
\end{proof}

\subsection{Proofs of Section \ref{section:NDRO} (NDRO)}

In this part we cover the technical proofs of the theoretical statements in Section \ref{section:NDRO}.

\begin{proof}[Proof of Lemma \ref{lemma:NDRO-smoothness} ] 
Since the set \mmode{\setx} is closed and the function \mmode{\FP{x}{\P}} is continuous and \mmode{\alpha}-strongly convex in \mmode{x} for every \mmode{\P \in \amb}, we conclude that the minimizer \mmode{\argmin_{x \in \setx} \ \FP{x}{\P}} exists and unique. Consequently, \mmode{x(\P) \define \argmin_{x \in \setx} \ \FP{x}{\P} } is a singleton for every \mmode{\P \in \amb}.

\noindent {\em Proof of Lemma \ref{lemma:NDRO-smoothness}(i): Danskin's theorem.} Letting \mmode{\pg \define \P + \stepsize \big( \Q - \P \big)} for \mmode{\stepsize \in [0,1]}, consider the mappings
\[
\begin{cases}
\begin{aligned}
    \setx \times [0,1] \ni (x, \stepsize) \mapsto f(x, \stepsize) \ &\define \ \FP{x}{\pg} \quad \text{and }\\
    [0,1] \ni \stepsize \mapsto g(\stepsize) \ &\define \ \min\limits_{x \in \setx} \; f(x,\stepsize) .
\end{aligned}
\end{cases}
\]
It is easily seen that \mmode{g(\stepsize) = \RP{\pg}}. Moreover, for the one sided derivatives of \mmode{g} at \mmode{\stepsize \in [0,1)} defined: \mmode{{\rm d}g(\gamma;1) \define \lim_{\delta \downarrow 0} \; \frac{1}{\delta} \big( g(\stepsize + \delta) - g(\stepsize) \big) }, we easily verify that \mmode{\RPgrad{\P}{\Q} = {\rm d}g(0;1)}. Using the short-hand notation \mmode{f_x (\cdot) \define f(x,\cdot)}, we also verify similarly that \mmode{{\rm d}f_x (0 ; 1) =  \FPgrad{x}{\P}{\Q} }, for every \mmode{x \in \setx}. For every \mmode{\stepsize \in [0,1]}, let \mmode{x_{\stepsize} \define \argmin_{x \in \setx} \ f(x,\stepsize)}, we know from the Danskin's theorem \cite[(A.22), p. 154]{bertsekas1971control} that the one sided derivatives \mmode{{\rm d}g(\gamma;1) } at \mmode{\stepsize}, are given by
\[
{\rm d}g(\gamma;1) = {\rm d}f_{x_\stepsize} (\stepsize ; 1) = \lim\limits_{\delta \downarrow 0} \frac{1}{\delta} \big( f(x_{\stepsize}, \stepsize + \delta) - f(x_{\stepsize} , \stepsize) \big) .
\]
Collecting everything, we have \mmode{\RPgrad{\P}{\Q} = {\rm d}g(0;1) = {\rm d}f_{x_0} (0;1) = \FPgrad{x(\P)}{\P}{\Q} }. 

\noindent {\em Proof of Lemma \ref{lemma:NDRO-smoothness}(ii): Smoothness.}
Recall that \mmode{\nabla_1 F(x,\P) \define \big( \nicefrac{\partial F}{\partial x} \big) (x, \P)} denotes the partial derivative of \mmode{F} w.r.t.\,\mmode{x} evaluated at \mmode{(x,\P)}. Since \mmode{x(\P) = \argmin_{x \in \setx} \FP{x}{\P}}, the first-order optimality conditions give
\begin{equation}
\label{eq:first-order-optimality-x}
\fxgrad{x(\P)}{\P}{y} \; \geq \; 0 \quad \text{ for all } y \in \setx .
\end{equation}
Due to \mmode{\alpha}-strong-convexity of \mmode{\FP{y}{\P}} in \mmode{y}, we have
\[
\begin{aligned}
\frac{\alpha}{2} \norm{x(\P) - x(\Q)}^2 \ &\leq \ \FP{x(\Q)}{\P} - \FP{x(\P)}{\P} - \fxgrad{x(\P)}{\P}{x(\Q)} \\
& \leq \ \FP{x(\Q)}{\P} - \FP{x(\P)}{\P} \quad \text{from \eqref{eq:first-order-optimality-x} with } y = x(\Q) . 
\end{aligned}
\]
Similarly, \mmode{\alpha}-strong-convexity of \mmode{\FP{\cdot}{\Q}} gives us
\[
\frac{\alpha}{2} \norm{x(\Q) - x(\P)}^2 \leq \FP{x(\P)}{\Q} - \FP{x(\Q)}{\Q} .
\]
Combining the two inequalities, we infer that the inequality
\begin{equation}
\label{eq:x-quadratic-lower-bound}
\begin{aligned}
\alpha \norm{x(\P) - x(\Q)}^2 \ 
&\leq \ \FP{x(\Q)}{\P} - \FP{x(\P)}{\P} + \FP{x(\P)}{\Q} - \FP{x(\Q)}{\Q} \\
&= \ \FP{x(\Q)}{\P} - \FP{x(\Q)}{\Q} + \FP{x(\P)}{\Q} - \FP{x(\P)}{\P} \\
&\leq \ \FPgrad{x(\Q)}{\Q}{\P} + \FPgrad{x(\P)}{\P}{\Q} && \text{from \eqref{eq:concavity-linear-bound} } \\
&= \ \RPgrad{\Q}{\P} + \RPgrad{\P}{\Q} && \text{from \eqref{eq:danskin}},
\end{aligned}
\end{equation}
holds for every \mmode{\P , \Q \in \amb}. On the one hand, for \mmode{\P_{\stepsize} = \P + \stepsize (\Q - \P)}, \mmode{\stepsize \in [0,1]}, we have
\begin{equation}
\label{eq:smoothness-ineq-1}
\begin{aligned}
& \RPgrad{\P}{\P_\stepsize} \; + \; \RPgrad{\P_\stepsize}{\P} && \\
& \quad \quad = \; \FPgrad{x(\P)}{\P}{\P_\stepsize} \; + \; \FPgrad{x(\P_\stepsize)}{\P_\stepsize}{\P} && \text{from \eqref{eq:danskin}} \\
& \quad \quad = \; \FPgrad{x(\P)}{\P}{\P_\stepsize} \; - \; \FPgrad{x(\pg)}{\P}{\P_\stepsize}\; + \; \FPgrad{x(\pg)}{\P}{\P_\stepsize} \; + \; \FPgrad{x(\P_\stepsize)}{\P_\stepsize}{\P} && \\
& \quad \quad = \; \stepsize \Big( \FPgrad{x(\P)}{\P}{\Q} \; - \; \FPgrad{x(\pg)}{\P}{\Q} \Big) \; + \; \Big( \FPgrad{x(\pg)}{\P}{\P_\stepsize} \; + \; \FPgrad{x(\P_\stepsize)}{\P_\stepsize}{\P} \Big) && \text{from \eqref{eq:positive-homogeniety}} \\
& \quad \quad \leq \; \stepsize \smooth_1 \norm{x(\P_\stepsize) - x(\P)} \; + \; \stepsize^2 \smooth_2 , &&
\end{aligned}
\end{equation}
where the last inequality is due to \eqref{eq:continuous-partial-derivatives} and the smoothness condition (iii) of Assumption \ref{assumption:NDRO-smoothness}. On the other hand, considering \mmode{\Q = \P_\stepsize} in \eqref{eq:x-quadratic-lower-bound}, we have
\begin{equation}
\label{eq:smoothness-ineq-2}
\alpha \norm{x(\P_\stepsize) - x(\P)}^2 \leq \RPgrad{\P_\stepsize}{\P} + \RPgrad{\P}{\P_\stepsize} . 
\end{equation}
Collecting \eqref{eq:smoothness-ineq-1} and \eqref{eq:smoothness-ineq-2} together, we see that
\[
\alpha \norm{x(\P) - x(\P_\stepsize)}^2 \leq \stepsize \smooth_1 \norm{x(\P) - x(\P_\stepsize)} \; + \; \stepsize^2 \smooth_2 .
\]
On rearranging and simplifying terms, it is now easily verified that
\[
\left( \norm{x(\P) - x(\P_\stepsize)} - \frac{\stepsize}{2\alpha} \Big( \smooth_1 - \sqrt{\smooth_1^2 + 4\alpha \smooth_2} \Big) \right) \left( \norm{x(\P) - x(\P_\stepsize)} - \frac{\stepsize}{2\alpha} \Big( \smooth_1 + \sqrt{\smooth_1^2 + 4\alpha \smooth_2} \Big) \right) \symbolspace{\leq}{\;} 0 ,
\]
which is only true if
\[
\frac{\stepsize}{2\alpha} \Big( \smooth_1 - \sqrt{\smooth_1^2 + 4\alpha \smooth_2} \Big) \leq \norm{x(\P) - x(\P_\stepsize)} \leq \frac{\stepsize}{2\alpha} \Big( \smooth_1 + \sqrt{\smooth_1^2 + 4\alpha \smooth_2} \Big) .
\]
The lower bound is irrelevant since it is negative. However, the upper bound is non-trivial, and employing it in \eqref{eq:smoothness-ineq-1} finally gives
\[
\begin{aligned}
\RPgrad{\P}{\P_\stepsize} + \RPgrad{\P_\stepsize}{\P} \ &\leq \ \stepsize^2 \frac{\smooth_1}{2\alpha} \Big( \smooth_1 + \sqrt{\smooth_1^2 + 4\alpha \smooth_2} \Big) \ + \stepsize^2 \smooth_2 \\
&\leq \ \stepsize^2 \left( \smooth_2 + \frac{\smooth_1}{2\alpha} \Big( \smooth_1 + \sqrt{\smooth_1^2 + 4\alpha \smooth_2} \Big) \right) .
\end{aligned}
\]
Since \mmode{ \smooth_2 + \frac{\smooth_1}{2\alpha} \Big( \smooth_1 + \sqrt{\smooth_1^2 + 4\alpha \smooth_2} \Big)  \ < \ +\infty }, we conclude that the risk measure \mmode{R} is \mmode{C}-smooth in the sense of Definition \ref{def:smoothness} for every \mmode{C \ \geq \ \smooth_2 + \frac{\smooth_1}{2\alpha} \Big( \smooth_1 + \sqrt{\smooth_1^2 + 4\alpha \smooth_2} \Big) }. The proof is now complete.
\end{proof}

\begin{proof}[Proof for Theorem \ref{theorem:NDRO-FW}]
We establish assertion (i) of the theorem assuming assertion (ii) holds which is independently proved later.

\noindent {\em Proof for Theorem \ref{theorem:NDRO-FW} (i): Strong duality.}
Assuming assertion (ii) holds, we see that
\[
\min_{x \in \setx} \; \sup_{\Q \in \amb} \ \FP{x}{\Q} \leq \sup\limits_{\Q \in \amb} \ \FP{x_{\eps}}{\Q} \leq \eps + \min_{x \in \setx} \ \FP{x}{\P_{\eps}} \leq \eps + \sup_{\Q \in \amb} \; \min_{x \in \setx} \ \FP{x}{\Q} ,
\]
holds for every \mmode{\eps > 0}, and thus, we have \mmode{\min_{x \in \setx} \; \sup\limits_{\Q \in \amb} \ \FP{x}{\Q} \leq \sup\limits_{\Q \in \amb} \; \min\limits_{x \in \setx} \ \FP{x}{\Q} }. This, together with weak duality: \mmode{\sup\limits_{\Q \in \amb} \; \min\limits_{x \in \setx} \ \FP{x}{\Q} \leq \min\limits_{x \in \setx} \; \sup\limits_{\Q \in \amb} \ \FP{x}{\Q}} proves assertion (i).

\noindent {\em Proof for Theorem \ref{theorem:NDRO-FW} (ii): Saddle point computation.}
Let us recall that 
\[
\RP{\P} = \min\limits_{x \in \setx} \; \FP{x}{\P} \quad \text{and} \quad x(\P) = \argmin\limits_{x \in \setx} \; \FP{x}{\P} .
\]
From \eqref{eq:danskin}, since \mmode{\RPgrad{\P}{\Q} = \FPgrad{x (\P)}{\P}{\Q} }, the iterates \mmode{(\P_k)_k} obtained from Algorithm \ref{algo:NDRO-min-max} can be equivalently regarded as the ones obtained from the FW-algorithm \eqref{eq:FW-iterations-main} for the maximization problem: \mmode{\sup_{\P \in \amb} \ \RP{\P}} under the setting of Proposition \ref{proposition:FW-gap}. Therefore, we conclude from Proposition \ref{proposition:FW-gap}, and more specifically from Remark \ref{remark:FW-eps-stopping}, we know that there exists a \mmode{\keps \leq \khat \leq 2\keps + 1} such that \mmode{\RPgrad{\P_\khat}{\Q_\khat} \leq \eps \frac{2 + 2\delta}{2 + 3\delta}}. 

We also conclude from Remark \ref{remark:FW-eps-stopping} that \mmode{\sup_{\Q \in \amb} \RPgrad{\P_\khat}{\Q} \leq \eps} for \emph{any} \mmode{\keps \leq \khat \leq 2\keps + 1} satisfying \mmode{\RPgrad{\P_\khat}{\Q_\khat} \leq \eps \frac{2 + 2\delta}{2 + 3\delta}}. In particular, for \mmode{(x_\eps , \P_\eps)} to be the output of Algorithm \ref{algo:NDRO-min-max}, we know that
\[
\sup\limits_{\Q \in \amb} \ \FPgrad{x_\eps}{\P_\eps}{\Q} = \sup_{\Q \in \amb} \; \RPgrad{\peps}{\Q} \leq \eps . 
\]
Consequently, for any \mmode{\Q \in \amb}, we have
\[
\FP{x_\eps}{\Q} - \FP{x_\eps}{\P_\eps} \leq \FPgrad{x_\eps}{\P_\eps}{\Q} \leq \eps ,
\]
where the first inequality follows from \eqref{eq:concavity-linear-bound} for \mmode{\FP{x_\eps}{\cdot}}. Finally, taking the supremum over \mmode{\Q \in \amb} we conclude \mmode{ \sup_{\Q \in \amb} \; \FP{\xeps}{\Q} \symbolspace{\leq}{\;} \eps + \FP{\xeps}{\peps} }; which together with the fact that \mmode{\FP{x_\eps}{\P_\eps} = \min_{x \in \setx} \; \FP{x}{\P_\eps}} implies \mmode{(\xeps , \peps)} being indeed an \mmode{\eps}-saddle point in the sense of Definition \ref{def:dual-gap-eps-solution}. The proof is now complete.
\end{proof}

\subsection{Proof of Section \ref{section:er-risk-portfolio-selection} (entropic risk)}
We first prove the results on the FW oracle (Lemma \ref{lemma:er-risk-FW-oracle}) for the entropic risk portfolio selection problem \eqref{eq:er-risk-portfolio-selection} and then prove the results on regularity conditions (Lemma \ref{lemma:er-risk-smaller-ambiguity} and \ref{lemma:er-risk-smoothness-assumptions}).

\paragraph{{\bf Proofs for the FW-oracle}}
We shall first establish two key lemmas that will be later used to prove Lemma \ref{lemma:er-risk-FW-oracle}.
\begin{lemma}
\label{lemma:er-risk-single-point-sup-problem}
Let \mmode{c > \theta > 0}, \mmode{\eta \geq 0}, \mmode{x \in [0,1]}, and \mmode{\xi \in \R{}}. Consider
\begin{equation}
\label{eq:er-risk-single-point-sup-problem}
\sup_{q \in \R{}} \ \lagrangian (\eta, q) \define e^{- \theta x (\xi + q)}  - \eta e^{c \abs{q}} .
\end{equation}
The following assertions hold:
\begin{enumerate}[label = {\rm (\roman*)}, itemsep = 0mm, topsep = 0mm, leftmargin = *]
\item If \mmode{\eta = 0}, the maximization problem \eqref{eq:er-risk-single-point-sup-problem} is unbounded.

\item If \mmode{\eta > 0}, the maximization problem admits a unique optimal solution \mmode{\qeta{\eta}}, given by
\begin{equation}
\label{eq:er-risk-single-point-optimal-solution}
\qeta{\eta} = \min \left\{ 0 , \frac{\theta x \xi + \log(\nicefrac{c \eta}{\theta x})}{ c - \theta x } \right\} .
\end{equation}
\end{enumerate}
\end{lemma}

\begin{proof}[Proof of Lemma \ref{lemma:er-risk-single-point-sup-problem}]
We emphasize that the maximization problem \eqref{eq:er-risk-single-point-sup-problem} is {\em non-convex}. Even then, we shall establish the conditions for \eqref{eq:er-risk-single-point-sup-problem} to admit an optimal solution and also explicitly characterize it. To this end, if \mmode{\eta = 0}, we see that 
\[
\lim_{q \longrightarrow -\infty} \lagrangian (0, q) =  e^{- \theta x \xi} \lim_{q \longrightarrow -\infty} e^{- \theta x q} = +\infty .
\]
Therefore, assertion (i) of the lemma follows at once.

Now, if \mmode{\eta > 0}, then we see that
\[
\begin{cases}
\begin{aligned}
\lim_{q \longrightarrow +\infty} L(\eta, q) &= \lim_{q \longrightarrow +\infty} e^{-\theta x (z + q)} - \eta e^{c q} = 0 - \eta (+\infty) && = - \infty , \\
\lim_{q \longrightarrow -\infty} L(\eta, q) &= \lim_{q \longrightarrow -\infty} e^{-c q} \left( e^{-\theta x z } e^{( c -\theta x) q} - \eta \right) = (+\infty) (0 - \eta) && = - \infty .
\end{aligned}
\end{cases}
\]
Therefore, we conclude that whenever \mmode{\eta > 0}, \eqref{eq:er-risk-single-point-sup-problem} admits an optimal solution \mmode{\qeta{\eta} \in \R{} }. Moreover, at the optimal solution \mmode{\qeta{\eta}}, we know that the necessary optimality conditions must be satisfied 
\begin{equation}
\label{eq:er-risk-single-sup-problem-necessary-condition}
0 \in \frac{\partial}{\partial q} \lagrangian (\eta, \qeta{\eta}) = -\theta x e^{- \theta x (z + \qeta{\eta}) } - \eta c e^{c \abs{\qeta{\eta}}} \sgn (\qeta{\eta}) .
\end{equation}
From \eqref{eq:er-risk-single-sup-problem-necessary-condition}, it is immediately evident that \mmode{\qeta{\eta} \leq 0}. Moreover, we also see that
\begin{itemize}
    \item If \mmode{(\nicefrac{c \eta}{\theta x}) \geq e^{-\theta x z}} - we see that 
    \[
    \frac{\partial}{\partial q} \lagrangian (\eta, q) = \theta x e^{-\theta x q} \left( (\nicefrac{c \eta}{\theta x}) e^{-( c - \theta x) q} - e^{-\theta x z} \right) \geq 0 \text{ for all \mmode{q<0}, and}
    \]
    \mmode{0 \in \frac{\partial}{\partial q} \lagrangian (\eta, 0) = \eta c [-1, +1] - \theta x e^{-\theta x z} }. Therefore, \mmode{\qeta{\eta} = 0} is the only point satisfying the necessary optimality condition \eqref{eq:er-risk-single-sup-problem-necessary-condition}, and consequently the unique optimal solution to \eqref{eq:er-risk-FW-scalar-problem}.

    \item Similarly if \mmode{(\nicefrac{c \eta}{\theta x}) < e^{-\theta x z}} - it is easily verified that for \mmode{\qeta{\eta} = \frac{\theta x z + \log (\nicefrac{\eta c}{\theta x})}{c - \theta x} }, is the unique point which satisfies the necessary first-order optimality condition \eqref{eq:er-risk-single-sup-problem-necessary-condition}, and consequently, is the unique optimal solution to \eqref{eq:er-risk-FW-scalar-problem}.
\end{itemize}
We finally note that the cases above can be compressed as \mmode{\qeta{\eta} = \min \left\{ 0 , \frac{\theta x z + \log (\nicefrac{\eta c}{\theta x})}{c - \theta x} \right\}}. The proof is now complete.
\end{proof}

\begin{lemma}
\label{lemma:er-risk-FW-scalar-problem-dual}
Let \mmode{c, \theta}, and \mmode{x} be as in Lemma \ref{lemma:er-risk-single-point-sup-problem}, and let \mmode{Z(t)}, \mmode{t = 1,2,\ldots,T} be a non-decreasing sequence of real numbers. Consider the following minimization problem
\begin{equation}
\label{eq:er-risk-FW-scalar-problem-dual}
\inf_{\eta \geq 0} \quad J(\eta) \define \eta e^{c \ambradius} + \frac{1}{T} \summ{t = 1}{T} \sup_{q(t) \in \R{}} e^{- \theta x (Z (t) + q (t))}  - \eta e^{c \abs{q(t)}} .
\end{equation}
The minimization problem \eqref{eq:er-risk-FW-scalar-problem-dual} admits a unique optimal solution \mmode{\eta\opt} given by
\begin{equation}
\label{eq:er-risk-dual-solution}
\eta\opt \define \frac{\theta x}{c} \left( \frac{1}{T e^{c\ambradius} - T'} \summ{t = T' + 1}{T} e^{ \frac{- c \theta x Z(t)}{c - \theta x} } \right)^{\frac{c - \theta x}{\theta x}} ,
\end{equation}
where \mmode{T' \in \{1, 2, \ldots, T \} } is the smallest integer such that \mmode{ (T e^{c \ambradius} - T' ) e^{\frac{- c \theta x Z(T') }{c - \theta x}}\geq \summ{t = T' + 1}{T} e^{\frac{- c \theta x Z(t) }{c - \theta x}} }.
\end{lemma}

\begin{proof}[Proof of Lemma \ref{lemma:er-risk-FW-scalar-problem-dual}]
Since \mmode{\eta \longmapsto \log(\nicefrac{c\eta}{\theta x})} is monotonically increasing and eventually positive, we observe from \eqref{eq:er-risk-single-point-optimal-solution} that \mmode{\lim\limits_{\eta \rightarrow +\infty} \qeta{\eta} = 0}. Consequently, we see that
\[
\lim_{\eta \rightarrow +\infty} J(\eta) = \lim_{\eta \rightarrow + \infty} \eta (e^{c \ambradius} - 1) + \frac{1}{T} \summ{t = 1}{T} e^{- \theta x z(t)} = + \infty .
\]
Similarly, as \mmode{\eta \downarrow 0}, we conclude from assertion (i) of Lemma \ref{lemma:er-risk-single-point-sup-problem} that \mmode{\qeta{\eta} < 0}, consequently, we see
\[
\lim_{\eta \downarrow 0} J(\eta) = \lim_{\eta \downarrow 0} \ \eta e^{c \ambradius} + \eta^{\frac{-\theta x}{c - \theta x}} \left( \frac{c}{\theta x} - 1 \right) (\nicefrac{\theta x}{c})^{\frac{c}{c - \theta x}}  \frac{1}{T} \summ{t = 1}{T} e^{\frac{- c \theta x Z(t)}{c - \theta x}} = +\infty .
\]
Since the mapping \mmode{\eta \longmapsto J(\eta)} is convex, we know that there exists some \mmode{\eta\opt > 0} such that \mmode{\eta\opt = \argmin\limits_{\eta \geq 0} J(\eta) }, i.e., an optimal solution to \eqref{eq:er-risk-FW-scalar-problem-dual} exists. To characterise an optimal solution \mmode{\eta\opt}, we observe that \mmode{\qeta{\eta}} is unique for any \mmode{\eta > 0}, and from Danskin's theorem, we conclude that \mmode{\frac{d}{d \eta} J(\eta) = e^{c \ambradius} - \frac{1}{T} \summ{t = 1}{T} e^{c \abs{q_{\eta} (t)}} }. Therefore, the optimal solution \mmode{\eta\opt} is such that \mmode{ T e^{c \ambradius} = \summ{t = 1}{T} e^{c \abs{q_{\eta\opt} (t)}} }. Finding such an \mmode{\eta\opt} where the equality holds is not straightforward.

For \mmode{\eta(t) = \big(\nicefrac{\theta x}{c} \big) e^{- \theta x Z(t)} }, \mmode{t = 1,2,\ldots,T}, it is easily verified that \mmode{\qeta{\eta_t} (s) = 0} for all \mmode{s \leq t}. Thus, 
\[
\frac{d}{d \eta} J(\eta_t) = e^{c \ambradius} - \frac{1}{T} \left( t + \summ{s = t+1}{T} e^{ \frac{c \theta x (Z(t) - Z(s))}{c - \theta x} } \right) .
\]
Let \mmode{T' \in \{1,2,\ldots,T\}} be the smallest integer such that \mmode{\frac{d}{d \eta} J(\eta (t)) \geq 0 }. Since \mmode{J(\eta)} is convex, \mmode{\frac{d}{d \eta} J(\eta)} is non-decreasing. Consequently, with \mmode{\eta_0 = 0}, we know that \mmode{\eta\opt \in \ ]\eta_{T' - 1} , \eta_{T'}]}, and therefore, \mmode{\qeta{\eta\opt} (t) = 0 } for all \mmode{t \leq T'}, and \mmode{\qeta{\eta\opt} (t) = \frac{\theta x Z(t) + \log(\nicefrac{c \eta\opt}{\theta x})}{ c - \theta x }  } for \mmode{t = T'+1 , \ldots, T}. Substituting, these values of \mmode{\qeta{\eta\opt} (t)} in the equation \mmode{ T e^{c \ambradius} = \summ{t = 1}{T} e^{c \abs{q_{\eta\opt} (t)}} } and simplifying for \mmode{\eta\opt} gives \eqref{eq:er-risk-dual-solution}. The proof of the lemma is complete.
\end{proof}

\begin{proof}[Proof of Lemma \ref{lemma:er-risk-FW-oracle}]
The dual problem of \eqref{eq:er-risk-FW-scalar-problem} is given by
\begin{equation}
\label{eq:-er-risk-FW-scalar-problem-dual-1}
\inf_{\eta \geq 0} \quad J(\eta) = \eta e^{c \ambradius} + \frac{1}{T} \summ{t = 1}{T} \sup_{q(t) \in \R{}} e^{- \theta x (z (t) + q (t))}  - \eta e^{c \abs{q(t)}} ,
\end{equation}

\noindent {\em Proof of Lemma \ref{lemma:er-risk-FW-oracle} (i): Optimal solution.}
Since the maximization over \mmode{q(t)} is separable over \mmode{t} we see that 
\[
J(\eta) = \eta e^{c \ambradius} + \frac{1}{T} \summ{s = 1}{T} \sup_{q'(s) \in \R{}} e^{- \theta x (Z (s) + q' (s))}  - \eta e^{c \abs{q'(s)}} ,
\]
where the sequence \mmode{(Z(s))_s} is non-increasing. Consequently, we conclude from Lemma \ref{lemma:er-risk-FW-scalar-problem-dual} that \eqref{eq:-er-risk-FW-scalar-problem-dual-1} admits a unique optimal solution \mmode{\eta\opt} given by \eqref{eq:er-risk-dual-solution}. Given \mmode{\eta\opt}, the optimal solution \mmode{\qeta{\eta\opt} (t)} for each \mmode{t = 1,2, \ldots, T} is obtained from \eqref{eq:er-risk-single-point-optimal-solution}. Substituting for \mmode{\qeta{\eta\opt} (t)} from \eqref{eq:er-risk-single-point-optimal-solution} and simplifying, we easily verify that \mmode{z\opt (t) = z(t) + \qeta{\eta\opt} (t)}. Consequently, \mmode{\Q\opt = \frac{1}{T} \summ{t = 1}{T} \delta (z\opt (t)) } is the optimal solution to \eqref{eq:er-risk-FW-scalar-problem}. This proves assertion (i) of the lemma.

\noindent {\em Proof of Lemma \ref{lemma:er-risk-FW-oracle} (ii): Lower and upper bounds.}
To prove assertion (ii) of the lemma, we first see that since \mmode{\eta\opt > 0}, the first-order optimality conditions imply \mmode{ 0 = \frac{d}{d \eta} J(\eta\opt) = e^{c \ambradius} - \frac{1}{T} \summ{t = 1}{T} e^{c \abs{\qeta{\eta\opt} (t)}} }. Therefore, for any \mmode{s = 1,2,\ldots,T}, we have
\[
e^{c \abs{\qeta{\eta\opt} (s)}} \leq \summ{t = 1}{T} e^{c \abs{\qeta{\eta\opt} (t)}} = T e^{c \ambradius} .
\]
Taking \mmode{\log_e(\cdot)} on both sides of the inequality, we obtain \mmode{ \abs{ \qeta{\eta\opt} (s) } \leq \ambradius + \frac{\log (T)}{c} }. Since \mmode{\qeta{\eta\opt} (s) \leq 0}, we have \mmode{-\ambradius - \frac{\log(T)}{c} \leq \qeta{\eta\opt} (s) \leq 0 } for all \mmode{s = 1,\ldots, T}. Combining this with \mmode{z\opt (s) = z(s) + \qeta{\eta} (s) } and \mmode{ \underline{z} \leq z(s) \leq \overline{z} }, assertion (ii) follows.
\end{proof}

\paragraph{{\bf Proofs for regularity of the entropic risk}}
\begin{proof}[Proof of Lemma \ref{lemma:er-risk-smaller-ambiguity}]
For any \mmode{x \in \simplex{n}}, we shall first show that
\begin{equation}
\label{eq:er-risk-arg-max-equality}
\argmax_{\P \in \amb_c} \errisk{x}{\P} = \argmax_{\P' \in \amb_c'} \errisk{x}{\P'} .
\end{equation}
To show \mmode{ \argmax_{\P \in \amb_c} \errisk{x}{\P} \subset \argmax_{\P' \in \amb_c'} \errisk{x}{\P'} }, let \mmode{\P_x \in \argmax_{\P \in \amb_c} \errisk{x}{\P} }, due to the first-order optimality conditions together with \eqref{eq:er-risk-FW-arg-max-equality}, we see that
\[
\P_x = \argmax_{\Q \in \amb_c} \erxpgrad{x}{\P_x}{\Q} = \argmax_{\Q' \in \amb_c'} \erxpgrad{x}{\P_x}{\Q'}, \text{ and thus, } \P_x \in \argmax_{\P' \in \amb_c'} \errisk{x}{\P'} .
\]
Similarly, to show that \mmode{ \argmax_{\P' \in \amb_c'} \errisk{x}{\P'} \subset \argmax_{\P \in \amb_c} \errisk{x}{\P} }, let \mmode{\P_x' \in \argmax_{\P' \in \amb_c'} \errisk{x}{\P'} }, then from again the first-order optimality conditions together with \eqref{eq:er-risk-FW-arg-max-equality}, we have
\[
\P_x' = \argmax_{\Q' \in \amb_c'} \erxpgrad{x}{\P_x'}{\Q'} = \argmax_{\Q \in \amb_c} \erxpgrad{x}{\P_x'}{\Q}, \text{ and thus, } \P_x' \in \argmax_{\P \in \amb_c} \errisk{x}{\P} .
\]

Now, for \mmode{(x\opt , \P\opt)} to be a saddle-point of \eqref{eq:er-risk-portfolio-selection}, we have
\[
x\opt \in \argmin_{x \in \setx} \errisk{x}{\P\opt} \quad \text{ and } \quad \P\opt \in \argmax_{\P \in \amb_c} \errisk{x\opt}{\P} = \argmax_{\P' \in \amb_c'} \errisk{x\opt}{\P'} ,
\]
where the last equality is due to \eqref{eq:er-risk-arg-max-equality}. Thus, \mmode{(x\opt , \P\opt)} is also a saddle point of \eqref{eq:er-risk-problem-smaller-ambiguity}, and vice versa. The proof is now complete.
\end{proof}

\begin{proof}[Proof of Lemma \ref{lemma:er-risk-smoothness-assumptions}]
For every \mmode{j = 1,2,\ldots,n}, since \mmode{ \ximin{j} \leq \xi_j \leq \ximax{j} }, \mmode{\P_j}-almost surely for all \mmode{\P_j \in \erwamb{j}}, we have
\begin{equation}
\label{eq:er-risk-exp-bounds}
e^{-\theta_j x_j \ximax{j} } \leq \EE{\P_j} [e^{- \theta_j x_j \xi_j}] \leq e^{-\theta_j x_j \ximin{j} } \quad \text{ for all } \P_j \in \erwamb{j} .    
\end{equation}

\noindent {\em Proof of Lemma \ref{lemma:er-risk-smoothness-assumptions}(i): Continuous derivatives.}
Let \mmode{\P, \Q \in \amb_c'}, and \mmode{x, y \in \simplex{n}}, we have
\[
\begin{aligned}
\erxpgrad{x}{\P}{\Q} - \erxpgrad{y}{\P}{\Q} 
&= \summ{j = 1}{n} \frac{1}{\theta_j} \left( \frac{ \EE{\Q_j} [e^{- \theta_j x_j \xi_j}] }{ \EE{\P_j} [e^{- \theta_j x_j \xi'_j}] } - \frac{\EE{\Q_j} [e^{- \theta_j y_j \xi_j}]}{ \EE{\P_j} [e^{- \theta_j y_j \xi'_j}] } \right) \\
&= \summ{j = 1}{n} \frac{1}{\theta_j} \left( \frac{ \EE{\Q_j} [e^{- \theta_j x_j \xi_j}] \cdot \EE{\P_j} [e^{- \theta_j y_j \xi'_j}] \ - \ \EE{\Q_j} [e^{- \theta_j y_j \xi_j}] \cdot \EE{\P_j} [e^{- \theta_j x_j \xi'_j}]  }{ \EE{\P_j} [e^{- \theta_j x_j \xi'_j}] \ \cdot \ \EE{\P_j} [e^{- \theta_j y_j \xi'_j}] } \right) \\
&= \summ{j = 1}{n} \frac{1}{\theta_j} \left( \frac{ \EE{\Q_j \times \P_j} \left[ e^{- \theta_j (x_j \xi_j + y_j \xi'_j)} - e^{-\theta_j (y_j \xi_j + x_j \xi'_j)} \right] }{ \EE{\P_j} [e^{- \theta_j x_j \xi'_j}] \ \cdot \ \EE{\P_j} [e^{- \theta_j y_j \xi'_j}] } \right)
\end{aligned}
\]
Due to convexity of \mmode{(\cdot) \longmapsto e^{(\cdot)}}, for any \mmode{a, b \in \R{}}, we have \mmode{e^b - e^a \leq \abs{b - a} e^{\max \{a, b\}} }, which gives
\[
\begin{aligned}
e^{- \theta_j (x_j \xi_j + y_j \xi'_j)} & - e^{-\theta_j (y_j \xi_j + x_j \xi'_j)} \\
&\leq \theta_j \abs{ x_j \xi_j + y_j \xi'_j - x_j \xi'_j - y_j \xi_j } e^{ -\theta_j \min \{ x_j \xi_j + y_j \xi'_j \ , \  x_j \xi'_j + y_j \xi_j \} } \\
& \leq \theta_j \abs{(x_j - y_j)(\xi_j - \xi'_j) } e^{ -\theta_j (2 \ximin{j}) } \quad \text{since \mmode{x_j, y_j \in [0,1]} and \mmode{\xi_j , \xi'_j \in [\ximin{j} , \ximax{j}]} } \\
& \leq \theta_j \abs{x_j - y_j} (\ximax{j} - \ximin{j}) e^{ -\theta_j (2 \ximin{j}) } .
\end{aligned}
\]
This together with \eqref{eq:er-risk-exp-bounds} finally gives
\[
\begin{aligned}
\erxpgrad{x}{\P}{\Q} - \erxpgrad{y}{\P}{\Q}
&\leq \summ{j = 1}{n} \frac{1}{\theta_j} \frac{ \theta_j \abs{x_j - y_j} (\ximax{j} - \ximin{j}) e^{ -\theta_j (2 \ximin{j})} }{ e^{- 2 \theta_j \ximax{j}} } = \summ{j = 1}{n} \abs{x_j - y_j} (\ximax{j} - \ximin{j}) e^{ 2\theta_j (\ximax{j} - \ximin{j})} \\
& \leq \pnorm{x - y}{2} \sqrt{ \summ{j = 1}{n} (\ximax{j} - \ximin{j})^2 e^{ 4\theta_j (\ximax{j} - \ximin{j})} } ,
\end{aligned}
\]
where the last inequality is due to Cauchy-Schwartz. This establishes assertion (i) of the lemma.

\noindent {\em Proof of Lemma \ref{lemma:er-risk-smoothness-assumptions}(ii): Smoothness.}
For every \mmode{x \in \simplex{n}}, and \mmode{\P , \Q \in \amb_c}, we recall that
\[
\erxpgrad{x}{\P}{\Q} = \summ{j = 1}{n} \frac{1}{\theta_j} \frac{ \EE{\Q_j - \P_j} [e^{-\theta_j x_j \xi_j}] }{ \EE{\P_j} [e^{-\theta_j x_j \xi_j}] } .
\]
Then for any \mmode{\stepsize \in [0,1]} and \mmode{\pg = \P + \stepsize (\Q - \P)}, we have from \eqref{eq:positive-homogeniety} that \mmode{\erxpgrad{x}{\P}{\pg} = \stepsize \erxpgrad{x}{\P}{\Q} }. Moreover, using the relations \mmode{ \EE{\P - \pg} [\cdot] = -\stepsize \EE{\Q - \P} [\cdot] }, we also verify that
\[
\erxpgrad{x}{\pg}{\P} = \summ{j = 1}{n} \frac{1}{\theta_j} \frac{ - \stepsize \EE{\Q_j - \P_j} [e^{-\theta_j x_j \xi_j}] }{ \EE{ \P_j + \stepsize (\Q_j - \P_j) } [e^{-\theta_j x_j \xi_j}] } .
\]
Thus, we obtain
\[
\begin{aligned}
\erxpgrad{x}{\P}{\pg} + \erxpgrad{x}{\pg}{\P} 
& = \stepsize \summ{j = 1}{n} \frac{ 1 }{\theta_j} \left( \frac{ \EE{\Q_j - \P_j} [e^{-\theta_j x_j \xi_j}] }{ \EE{\P_j} [e^{-\theta_j x_j \xi_j}] } - \frac{ \EE{\Q_j - \P_j} [e^{-\theta_j x_j \xi_j}] }{ \EE{ \P_j + \stepsize (\Q_j - \P_j) } [e^{-\theta_j x_j \xi_j}] } \right) \\
& = \stepsize \summ{j = 1}{n} \frac{ \EE{\Q_j - \P_j} [e^{-\theta_j x_j \xi_j}] }{\theta_j} \left( \frac{ \EE{ \P_j + \stepsize (\Q_j - \P_j) } [e^{-\theta_j x_j \xi_j}] \ - \ \EE{\P_j} [e^{-\theta_j x_j \xi_j}] }{ \EE{\P_j} [e^{-\theta_j x_j \xi_j}] \ \cdot \ \EE{ \P_j + \stepsize (\Q_j - \P_j) } [e^{-\theta_j x_j \xi_j}] } \right) \\
& = \stepsize^2 \summ{j = 1}{n} \frac{1}{\theta_j} \frac{ \left(  \EE{\Q_j - \P_j} [e^{-\theta_j x_j \xi_j}]  \right)^2 }{ \EE{\P_j} [e^{-\theta_j x_j \xi_j}] \ \cdot \ \EE{ \P_j + \stepsize (\Q_j - \P_j) } [e^{-\theta_j x_j \xi_j}] }
\end{aligned}
\]
Since \mmode{\P_j , \P_j + \stepsize (\Q_j - \P_j) \in \erwamb{j} }, employing \eqref{eq:er-risk-exp-bounds} gives 
\[
\begin{cases}
\begin{aligned}
\left( \EE{\P_j} [e^{- \theta_j x_j \xi_j}] \right)^{-1} & \leq e^{\theta_j x_j \ximax{j}} \\
\left( \EE{ \P_j + \stepsize (\Q_j - \P_j) } [e^{-\theta_j x_j \xi_j}] \right)^{-1} &\leq e^{\theta_j x_j \ximax{j}} , \quad \text{and} \\
\EE{\Q_j - \P_j} [e^{- \theta_j x_j \xi_j}] & \leq e^{-\theta_j x_j \ximin{j}} - e^{-\theta_j x_j \ximax{j}} .
\end{aligned}
\end{cases}
\]
Putting things together, we finally obtain
\[
\erxpgrad{x}{\P}{\pg} + \erxpgrad{x}{\pg}{\P} \leq \stepsize^2 \summ{j = 1}{n} \frac{1}{\theta_j} e^{2 \theta_j x_j \ximax{j}} \big( e^{-\theta_j x_j \ximin{j} } - e^{-\theta_j x_j \ximax{j}} \big)^2 = \stepsize^2 \summ{j = 1 }{n} \frac{1}{\theta_j} \left( e^{\theta_j (\ximax{j} - \ximin{j})} - 1 \right)^2 .
\]
Thus, assertion (ii) of the lemma follows and the proof is now complete.
\end{proof}

\subsection{Proofs of Section \ref{section:min-variance-portfolio-selection} (variance risk)}

In this part, we cover the technical proofs of the theoretical statements in Section \ref{section:min-variance-portfolio-selection}. The first proof is concerned with the regularity of the min-variance portfolio selection problem. 

\paragraph{{\bf Proofs for regularity conditions (Variance)}}
\begin{proof}[Proof of Lemma \ref{lemma:variance-smoothness-assumptions}]
For any \mmode{x \in \setx}, since \mmode{V_x} is an RR measure, we simplify \eqref{eq:variance-derivative} to get
\begin{subequations}
\begin{align}
\vxpgrad{x}{\P}{\Q} \ 
& = \ x\transp \big( \Sigma_{\Q} - \Sigma_{\P} \big) x \; - \; 2 \big( x\transp\mu_{\P} \big) \big( x\transp (\mu_{\Q} - \mu_{\P}) \big) \label{eq:dvx-formula-1}\\
& = \  x\transp (\Sigma_{\Q} - \Sigma_{\P}) \cdot x \; + \; (x\transp \mu_{\P})^2 \; + \; (x\transp (\mu_{\Q} - \mu_{\P}))^2 \; - \; (x\transp \mu_{\Q})^2 \label{eq:dvx-formula-2} .
\end{align}
\end{subequations}

\noindent {\em Proof of Lemma~\ref{lemma:variance-smoothness-assumptions}(i): Continuous derivatives.}
Firstly, we begin by showing that the inequality
\begin{equation}
\label{eq:xtrnspMx-inequality}
x\transp M x - y\transp M y \leq 2 \pnorm{M}{o} \pnorm{x - y}{2} \quad \text{holds for any \mmode{x, y \in \simplex{n}} and \mmode{M \in \symmetric{n}}} ,
\end{equation}
where, \mmode{\pnorm{M}{o} \define \max_{\pnorm{v}{2} = 1} \ \abs{ \inprod{v}{Mv} } \; }, is the operator norm. This follows since
\[
\begin{aligned}
x\transp M x - y\transp M y \ & = \ (x + y)\transp M (x - y) && \\
& \leq \pnorm{x-y}{2} \pnorm{M(x+y)}{o} && \text{from the Cauchy-Schwartz inequality} \\
& \leq \pnorm{x-y}{2} \pnorm{x+y}{2} \pnorm{M}{o} && \text{from the definition of \mmode{\pnorm{\cdot}{o}}} \\
& \leq 2 \pnorm{M}{o} \pnorm{x-y}{2} && \text{since \mmode{\pnorm{z}{2} \leq \pnorm{z}{1} = 1 } for all \mmode{z \in \simplex{n}}}.
\end{aligned}
\]
Now, for any \mmode{x , y \in \simplex{n}}, we conclude from \eqref{eq:dvx-formula-2} that
\[
\begin{aligned}
\vxpgrad{x}{\P}{\Q} - \vxpgrad{y}{\P}{\Q} \ \quad
&= \quad x\transp (\Sigma_{\Q} - \Sigma_{\P}) x \ - \ y\transp (\Sigma_{\Q} - \Sigma_{\P}) y \\
& \quad \quad \quad + \; (x\transp \mu_{\P})^2 \; + \; (x\transp (\mu_{\Q} - \mu_{\P}))^2 \; - \; (x\transp \mu_{\Q})^2 \\
& \quad \quad \quad \quad \quad   - \; (y\transp \mu_{\P})^2 \; - \; (y\transp (\mu_{\Q} - \mu_{\P}))^2 \; + \; (y\transp \mu_{\Q})^2 .
\end{aligned}
\]
Employing \eqref{eq:xtrnspMx-inequality} for \mmode{M = (\Sigma_{\Q} - \Sigma_{\P})},  \mmode{\mu_{\P}\mu_{\P}\transp}, \mmode{\mu_{\Q}\mu_{\Q}\transp}, and \mmode{(\mu_{\P} - \mu_{\Q})(\mu_{\P} - \mu_{\Q})\transp}; we have the inequalities
\begin{equation}
\begin{cases}
\begin{aligned}
x\transp (\Sigma_{\Q} - \Sigma_{\P}) x \; - \; y\transp (\Sigma_{\Q} - \Sigma_{\P}) y \ &\leq \ 2 \pnorm{\Sigma_{\Q} - \Sigma_{\P}}{o} \pnorm{x- y}{2} \\
(x\transp \mu_{\P})^2 - (y\transp \mu_{\P})^2 \ & \leq \ 2 \pnorm{\mu_{\P}}{2}^2 \pnorm{x- y}{2} \\
(y\transp \mu_{\Q})^2 - (y\transp \mu_{\Q})^2 \ & \leq \ 2 \pnorm{\mu_{\Q}}{2}^2 \pnorm{x- y}{2} \\
(x\transp (\mu_{\Q} - \mu_{\P}))^2 - (y\transp (\mu_{\Q} - \mu_{\P}))^2 \ & \leq \ 2 \pnorm{\mu_{\Q} - \mu_{\P}}{2}^2 \pnorm{x- y}{2},
\end{aligned}
\end{cases}
\end{equation}
where we have used the fact that \mmode{\pnorm{vv\transp}{o} = \pnorm{v}{2}^2} for any \mmode{v}. These inequalities give us the upper bound
\[
\begin{aligned}
\vxpgrad{x}{\P}{\Q} - \vxpgrad{y}{\P}{\Q} \ & \leq \ 2 \Big( \pnorm{\Sigma_{\Q} - \Sigma_{\P}}{o} + \pnorm{\mu_{\P}}{2}^2 + \pnorm{\mu_{\Q}}{2}^2 + \pnorm{\mu_{\Q} - \mu_{\P}}{2}^2 \Big) \pnorm{x - y}{2}
\end{aligned}
\]
Since \mmode{ \pnorm{\mu_{\P}}{2} \symbolspace{\leq}{\;} \pnorm{\mu_{\P} - \muref}{2} + \pnorm{\muref}{2} }, and \mmode{\P, \Pref \in \amb}, we conclude from \eqref{eq:bounded-mean-variance} that \mmode{ \pnorm{\mu_{\P}}{2} \symbolspace{\leq}{\;} B_{\mu} + \pnorm{\muref}{2}}, (and similarly for \mmode{\mu_{\Q}}). This together with the condition \mmode{\pnorm{\Sigma_{\Q} - \Sigma_{\P}}{o} \leq B_{\Sigma}} from \eqref{eq:bounded-mean-variance}, finally gives
\[
\vxpgrad{x}{\P}{\Q} - \vxpgrad{y}{\P}{\Q} \  \leq \ 2 \Big( B_{\Sigma} + 2 \big( B_{\mu} + \pnorm{\muref}{2} \big)^2 + B_{\mu}^2 \Big) \pnorm{x - y}{2}
\]
The constant \mmode{\smooth_1} in the assertion (i) of the Lemma is immediately picked as \mmode{C_1 = 2 \Big( B_{\Sigma} + 2 \big( B_{\mu} + \pnorm{\muref}{2} \big)^2 + B_{\mu}^2 \Big)}.

\noindent {\em Proof of Lemma~\ref{lemma:variance-smoothness-assumptions}(ii): Smoothness.}
Consider any \mmode{x \in \setx}, \mmode{\P , \Q \in \wamb{m}}, and let \mmode{\P_{\stepsize} \define \P + \stepsize (\Q - \P)} for \mmode{\stepsize \in [0,1]}. Using \mmode{\Sigma_{\P_{\stepsize}} = \Sigma_{\P} + \stepsize \big( \Sigma_{\Q} - \Sigma_{\P} \big)} and \mmode{ \mu_{\P_{\stepsize}} = \mu_{\P} + \stepsize \big( \mu_{\Q} - \mu_{\P} \big) }, we conclude from \eqref{eq:dvx-formula-1}, that
\[
\begin{aligned}
\vxpgrad{x}{\P}{\pg} 
& = \, x\transp ( \Sigma_{\P_{\stepsize}} - \Sigma_{\P} ) x \; + \; 2 \big( x\transp \mu_{\P} \big) \big( x\transp (\mu_{\P_{\stepsize}} - \mu_{\P}) \big) \\
& = \; \stepsize \Big( x\transp (\Sigma_{\Q} - \Sigma_{\P}) x \; - \; 2 \big( x\transp \mu_{\P} \big) \big( x\transp ( \mu_{\Q} - \mu_{\P} ) \big)  \Big) .
\end{aligned}
\]
Similarly, we also have  
\[
\begin{aligned}
\vxpgrad{x}{\pg}{\P}
& = \, x\transp ( \Sigma_{\P} - \Sigma_{\P_{\stepsize}} ) x \; + \; 2 \big( x\transp \mu_{\P_{\stepsize}} \big) \big( x\transp (\mu_{\P} - \mu_{\P_{\stepsize}}) \big) \\
&= 2 \stepsize^2 \big( x\transp (\mu_{\Q} - \mu_{\P}) \big)^2 + \; \stepsize \Big( 2 \big( x\transp (\mu_{\Q} - \mu_{\P}) \big) \big( x\transp \mu_{\P} \big) \; - \; x\transp ( \Sigma_{\Q} - \Sigma_{\P} )x \Big) .
\end{aligned}
\]
Combining the two equalities, we get
\[
\begin{aligned}
\vxpgrad{x}{\P}{\pg} \, + \, \vxpgrad{x}{\pg}{\P} = 2 \stepsize^2 \big( x\transp (\mu_{\Q} - \mu_{\P}) \big)^2 \leq 2 \stepsize^2 \pnorm{x}{2}^2 \pnorm{\mu_{\Q} - \mu_{\P}}{2}^2 \leq \stepsize^2 \big( 2 B_{\mu}^2 \big).
\end{aligned}
\]
The last inequality follows from \eqref{eq:bounded-mean-variance} and the fact that \mmode{\pnorm{x}{2} \leq \pnorm{x}{1} = 1}, for every \mmode{x \in \simplex{n}}. Thus, the risk measure \mmode{V_x} is \mmode{ \big( 2 B_{\mu}^2 \big)}-smooth, uniformly over \mmode{x \in \setx}. This proves assertion (ii) of the Lemma and the proof is now complete.
\end{proof}

\paragraph{{\bf Proofs for the case of unconstrained support, \mmode{(\Xi = \R{n})}.}}
\begin{proof}[Proof of Lemma \ref{lemma:min-variance-perturbation-solution}]
Substituting \mmode{q = s \bar{q}} for \mmode{s \geq 0} and \mmode{ \norm{\bar{q}} = 1 } in \eqref{eq:min-var-perturbation-problem}, it is equivalently reformulated as
\begin{equation}
\label{eq:ldro-perturbation-dummy-1}
\begin{cases}
\begin{aligned}
& \max_{s \geq 0, \, \bar{q} \in \R{n}} && \big( s (x\transp \bar{q}) \, + \, x\transp (\xi - v) \big)^2 - \ \eta s^m \\
& \sbjto && \norm{\bar{q}} = 1 .
\end{aligned}
\end{cases}
\end{equation}
Observe that \mmode{ \big( s(x\transp \bar{q}) \, + \, x\transp (\xi - v) \big)^2 \leq \, \big( s \left| x\transp \bar{q} \right| \, + \abs{x\transp (\xi - v)} \big)^2 }, where, equality holds if and only if \mmode{ \sgn (x\transp \bar{q}) = \sgn (x\transp (\xi - v))}. Moreover, applying the Holder's inequality \mmode{ \abs{x\transp \bar{q}} \leq \norm{\bar{q}} \pnorm{x}{*} = \pnorm{x}{*} } yields
\begin{equation}
\label{eq:ldro-perturbation-dummy-2}
\big( s (x\transp \bar{q}) \, + \, x\transp (\xi - v) \big)^2 \leq \, \big( s \pnorm{x}{*} \, + \abs{x\transp (\xi - v)} \big)^2 . 
\end{equation}
The upper bound~\eqref{eq:ldro-perturbation-dummy-2} is achieved if and only if \mmode{\bar{q} = \sgn (x\transp (\xi - v)) \bar{q}_x }, for any \mmode{\bar{q}_x \in \argmax_{\{\norm{\bar{q}} \leq 1\}} \; x\transp \bar{q}}. Thus, every such \mmode{\bar{q}} is an optimal solution to \eqref{eq:ldro-perturbation-dummy-1}, irrespective of \mmode{s}. Simplifying the optimization over \mmode{\bar{q}} in \eqref{eq:ldro-perturbation-dummy-1}, the problem reduces to
\begin{equation}
\label{eq:ldro-perturbation-dummy-3}
 \max\limits_{s \geq 0} \quad \big( s \pnorm{x}{*} \, + \abs{x\transp (\xi - v)} \big)^2 - \ \eta s^m .
\end{equation}
Now, an optimal solution to \eqref{eq:min-var-perturbation-problem} exists if and only if \eqref{eq:ldro-perturbation-dummy-3} admits an optimal solution \mmode{s\opt}; in which case, we have \mmode{q(\eta, \xi) = s\opt \sgn (x\transp (\xi - v)) \bar{q}_x }.

{\em Optimal solution to \eqref{eq:ldro-perturbation-dummy-3} under different settings.} If \mmode{m<2}, it is straightforward to see that the optimal value of \eqref{eq:ldro-perturbation-dummy-3} is unbounded irrespective of \mmode{\eta}. Consequently, no optimal solution exists for the maximization problem \eqref{eq:min-var-perturbation-problem} in this setting. On the contrary, if \mmode{m > 2}, it is also easily seen that the objective function of \eqref{eq:ldro-perturbation-dummy-3} is coercive if and only if \mmode{\eta > 0}. Therefore, the maximal value of \eqref{eq:ldro-perturbation-dummy-3} (and Consequently \eqref{eq:min-var-perturbation-problem}), is bounded and achieved.

For \mmode{m = 2}, the objective function of \eqref{eq:ldro-perturbation-dummy-3} can be simplified to
\[
-s^2 \big( \eta - \pnorm{x}{*}^2 \big) \, + 2s \pnorm{x}{*} \abs{ x\transp (\xi - v) } \, + \abs{ x\transp (\xi - v) }^2 .
\]
It is clear that the optimal value of \eqref{eq:ldro-perturbation-dummy-3} is unbounded if \mmode{\eta < \pnorm{x}{*}^2}. Whereas, if \mmode{\eta > \pnorm{x}{*}^2}, it is bounded, in which case, equating the derivative w.r.t.\,\mmode{s} equal to \mmode{0}, gives that the optimal value of \eqref{eq:ldro-perturbation-dummy-3} (and consequently, also \mmode{J(\eta, \xi)}) is equal to \mmode{\frac{\eta}{\eta - \pnorm{x}{*}^2} \abs{x\transp (\xi - v)}^2}, which is achieved at \mmode{ s\opt = \frac{\pnorm{x}{*} \abs{x\transp (\xi - v)} }{\eta - \pnorm{x}{*}^2}}. Finally, if \mmode{\eta = \pnorm{x}{*}^2}, the optimal value of \eqref{eq:ldro-perturbation-dummy-3} is unbounded if \mmode{x\transp (\xi_i - v) \neq 0}; otherwise, it is bounded and equal to \mmode{0} which is achieved for any \mmode{s \geq 0}.
\end{proof}

\begin{proof}[Proof of Lemma \ref{lemma:min-var-ldro-solution-for-m2}]
On the one hand, if \mmode{\abs{ x\transp (\xi_i - v) } = 0} for all \mmode{i = 1,2,\ldots,N}, we now see that the dual problem \eqref{eq:min-variance-dual} reduces to
\[
\inf_{\eta \, \geq \, \pnorm{x}{*}^2} \quad \eta \rho^2 ,
\]
which admits the optimal solution \mmode{\eta\opt = \pnorm{x}{*}^2} with an optimal value \mmode{\pnorm{x}{*}^2 \rho^2}. On the other hand, if \mmode{\abs{ x\transp (\xi_i - v) } > 0} for at least some \mmode{i \in \{ 1,2,\ldots,N\}}, the dual problem \eqref{eq:min-variance-dual} reduces to
\begin{equation}
\label{eq:min-variance-dual-1}
\inf_{\eta \, > \, \pnorm{x}{*}^2} \quad \eta \rho^2 + \frac{\eta}{\eta - \pnorm{x}{*}^2} \left( \frac{1}{N} \summ{i = 1}{N} \abs{x\transp (\xi_i - v)}^2 \right) .
\end{equation}
It is easily verified that the first-order optimality conditions for \eqref{eq:min-variance-dual-1} are satisfied at
\[
\eta_x \define \pnorm{x}{*}^2 + \frac{\pnorm{x}{*}}{\rho} \sqrt{\frac{1}{N} \summ{i = 1}{N} \abs{x\transp (\xi_i - v)}^2 } .
\]
Thus, \mmode{\eta_x} is the unique optimal solution to \eqref{eq:min-variance-dual-1}, and consequently, to the dual problem \eqref{eq:min-variance-dual}. Moreover, due to strong duality of \eqref{eq:min-variance-dual}, we also know that \mmode{\Q_x (\xi) = \frac{1}{N} \summ{i = 1}{N} \delta \big( \xi - (\xi_i + q'_i) \big) } is an optimal solution to the linear worst case distribution problem \eqref{eq:min-variance-ldro-general} for any collection \mmode{q'_i \in q(\etax{}, \xi_i)}, \mmode{i = 1,2,\ldots,N}.
\end{proof}

\begin{proof}[Proof of Proposition \ref{proposition:min-variance-blanchet}]
We first see that if \mmode{x\opt \in \argmin_{x \in \setx} \; \big( \ambradius \dualnorm{x} + \sqrt{\inprod{x}{ \vref x  }} \big) }, then the first-order necessary optimality conditions imply that there exists a sub-gradient \mmode{g} of the function \mmode{\big( \ambradius \dualnorm{x} + \sqrt{\inprod{x}{ \vref x  }} \big)} at \mmode{x\opt} for which the inclusion \mmode{x\opt \in \argmin_{y \in \setx} \ \inprod{g}{y}} holds. A quick look reveals that \mmode{g} must be of the form \mmode{\ambradius \bar{q}_{x\opt} + \big( \nicefrac{1}{\sqrt{\inprod{x\opt}{ \vref x\opt}} } \big) \vref x\opt }, where \mmode{\bar{q}_{x\opt} \in \argmax_{\{\norm{\bar{q}} \leq 1 \}} \ \inprod{x\opt}{\bar{q}} }. Therefore, there exists some \mmode{\bar{q}_{\xopt} \in \argmax_{\{ \norm{\bar{q}} \leq 1 \}} \ \inprod{x\opt}{\bar{q}} } such that \mmode{x\opt} satisfies the inclusion
\begin{equation}
\label{eq:xopt-optimality-condition-necessary}
x\opt \in \argmin_{y \in \setx} \ \left(  \ambradius \inprod{y}{\bar{q}_{x\opt}} \ + \ \frac{\inprod{x\opt}{ \vref y }}{\sqrt{\inprod{x\opt}{ \vref x\opt  }} }  \right) .
\end{equation}
Selecting such a \mmode{\bar{q}_{\xopt}} to define \mmode{\P\opt} in \eqref{eq:popt-explicit-solution}, we also verify that 
\[
\mu_{\P\opt} \symbolspace{=}{\;} \frac{1}{N} \summ{i = 1}{N} (\xi_i + q_i\opt) \symbolspace{=}{\;} \muref + \frac{\ambradius \frac{1}{N} \summ{i = 1}{N} \inprod{x\opt}{\xi_i - \muref} }{\sqrt{ \inprod{x\opt}{ (\sigmaref - \muref \muref\transp) x\opt } }} \bar{q}_{x\opt} \symbolspace{=}{\;} \muref. 
\]
Now, we establish that the pair \mmode{(x\opt, \P\opt)} is a saddle point by proving that both \mmode{x\opt} and \mmode{\P\opt} are optimal solutions to their respective problems while the other is held fixed.

{\em Optimality of \mmode{x\opt} for the minimization condition.}
We begin by showing that the inclusion \mmode{x\opt \in \argmin_{y \in \setx} \vrisk{y}{\P\opt}} holds, by showing that the corresponding first-order optimality condition
\begin{equation}
\label{eq:xopt-optimality-condition-sufficient}
x\opt \in \argmin\limits_{y \in \setx} \inprod{\nabla_1 \vrisk{x\opt}{\P\opt} }{y} = \argmin\limits_{y \in \setx} \inprod{\big( \Sigma_{\P\opt} - \muref \muref\transp \big) x\opt }{y},
\end{equation}
is satisfied. This condition is also sufficient for optimality since \mmode{\vrisk{y}{\P\opt}} is convex in \mmode{y}. Simplifying the cost function in \eqref{eq:xopt-optimality-condition-sufficient} yields 
\begin{align*}
& \inprod{\big( \Sigma_{\P\opt} - \muref \muref\transp \big) x\opt }{y} \\
& \quad = \frac{1}{N} \summ{i = 1}{N} \inprod{x\opt}{\xi_i + q_i\opt} \inprod{y}{\xi_i + q_i\opt}  -  \inprod{x\opt}{(\muref \muref\transp) y} \\
& \quad = \frac{1}{N} \summ{i = 1}{N} \inprod{x\opt}{( \xi_i \xi_i\transp - \muref \muref\transp) y} + \inprod{x\opt}{\xi_i} \inprod{y}{q_i\opt} + \inprod{x\opt}{q_i\opt}\inprod{y}{\xi_i} + \inprod{x\opt}{q_i\opt }\inprod{y}{q_i\opt}.
\end{align*}
Letting \mmode{\vref \define \sigmaref - \muref \muref\transp }, we see that \mmode{ \vref = \frac{1}{N} \sum_{i = 1}^{N} (\xi_i - \muref)\xi_i\transp }, and \mmode{\frac{1}{N} \sum_{i = 1}^{N} \abs{ \inprod{x\opt}{\xi_i - \muref} }^2 = \inprod{x\opt}{\vref x\opt}}, from which we simplify
\[
\begin{cases}
\begin{aligned}
\frac{1}{N} \summ{i = 1}{N} \inprod{x\opt}{\xi_i} \inprod{y}{q_i\opt} \symbolspace{&=}{ \ } \ambradius \inprod{y}{\bar{q}_{x\opt} } \frac{ \frac{1}{N} \summ{i = 1}{N} \inprod{x\opt}{ \xi_i - \muref} \inprod{ \xi_i}{x\opt}  }{\sqrt{ \inprod{x\opt}{ \vref x\opt }}} &&= \ \ambradius \inprod{y}{\bar{q}_{x\opt} } \sqrt{\inprod{x\opt}{ \vref x\opt }} , \\
\frac{1}{N} \summ{i = 1}{N} \inprod{x\opt}{q_i\opt} \inprod{y}{\xi_i} \symbolspace{&=}{ \ }\ambradius \dualnorm{x\opt} \frac{ \frac{1}{N} \summ{i = 1}{N} \inprod{x\opt}{ \xi_i - \muref} \inprod{ \xi_i}{y} }{ \sqrt{\inprod{x\opt}{ \vref x\opt }} } &&= \ \ambradius \dualnorm{x\opt} \frac{\inprod{x\opt}{ \vref y }}{\sqrt{\inprod{x\opt}{ \vref x\opt }} } , \\
\frac{1}{N} \summ{i = 1}{N} \inprod{x\opt}{q_i\opt} \inprod{y}{q_i\opt} \symbolspace{&=}{ \ } \ambradius^2 \dualnorm{x\opt} \inprod{y}{\bar{q}_{x\opt}} \frac{ \frac{1}{N} \summ{i = 1}{N} \abs{ \inprod{x\opt}{\xi_i - \muref} }^2 }{ \inprod{x\opt}{ \vref x\opt} }  &&= \ \ambradius^2 \dualnorm{x\opt} \inprod{y}{\bar{q}_{x\opt}} .
\end{aligned}
\end{cases}
\]
Employing the above relations, we see that the objective function in \eqref{eq:xopt-optimality-condition-necessary} simplifies to
\begin{align*}
\inprod{\big( \Sigma_{\P\opt} - \muref \muref\transp \big) x\opt }{y} 
 &= \inprod{x\opt}{ \vref y } + \ambradius \dualnorm{x\opt} \frac{\inprod{x\opt}{ \vref y }}{\sqrt{\inprod{x\opt}{ \vref x\opt  }} } + \ambradius \inprod{y}{\bar{q}_{x\opt} } \sqrt{\inprod{x\opt}{ \vref x\opt  }} + \ambradius^2 \dualnorm{x\opt} \inprod{y}{\bar{q}_{x\opt}} \\
& =  \Big( \sqrt{\inprod{x\opt}{ \vref x\opt  }} + \ambradius \dualnorm{x\opt}  \Big) \left( \frac{\inprod{x\opt}{ \vref y }}{\sqrt{\inprod{x\opt}{ \vref x\opt  }} } + \ambradius \inprod{y}{\bar{q}_{x\opt}} \right).
\end{align*}
In view of the inclusion \eqref{eq:xopt-optimality-condition-necessary}, the sufficient optimality condition \eqref{eq:xopt-optimality-condition-sufficient} follows immediately. Consequently, the inclusion \mmode{x\opt \in \argmin_{x \in \setx} \ \vrisk{x}{\P\opt}} also holds.

{\em Optimality of \mmode{\P\opt} for the maximization condition.}
Similar to the proof of the minimization condition, we establish the maximization condition of the saddle point by showing that the first-order optimality conditions are satisfied. We first recall from \eqref{eq:min-var-unconstrained-support-FW-oracle} (with \mmode{v = \mup{\popt} = \muref}) that
\[
\etax{\xopt} \symbolspace{=}{\;} \dualnorm{\xopt}^2 + \frac{\dualnorm{\xopt}}{\ambradius} \sqrt{\inprod{\xopt}{( \sigmaref - \muref \muref\transp )\xopt}} .
\]

On the one hand, if \mmode{\inprod{\xopt}{\big( \sigmaref - \muref \muref\transp \big)\xopt} = 0}, we have \mmode{\etax{\xopt} = \dualnorm{\xopt}^2}, in which case, we conclude from assertion (ii-c) of Lemma \ref{lemma:min-variance-perturbation-solution} that \mmode{ 0 \in q(\etax{\xopt} , \xi_i , \muref)} for all \mmode{i = 1,2,\ldots,N}. Moreover, \mmode{\inprod{\xopt}{\big( \sigmaref - \muref \muref\transp \big)\xopt} = 0} also implies that \mmode{\inprod{\xopt}{\xi_i - \muref} = 0}, and hence \mmode{q_i\opt = 0} for all \mmode{i = 1,2,\ldots,N}. Thus, we have \mmode{ q_i\opt \in q(\etax{\xopt} , \xi_i , \mup{\P\opt}) } for all \mmode{i = 1,2,\ldots,N}.

On the other hand, if \mmode{\inprod{\xopt}{\big( \sigmaref - \muref \muref\transp \big)\xopt} \neq 0}, we have \mmode{\etax{\xopt} > \dualnorm{\xopt}^2}, in which case, we see 
\[
q\opt_i = \frac{\ambradius \inprod{x\opt}{\xi_i - \muref} }{\sqrt{ \inprod{x\opt}{ (\sigmaref - \muref \muref\transp) x\opt } }} \bar{q}_{\xopt} = \frac{\dualnorm{\xopt} \inprod{\xopt}{\xi_i - \muref}}{\etax{\xopt} - \dualnorm{\xopt}^2} \bar{q}_{\xopt} \symbolspace{\in}{\;} q(\etax{\xopt}, \xi_i , \muref) \quad \text{for all } i = 1,2,\ldots,N, 
\]
where the last inclusion follows from assertion (ii-b) of Lemma \ref{lemma:min-variance-perturbation-solution} since \mmode{\bar{q}_{\xopt} \in \argmax_{\norm{\bar{q}} \leq 1} \ \inprod{\xopt}{\bar{q}} }. Finally, noting that \mmode{\muref = \mu_{\P\opt}}, we have \mmode{q_i\opt \in q(\eta_{x\opt} , \xi_i , \mu_{\P\opt}) } for all $i\le N$. Since \mmode{\P\opt = \frac{1}{N} \sum_{i = 1}^{N} \delta (\xi - \xi_i - q_i\opt) }, we conclude from Lemma \ref{lemma:min-var-ldro-solution-for-m2} that \mmode{\P\opt \in \argmax_{\Q \in \amb} \ \vxpgrad{x\opt}{\P\opt}{\Q}}. Thus, \mmode{\P\opt} satisfies the first-order optimality conditions which are also sufficient. Hence, in view of Remark \ref{remark:optimality conditions}, we conclude that the inclusion \mmode{\P\opt \in \argmax_{\P \in \amb} \ \vrisk{\xopt}{\P}} also holds. Therefore, \mmode{(\xopt , \P\opt)} is indeed a saddle point of \eqref{eq:min-variance-portfolio-selection}. The proof is now complete.
\end{proof}

\paragraph{{\bf Proofs for the case of ellipsoidal support~\mmode{(\Xi = \ellipsoid)}}}

\begin{proof}[Proof of Lemma \ref{lemma:FW-oracle-min-variance-ellipsoidal-support}]
Letting \mmode{q' \define q + \xi}, the maximization problem \eqref{eq:min-var-perturbation-problem} can be equivalently written in terms of \mmode{q'} as
\begin{equation}
\label{eq:ellipsoidal-perturbation-problem-1}
\sup_{\{ q' : \inprod{q'}{M q'} \leq 1 \}} \ \inprod{ q'}{ \big( xx\transp - \eta \identity{n} \big) q' } + 2 \inprod{q'}{ \eta \xi -(xx\transp ) v} + \big( (x\transp v)^2 - \eta \pnorm{\xi}{2}^2 \big) .
\end{equation}
Even though \mmode{M \succeq 0}, which makes the feasible set of \eqref{eq:ellipsoidal-perturbation-problem-1} convex; it is to be observed that the objective function is concave if and only if \mmode{\eta \geq \pnorm{x}{2}^2}. Therefore, \eqref{eq:ellipsoidal-perturbation-problem-1} is not a convex problem in general. However, we observe that \eqref{eq:ellipsoidal-perturbation-problem-1} is a quadratically constrained quadratic program, that is strictly feasible. For such problems, the \emph{S-procedure} guarantees an equivalent reformulation as a tractable SDP
\begin{equation}
\label{eq:ellipsoidal-support-perturbation-SDP-problem}
\begin{cases}
\begin{aligned}
& \min_{\lambda \in [0 , +\infty) , \ \theta \in \R{}} && - \theta \\
& \sbjto && 
\begin{bmatrix}
\eta \identity{n} - x x\transp + \lambda M  & (xx\transp)v - \eta \xi  \\
(xx\transp)v\transp - \eta \xi\transp &  \eta \pnorm{\xi}{2}^2 - (x\transp v)^2 - \lambda - \theta
\end{bmatrix}
\succeq 0 ,
\end{aligned}
\end{cases}
\end{equation}
with the optimal values of \eqref{eq:ellipsoidal-support-perturbation-SDP-problem} and \eqref{eq:ellipsoidal-perturbation-problem-1} being equal. Moreover, if \mmode{(\theta\opt , \lambda\opt)} is a solution to the SDP \eqref{eq:ellipsoidal-support-perturbation-SDP-problem}, we also conclude from the S-procedure that \mmode{q' = \big( \eta \identity{n} - xx\transp + \lambda\opt M \big)^{-1} \big( \eta \xi - x x\transp v \big) }, is an optimal solution to the maximization problem \eqref{eq:ellipsoidal-perturbation-problem-1}. Consequently, for every \mmode{\eta \geq 0}, we conclude
\[
q(\eta , \xi) = - \xi + \big( \eta \identity{n} - xx\transp + \lambda\opt M \big)^{-1} \big( \eta \xi - x x\transp v \big),
\]
is a solution to the maximization problem \eqref{eq:ellipsoidal-perturbation-problem-1}. Substituting for each \mmode{i = 1,2,\ldots, N}, the maximization problem over \mmode{q_i} in \eqref{eq:min-variance-dual} with its equivalent SDP \eqref{eq:ellipsoidal-support-perturbation-SDP-problem}, we immediately arrive at \eqref{eq:FW-problem-ellipsoid-support-SDP}. Now, suppose \mmode{(\eta\opt , \lambda\opt, \theta\opt)} is a solution to the SDP \eqref{eq:FW-problem-ellipsoid-support-SDP}. Then, the pair~\mmode{(\Q_x , \etax{x})} as given by \eqref{eq:min-var-ellipsoidal-support-FW-oracle} is a solution to the FW problem \eqref{eq:min-variance-ldro-general} and its dual \eqref{eq:min-variance-dual}, respectively. This concludes the proof.
\end{proof}



\bibliographystyle{siam}
\bibliography{NDRO}

\end{document}